\PassOptionsToPackage{hyperfootnotes=false}{hyperref}
\documentclass[onefignum,onetabnum,hidelinks]{siamonline250211}

\ifdefined\pdfminorversion
\fi
\usepackage{lipsum}
\usepackage{amsfonts}
\usepackage{graphicx}
\usepackage{algorithmic}

\usepackage{nicefrac} 
\usepackage{subfigure}
\usepackage{amsmath}
\usepackage{amssymb}
\usepackage{url}
\usepackage{algorithm}               %format of the algorithm
\usepackage{multirow}                %multirow for format of  table
\usepackage{bm}
\usepackage{footmisc}
\usepackage{color}
\usepackage{booktabs}       % professional-quality tables
\usepackage{flushend}
\usepackage{units}
\usepackage{caption}
\usepackage{colortbl}
\usepackage{xcolor}
\usepackage{wrapfig}

\def\R{\mathbb{R}}

\def\l{\left}
\def\r{\right}

\def\EM{\mathcal{L}_{emp}}

\def\({\l(}
\def\){\r)}

\ifpdf
  \DeclareGraphicsExtensions{.pdf,.png,.jpg,.jpeg}
\else
  \DeclareGraphicsExtensions{.eps}
\fi

\usepackage{enumitem}
\setlist[enumerate]{leftmargin=.5in}
\setlist[itemize]{leftmargin=.5in}

\newsiamremark{remark}{Remark}
\newsiamremark{hypothesis}{Hypothesis}
\crefname{hypothesis}{Hypothesis}{Hypotheses}
\newsiamthm{claim}{Claim}
\newsiamremark{fact}{Fact}
\crefname{fact}{Fact}{Facts}
\newsiamremark{assumption}{Assumption}

\headers{Aligning Equivariance with Symmetry}{Feiyu Tan, Qi Xie, Zongben Xu, and Deyu Meng}

\title{Aligning Network Equivariance with Data Symmetry: A Theoretical Framework and Adaptive Approach for Image Restoration}
\author{Feiyu Tan\thanks{School of Mathematics and Statistics, Xi'an Jiaotong University, Xi'an, Shaanxi, China 
  (\email{tanfy929@stu.xjtu.edu.cn}).}
\and Qi Xie\thanks{School of Mathematics and Statistics, Xi'an Jiaotong University, Xi'an, Shaanxi, China 
  (\email{xie.qi@mail.xjtu.edu.cn}).}
\and Zongben Xu\thanks{School of Mathematics and Statistics, Xi'an Jiaotong University, Xi'an, Shaanxi, China 
  (\email{zbxu@mail.xjtu.edu.cn}).}
\and Deyu Meng\thanks{School of Mathematics and Statistics, Xi'an Jiaotong University, Xi'an, Shaanxi, China 
  (\email{dymeng@mail.xjtu.edu.cn}).}
}
\usepackage{amsopn}

\newtheorem{qgassumption}{Assumption}[section]
\crefname{qgassumption}{Assumption}{Assumptions}

\ifpdf
\hypersetup{
	pdftitle={Aligning Network Equivariance with Data Symmetry: A Theoretical Framework and Adaptive Approach for Image Restoration},
	pdfauthor={Feiyu Tan, Qi Xie, Zongben Xu, and Deyu Meng}
}
\fi

\begin{document}
	
	\maketitle
	
	% REQUIRED
	\begin{abstract}
		Image restoration is an inherently ill posed inverse problem. Equivariant networks that embed geometric symmetry priors can mitigate this ill posedness and improve performance. However, current understanding of the relationship between network equivariance and data symmetry remains largely heuristic. Particularly for real world data with imperfect symmetry, existing research lacks a systematic theoretical framework to quantify symmetry, select transformation groups, or evaluate model-data alignment.
		To bridge this gap, we conduct an analysis from an optimization perspective and formalize the intrinsic relationship among data symmetry priors, model equivariance, and generalization capability. Specifically, we propose for the first time a quantifiable definition of non strict symmetry at the dataset level (rather than sample level) and use it as a constraint to formulate the restoration inverse problem. We then show that the equivariance for restoration models can be naturally derived from the inverse problem incorporating the proposed symmetry constraints, and that the equivariance error of the optimal discrete restoration operator is bounded by the data symmetry error and the discretization mesh size.
		\textcolor{black}{Furthermore, by evaluating the empirical risk, we theoretically explain how aligning network equivariance with data symmetry prevents severe fitting degradation, despite the strict geometric constraints imposed on the architecture. Integrating these findings with generalization analysis reveals how network equivariance fundamentally influences the overall expected risk when data inherently possesses symmetry.}
		Motivated by these insights, we propose a Sample Adaptive Equivariant Network that uses a hypernetwork and transformation learnable equivariant convolutions to dynamically align with each sample’s inherent symmetry. Extensive experiments on super resolution, denoising (natural and remote sensing images), and deraining validate our theoretical findings and demonstrate improvements over standard baselines and traditional equivariant models.
		Our code is available at \href{https://github.com/tanfy929/SA-Conv}
		{https://github.com/tanfy929/SA-Conv}.
	\end{abstract}
	
	% REQUIRED
	
	% REQUIRED
	
	\section{Introduction}\label{sec:Intro}
	
	Images serve as a crucial medium for information recording and transmission, where the quality affects  the reliability of downstream tasks such as analysis, diagnosis, and further processing. However, the processes of acquisition, compression, and storage inevitably introduce various degradations, including noise, blur, and artifacts. Consequently, image restoration has emerged as a fundamental problem in computer vision, aiming to recover high-quality images from degraded observations \cite{demoment1989image, wen2008iterative}. It holds great theoretical and practical value, which is essential for remote sensing, medical imaging, security surveillance, and heritage preservation \cite{rasti2021image, stanco2011digital, zhang2018reweighted}.
	
	Mathematically, image restoration can be formulated as an ill-posed inverse problem \cite{kerepecky2024inverse, pendu2023preconditioned, vogel2002computational}, where  regularizations are usually introduced to mitigate the ill-posed issue. Formally, let $\hat{Y}\in \mathbb{R}^{m\times m}$ denote a degraded image, $K\in \mathbb{R}^{p\times p}$ be a blur kernel, and $D$ represent an identity or down-sampling operator. 
	The general single-image restoration problem can be formulated as:
	\begin{equation}\label{oralProm}
		Y^* = \arg \min_{X\in \tilde{\Omega}}(\min_{K}\|\hat{Y} - D(K \ast X)\|_F^2 + \tilde{R}(X)),
	\end{equation}
	where $\ast$ denotes discrete spatial convolution, the feasible set  $\tilde{\Omega}$  and regularization term $\tilde{R}(X)$ both encode the prior knowledge of the data, restricting the solution space to valid images.
	This formulation has laid the foundation for unsupervised image restoration and has played a key role over the past decade or so. Recently, it still provides a principled framework for the theoretical analysis of image processing, and has become one of the most standard formal cornerstones in image restoration theory \cite{chen2016trainable, zhang2017beyond}.
	
	%It is worth noting that the design of regularization terms is often based on assumptions with human intuition, making it difficult to accurately capture the complex, true prior information of natural images. This inevitably leads to discrepancies between the model's solution and the ideal restoration results\cite{chen2016trainable, zhang2017beyond}. Furthermore, solving the aforementioned problem typically involves a complex non-convex optimization process, which introduces challenges such as difficulty in parameter tuning and low computational efficiency. These issues, in turn, limit the model's complexity and consequently constrain its representational capacity \cite{chambolle2011first, chan2016plug}.
	
	Rather than directly solving problem \eqref{oralProm}, data-driven deep learning strategies have been more widely adopted in recent years \cite{dong2015image, zhang2017beyond}. Leveraging their powerful non-linear approximation capabilities, deep networks have significantly outperformed traditional methods across various tasks. 
	Formally, let $\Psi$ denote a deep network, the deep learning-based image restoration problem can be formulated as:
	\begin{equation}\label{DL-Model}
\Psi^* = \arg\min_{\Psi\in\Lambda} {\sum_{n=1}^{N} \|\hat{X}_n -\Psi(\hat{Y}_n)\|_F^2},
	\end{equation}
	where $\{\hat{Y}_n, \hat{X}_n\}_{n=1}^N$ represents the paired training dataset, $\hat{Y}_n$ is the degraded observation (e.g., low-resolution, noisy, or rainy image) and $\hat{X}_n$ is the corresponding ground-truth high-quality image. %\textcolor{black}{Here $m_r\times m_r$ is the output image size, so the loss is the mean squared error per output pixel.}  
$\Lambda$ represents the hypothesis space of the network determined by the chosen neural network architecture. 	Despite their remarkable success, standard deep learning architectures still suffer from notable limitations. The most prominent issue is their black-box design, which makes them difficult to explicitly embed data priors (i.e., designing $\Lambda$). As a result, standard networks exhibit poor interpretability, and heavily rely on large-scale paired data.
	
	Recently, equivariant network\footnote{The equivariance of a network requires that when the input undergoes a specific transformation, the internal feature layers and the final output undergo corresponding and predictable transformations.} \cite{cohen2016group, weiler2019general} has been widely recognized as an effective approach to explicitly embed geometric symmetry priors, which is one of the most intrinsic data priors. Its effectiveness has been extensively validated \cite{cohen2016group, weiler2019general,xie2022fourier, fu2024rotation}. The most representative example is the classical CNN architecture, whose convolutional operations perfectly realize translation equivariance, thereby improving parameter efficiency and model generalization, and achieving performance far superior to that of fully connected networks in image processing tasks. Inspired by CNNs, equivariant network structures incorporating more transformations, such as rotation, reflection, and scaling, have been successively proposed \cite{cohen2016group, weiler2019general}, further demonstrating the advantages of equivariant design in specific scenarios \cite{fu2024rotation, xie2022fourier}. Such successes are often attributed to the explicit utilization of the corresponding geometric symmetry priors.
	
	However, the current understanding of the relationship between network equivariance and data symmetry priors remains largely at a heuristic level, relying heavily on human intuition. In particular, for real-world data with imperfect symmetry \cite{petrache2023approximation}, existing research lacks a systematic theoretical framework to quantify data symmetry, select transformation groups, or formally evaluate the alignment between  data characteristics ($\tilde{\Omega}$ and $\tilde{R}$) and  network constraints ($\Lambda$).
The pioneering work \cite{celledoni2021equivariant} took an important step by theoretically analyzing the relationship between  invariant regularizers ($\tilde{R}$)  and equivariant proximal operators. Nevertheless, this conclusion relies on the assumption of an ideally invariant regularizer, which deviates from the practical need for expressive regularizers. Moreover, this work does not clarify the relationship between general network and the dataset, and thus offers limited guidance for designing equivariant networks.
%Recent studies \cite{petrache2023approximation} have suggested that strict equivariance may not be optimal and that the embedded symmetry must align with the data itself. 
%However, these studies still do not provide a mathematical definition of data symmetry, falling short of providing rigorous theoretical justifications.
	More fundamentally, prior works generally treat geometric symmetry as an inherent property of individual images and establish models based on this single-sample assumption. This perspective is inherently limited: an individual image rarely exhibits strict geometric symmetry, and enforcing constraints on isolated samples cannot exploit the rich inter-sample correlations that naturally exist across a dataset.
	
	In summary, as shown in the left part of Fig.~\ref{roadmap},  existing approaches suffer from three key limitations. 
	(I) Their consideration of symmetry is mainly grounded in single-sample modeling, which is misaligned with real-world data distributions (image samples are usually not symmetric). Moreover, its definition lacks quantifiability, making it difficult to formally evaluate the degree of symmetry and to theoretically analyze its relationship with model equivariance. 
	(II) The sample-level consideration unavoidably results in a failure to leverage inter-sample information.
	(III) The selection of transformation groups remains heuristic, offering little principled guidance for network design.

	\begin{figure}
		\centering
		\includegraphics[width=1\linewidth]{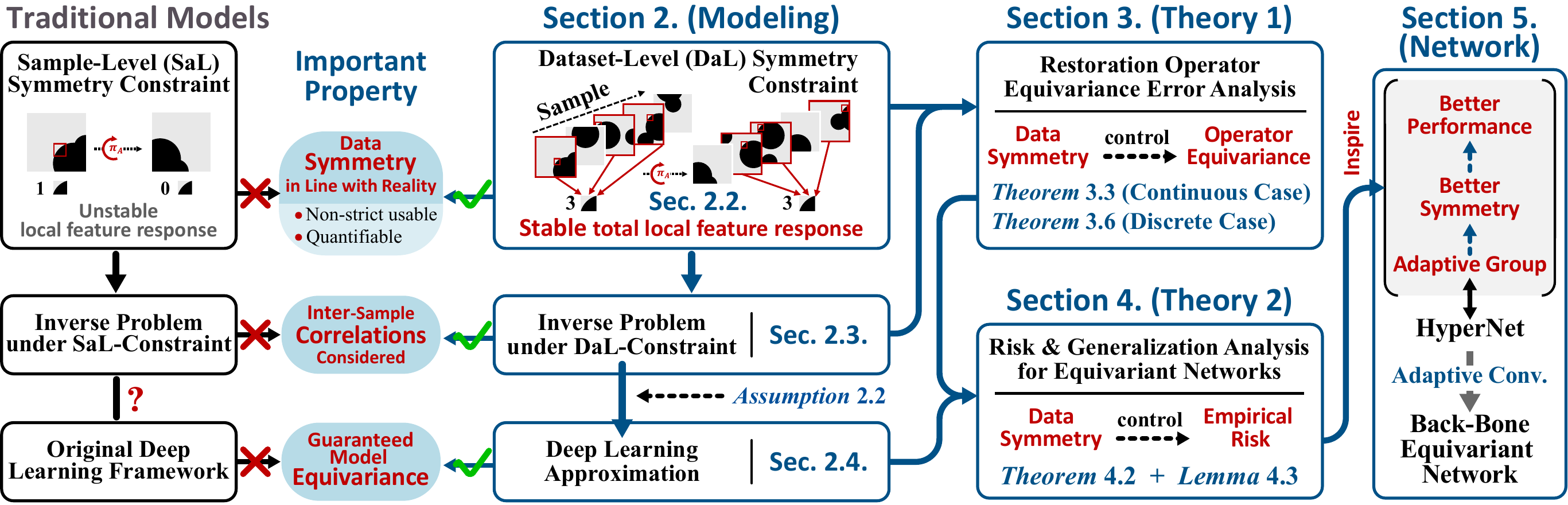}
		\caption{ Overview of the structure of this work. To address the limitations of existing works, Section \ref{sec:model} formulates a dataset-level, non-strict symmetry constraint and the associated restoration model. This forms the foundation for the analyses in Sections \ref{sec:eq-analysis} and \ref{sec:risk-analysis}, which reveal that data symmetry governs operator equivariance and network empirical risk. These findings motivate the adaptive equivariant network design in Section \ref{sec:network-design}, achieving better alignment between model equivariance and data symmetry and leading to improved performance.}
		\label{roadmap}
		\vspace{-10mm}
	\end{figure}
	
	To overcome these limitations, as outlined in right part of Fig.~\ref{roadmap}, we conduct a rigorous analysis for dataset symmetry,  model the image restoration problem from a statistical, dataset-level perspective and  establish the intrinsic relationship among data symmetry  and model equivariance.
	Specifically, 
	we propose, for the first time, a quantifiable definition of non-strict symmetry at the dataset level, and utilize it as a dataset-level constraint to formulate the corresponding image restoration inverse problem.
	Theoretically, this formulation yields an optimal restoration operator, which represents the performance upper bound. Recognizing deep neural networks as universal function approximators, we further postulate that an adequately trained deep neural network effectively converges to this optimal restoration operator. Based on this insight, we show that the equivariance of restoration models can be naturally derived from the proposed dataset-constrained inverse problem. This allows us to quantify the equivariance error bound and further analyze the network's empirical risk, thereby providing a solid theoretical foundation for general deep network design.
	
	To avoid the comprehension difficulties caused by the complex technical approach described above, we provide a detailed illustration of the paper’s main structure, contributions, and the logical relationships among different sections, in Fig.~\ref{roadmap}. Please refer to it for an easy understanding.
	Overall, the main contributions of our work can be summarized as follows:
	\begin{itemize}
		\item \textbf{Theoretical Modeling of Data Symmetry and Operator Equivariance.} 
		We construct a quantifiable symmetry metric for non-strict symmetry at the dataset level, which overcomes the inherent mismatch between single-sample modeling and real-world data distributions.
		By incorporating this metric as a constraint in the image restoration inverse problem, we mathematically formalize the intrinsic relationship between data symmetry and operator equivariance. 
		Our proposed theorems establish that in the continuous domain, the equivariance error of the optimal restoration operator is bounded by the dataset symmetry error (Theorem \ref{Error-IR}). 
		When extended to the discrete domain, this error is jointly governed by the data symmetry error and the spatial discretization mesh size (Theorem \ref{Error-dis}).

%		\item \textbf{Risk Analysis of Equivariant Networks.} 
%		Based on the established equivariance properties of the optimal restoration operator, we further derive the empirical risk bound for equivariant networks (Theorem \ref{ER}). 
%		\textcolor{black}{We further establish a generalization bound for shared-parameter networks whose sample-dependent groups are supplied by an ideal hypernetwork fixed before sampling. Combining this result with the empirical risk bound identifies conditions under which symmetry alignment and coefficient sharing yield a smaller expected risk upper bound, motivating our sample-adaptive network design.}

\item \textbf{Risk Analysis of Equivariant Networks.} 
Based on the established equivariance properties of the optimal restoration operator, we further derive the empirical risk bound for equivariant networks (Theorem \ref{ER}). 
\textcolor{black}{Particularly, by examining the special case of uniform transformation groups, our framework explains why classical equivariant networks avoid severe fitting degradation despite their strict geometric constraints. Furthermore, by integrating this with a generalization analysis, we provide a high-level intuition into how network equivariance impacts the overall expected risk when the underlying data inherently possesses symmetry. Together, these insights explicitly motivate our sample-adaptive network design.}
		
		\item \textbf{Sample-Adaptive Network Design.} 
		Based on our theoretical findings, we propose for the first time an equivariant network whose transformation group adapts to individual samples (Section \ref{sec:network-design}). 
		Specifically, we construct a class of parameterized transformation groups and innovatively employ a hypernetwork to predict the group parameters from each sample, thereby achieving adaptive adjustment of the transformation group. 
		Furthermore, by feeding the adaptive group parameters into an equivariant network architecture built upon parameterized transformation groups \cite{tan2026image}, we can  obtain a transformation-group-adaptive equivariant network. 
		This mechanism can be plug-and-play integrated into existing convolutional architectures, maximizing the alignment between network equivariance and data symmetry.
		
		\item \textbf{Extensive Experimental Verification.} We evaluate our method on diverse tasks, including super-resolution and denoising for natural and remote sensing images, as well as single-image rain removal (Section \ref{sec:experiment}). The proposed method consistently outperforms both standard baselines and traditional strict equivariant networks, validating the correctness and practical value of our theoretical framework.
	\end{itemize}
	
	The rest of the paper is organized as follows. Section \ref{sec:model} establishes the mathematical formulation of the dataset symmetry and restoration problem. Section \ref{sec:eq-analysis} analyzes the equivariance properties of the optimal restoration operator. Section \ref{sec:risk-analysis} investigates the empirical risk and generalization of equivariant networks. Section \ref{sec:network-design} details the proposed sample-adaptive equivariant network. Finally, experiments are presented in Section \ref{sec:experiment}.
	
	\section{Mathematical Modeling}
	\label{sec:model}
	In this section, we formulate image restoration for the first time as an inverse optimization problem at the dataset level, introduce dataset symmetry constraints, and associate it with a deep network trained on the same dataset.
	We start from introducing the continuous formulation of both traditional sample-level inverse problem and the proposed dataset-level one, treating digital images as discretely sampled instances of an underlying continuous physical signal. This framework allows us to analyze the problem at its source, effectively decoupling the intrinsic geometric properties of the solution from sampling artifacts.
	Subsequently, we extend this model to the discrete domain to reflect practical digital imaging processes.
	Furthermore, following the commonly adopted assumption that neural networks are universal approximators, we present a key hypothesis concerning the relationship between deep learning methods for image restoration and the optimization framework, thereby laying the foundation for the subsequent sections.
	
	\subsection{Continuous Problem Formulation on Single Sample}\label{model for single image}
	In the continuous domain, images are typically modeled as smooth functions defined on a continuous support. Consequently, for the image restoration problem, the degraded image and the blur kernel can be represented as functions $\hat{r}\in C^\infty(\R^2)$ and $k\in C^\infty(\R^2)$, respectively.  
	Formally, let $h$ denote the spatial mesh size, a digital image can be viewed as the spatial discretization of a smooth function at the cell-centers of a regular grid, thus we have:
	\begin{equation}\label{dis}
		\hat{r}(x_{ij})=\hat{Y}_{ij} ,~   K_{uv}=k(t_{uv})h^2, 
		% \vspace{-2.5mm}
	\end{equation}
	where $x_{ij} \!=\! \left(\left(i\!-\!\nicefrac{(m\!+\!1)}{2}\right)sh, \left(j\!-\!\nicefrac{(m\!+\!1)}{2}\right)sh\right)^T$\!\!, \!$t_{uv} \!=\! \left(\left(u\!-\!\nicefrac{(p\!+\!1)}{2}\right)h, \left(v\!-\!\nicefrac{(p+1)}{2}\right)h\right)^T$, and $s$ denotes the downsampling rate of  $D$ in \eqref{oralProm}. If the image degradation process does not involve  downsampling, we can simply set $s=1$. For a systematic treatment of the continuous-discrete modeling details, please refer to Appendix \ref{app:supp:discretization-framework}.
%\begingroup\color{black}\hypersetup{allcolors=green}
%For the discretization analysis, let $D_r,D_k,D_R>0$ be fixed support-radius bounds for the images $\hat r,r$, the blur kernels $k$, and the regularization kernels $k^R$, respectively. These bounds also apply to the admissible transformed functions used below. We use fixed square sampling windows with
%\begin{equation}\label{fixed-windows}
%    m_rh=msh=2(D_r+D_k),\qquad ph=2D_k,\qquad m_Rh=2(D_r+D_R).
%\end{equation}
%Here $m_r$ is the restored-image grid size and $m_R$ is the regularization quadrature grid size in each coordinate. The radii may be enlarged to fit an initial grid; subsequent grids are obtained by halving $h$ and doubling all grid sizes, while $s$ remains fixed. The window sizes include the support of the corresponding convolution responses. Thus $m,m_r,p,m_R$ vary with $h$, whereas the continuous model and its physical support bounds remain fixed. Pixel centers retain the convention in \eqref{dis}; the window side lengths are $msh$, $m_rh$, $ph$, and $m_Rh$, respectively.
%\par\endgroup

	Then, the general restoration problem for a single image can be formulated as the following optimization task:
	\begin{equation}\label{model-single}
		\Phi[\hat{r}] = \arg\min_{r\in \Omega}\l(\min_{k}\int_{\R^2}\left(\hat{r}(x)-\int_{\R^2}k(t)r(x-t)dt\right)^{2}\!dx+R[r]\r),
	\end{equation}
	where $\Phi[\hat{r}]$ denotes the restoration operator, $R[r]$ is the regularization functional, and the feasible set $\Omega$ constrains the space of valid image functions. The inner minimization subproblem with respect to the blur kernel $k$ usually admits a unique optimal solution in practice.

	\subsection{Symmetry Constraints on Dataset Level}\label{sec:symmetry}
	
	In formulation \eqref{model-single}, the regularizer $R[r]$ and the feasible set $\Omega$ encode prior knowledge to restrict the solution space for a single sample. 
	However, as discussed in Section \ref{sec:Intro}, this sample-level approach inevitably results in a failure to exploit inter-sample correlations and an inability to overcome the fact that an individual image rarely exhibits ideal, strict geometric symmetry.
	Conversely, empirical observations suggest that the statistical distribution of local features exhibits strong symmetric properties across the entire dataset, yet a rigorous mathematical definition of this dataset-level symmetry has been absent in previous literature. 
	In this section, we propose, for the first time, dataset-level symmetry metrics, which are expected to be significant for model generalization.
	
\textbf{Consistent Symmetry Constraint.} 
		\textcolor{black}{Motivated by the classical symmetry definition for a single object $r$, where a descriptive functional $\bar{R}$ remains strictly invariant under transformations \cite{weinberg1995quantum}, i.e., $\bar{R}[\pi_A[r]] = \bar{R}[r], \forall A \in G$ (which is equivalent to $\bar{R}[\pi_A[r]] = \bar{R}[\pi_B[r]], \forall A, B \in G$), we define the symmetry constraint set for the image dataset $\{r_n\}_{n=1}^{N}$ with respect to  transformation group $G$. Specifically,  by setting the descriptive functional as $R$, the corresponding symmetry set is formulated as:
		\begin{equation}\label{symmetry}
			\begin{aligned}
				\Omega=\Bigg\{[r_1,\ldots,r_N]\ \Bigg|\ 
				\sum_{n=1}^{N}R[\pi_A[r_n]]
				=\sum_{n=1}^{N}R[\pi_B[r_n]],\forall A,B\in G\Bigg\},
			\end{aligned}
		\end{equation}
		where $\pi_A$ denotes the action of the transformation $A$ on the image function, explicitly given by $\pi_A[r](x)=r(A^{-1}x)$ for $x\in\mathbb{R}^2$, and $\pi_B$ is defined in the same way.
	In this paper, we take the regularization functional $R[\cdot]$ of the image restoration problem \eqref{model-single} as the characteristic functional of symmetry.
The constraint in \eqref{symmetry} indicates that, as long as the sum (equivalent to the expectation) of the local feature quantity characterized by 
$R$ over the entire dataset is invariant under a transformation,  we can say that the local feature is symmetric with respect to that  characteristic functional and transformation at the dataset level. } 
	For the adaptability of this definition to real data (comparing with defining symmetry at single sample level), please refer to Fig.\ref{symmetry_show}.
	
	\textbf{Adaptive Symmetry Constraint.} 
	The consistent symmetry constraint assumes that all images share the same transformation group $G$, which may still deviate considerably from real-world situations. For example, in image data, different samples are often captured from varying viewpoints, leading to different optimal symmetry transformation groups for different samples \cite{tan2026image}. Therefore, we further extend this to a scenario where each image is optimally aligned with a sample-specific transformation group. Formally, we define symmetry with respect to a family of transformation groups 
	% $\{G_n\}_{n=1}^N$ as follows\footnote{Please refer to Section \ref{sec:network-design} for an intuitive example of transformation group families.}:
	$\{G_n\}_{n=1}^N$ as follows\footnote{
{\color{black}It is important to emphasize that the sample-adaptive groups in such a network are not constructed arbitrarily but are structurally constrained to share a rigorously consistent algebraic correspondence across all samples. Specifically, we write $G_n=\{A_n^t\}_{t=1}^T$ and index the group elements consistently across samples, so that multiplying corresponding elements yields corresponding products. For brevity, in dataset-level statements, $A_n$ and $B_n$ denote $A_n^{t_1}$ and $A_n^{t_2}$, respectively, with the same indices $t_1,t_2$ used for all samples. This convention applies to general group families with such a correspondence. As a concrete example, the implementation in Section~\ref{sec:network-design} uses $A_n^t=D_nA^tD_n^{-1}$, where $\{A^t\}_{t=1}^T$ is a reference rotation group and each sample-specific affine transformation $D_n$ is controlled by three learnable parameters.
}
% END APPROVED-REVIEW 01-group-footnote
}:
	%\begin{equation}\label{symmetry2}
	%    \Omega = \left\{[r_1, r_2, \cdots, r_N]\mid \sum_{n =1}^{N} R[{r}_n] = \sum_{n =1}^{N} R[\pi_{A_n}[{r}_n]], \forall A_n\in G_n \right\},
	%\end{equation}
	\textcolor{black}{
		\begin{equation}\label{symmetry2}
			\begin{aligned}
				\Omega=\Bigg\{[r_1,\ldots,r_N]\ \Bigg|\ 
				\sum_{n=1}^{N}R[\pi_{A_n}[r_n]]
				=\sum_{n=1}^{N}R[\pi_{B_n}[r_n]], \forall A_n,B_n\in G_n, ~n=1,\ldots,N\Bigg\},
			\end{aligned}
		\end{equation}
		where $G_n$ represents the transformation group for the $n$-th sample, $\pi_{A_n}$ and $\pi_{B_n}$ denote the corresponding sample-dependent group actions.}
	\textcolor{black}{For a given dataset, the family of groups $\{G_n\}_{n=1}^N$ is predefined and held fixed throughout the formulation of the restoration problem and the subsequent theoretical analysis.}
	This formulation permits the establishment of distinct, adaptive symmetries for individual samples. 
	For the adaptability of this definition  to real data, please refer to Fig.~\ref{symmetry_show} and Table \ref{Tab:Symmetry_Validation}.

\textbf{Relaxed Symmetry Constraint.} 
Moreover, in practice, strict symmetry rarely holds. To accommodate this, we introduce a relaxed constraint set bounded by an error margin $\varepsilon$:
\textcolor{black}{
    \begin{equation}\label{symmetry3}
        \begin{aligned}
            \Omega_\varepsilon \!=\! \Bigg\{[r_1,\ldots,r_N] \Bigg|\ 
            \!\!\frac{1}{N}\!\left|
            \sum_{n=1}^{N}\!R[\pi_{A_n}[r_n]]
            \!-\!\!\sum_{n=1}^{N}\!R[\pi_{B_n}[r_n]]
            \right| \!\leq\varepsilon,\forall A_n,B_n\!\in \!G_n, n=1,\ldots,N\Bigg\},
        \end{aligned}
    \end{equation}} 
    where $\varepsilon\geq0$ bounds the variation of the dataset-averaged regularizer between any two transformed versions of the dataset, ensuring that the feasible set $\Omega_\varepsilon$ is invariant under the considered group actions.

\textcolor{black}{To formally support the subsequent theoretical analysis for the sample-adaptive scenario, we extract the fundamental properties of the considered transformations and the group mapping into the following assumption:}
\textcolor{black}{
\begin{assumption}[Admissible Sample-Adaptive Groups]\label{ass:adaptive_group}
    For any continuous sample $r_n$ and its discrete counterpart $Y_n$, there exists a mapping $\mathcal{H}_G$ that assigns a sample-specific transformation group $G_n$, i.e., $\mathcal{H}_G(r_n)=G_n$ and $\mathcal{H}_G(Y_n)=G_n$. We assume that the mapping $\mathcal{H}_G$ is invariant to the actions of the assigned group, i.e., for any transformation $A \in G_n$, it holds that $\mathcal{H}_G(r_n)=\mathcal{H}_G(\pi_A[r_n])$ in the continuous domain, and $\mathcal{H}_G(Y_n)=\mathcal{H}_G(\tilde{\pi}_A(Y_n))$ in the discrete domain, where $\tilde{\pi}_A$ denotes the discrete transformation action detailed in \eqref{disTrans}. Besides, all transformations within the assigned group $G_n$ are measure-preserving, satisfying $|\det(A)|=1$ for all $A \in G_n$, and each assigned group is finite of order $T$.
\end{assumption}
}

\textcolor{black}{This assumption guarantees that symmetric data samples consistently map to a stable symmetry group, which well aligns with our intuitive understanding of data. Furthermore, the measure-preserving condition satisfies the requirements of the transformation groups considered in most existing works, such as standard rotations and reflections, as well as the sample-adaptive transformations \cite{tan2026image} constructed in our framework.}

In the remainder of this paper, we focus on solving the image restoration problem under this type of symmetry constraint, as it most faithfully reflects practical data distributions. As shown in Fig.~\ref{roadmap}, it serves as the fundamental premise of our theory, controlling the equivariance error bound of the optimal restoration operator, which will be derived in Section \ref{sec:eq-analysis}. It is worth noting that when the transformation group and operations are uniform across all samples (i.e., $G_n=G$, $A_n=A$, and $B_n=B$ for all $n$), \eqref{symmetry3} degenerates into a relaxed symmetry constraint for a consistent transformation group. Consequently, our subsequent theoretical analysis also applies to fixed-group equivariant networks, thereby providing a formal symmetry characterization and corresponding theoretical support for the majority of existing equivariant networks based on fixed groups.

\textbf{Symmetry Verification on Real Datasets.} To empirically validate this symmetry on real-world datasets, we introduce a metric that quantifies the stability of the dataset-averaged regularizer under sampled transformations. Specifically, given a discretized transformation group $G_n = \{A_n^t\}_{t=1}^T$ assigned to each sample and defining the encompassing group family $\mathcal{G}=\{G_n\}_{n=1}^N$, we formulate the empirical symmetry error to strictly align with the pairwise theoretical constraint defined in \eqref{symmetry3}. \textcolor{black}{By capturing the maximum discrepancy between dataset-averaged responses under any two corresponding transformations, this metric is computed as:
\begin{equation}\label{metric}
    \varepsilon_{\mathcal{G}}
    =\max_{1\leq t_1,t_2\leq T}\frac{1}{N}\sum_{n=1}^{N}
    \left(
    R[\pi_{A_n^{t_1}}[r_n]]
    -R[\pi_{A_n^{t_2}}[r_n]]
    \right).
\end{equation}}

\begin{wrapfigure}{h}{0.5\linewidth}
		\centering
		\vspace{-2mm}
		\includegraphics[width=0.5\textwidth]{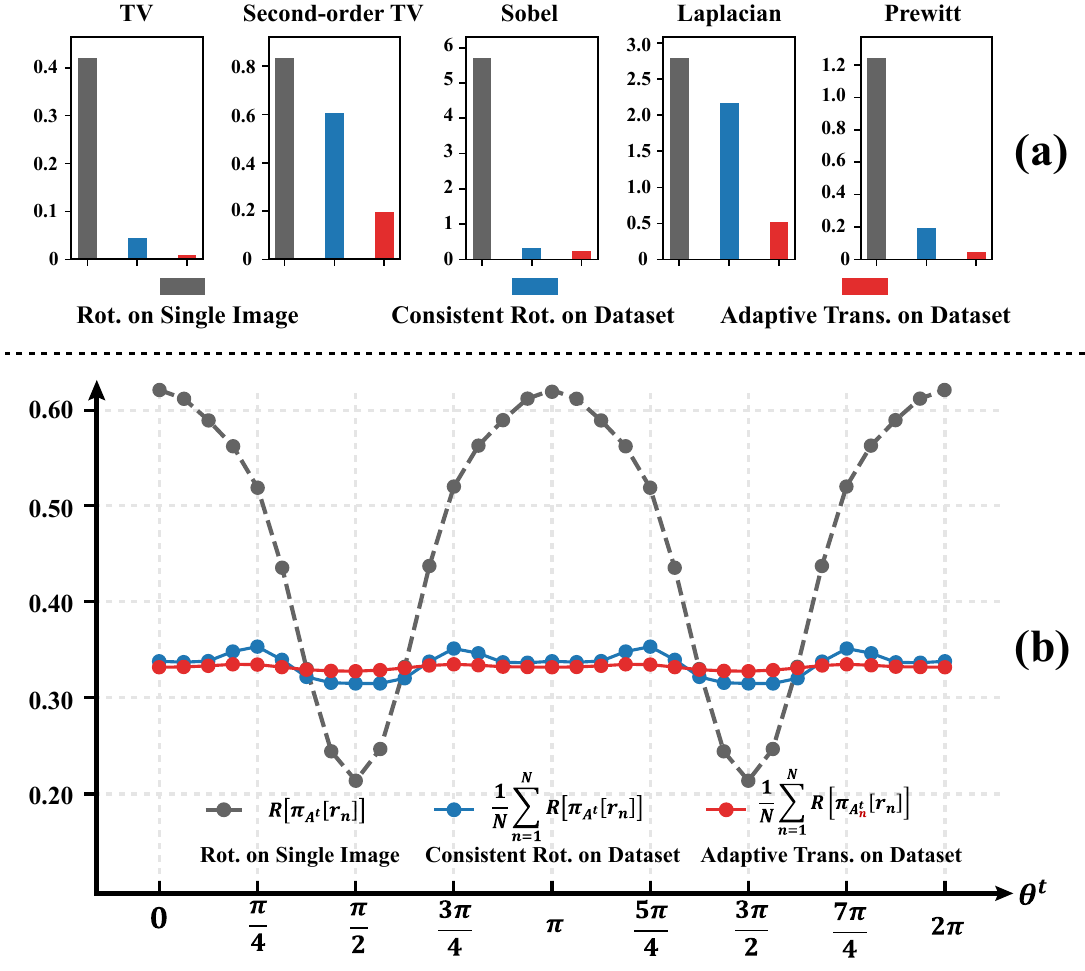}
		\vspace{-7mm}
		\caption{\textcolor{black}{(a) Bar charts of the symmetry metric $\varepsilon_{\mathcal{G}}$ for three constraint scenarios under five different regularizers. (b) The local feature responses under varying rotation angles across three scenarios, using TV as an example.}}
		\label{symmetry_show}
		\vspace{-7mm}
	\end{wrapfigure}

	We evaluate $\varepsilon_{\mathcal{G}}$ on the widely used DIV2K dataset \cite{agustsson2017ntire}, under various commonly used regularizers $R$.
	We compare three scenarios: (i) strict rotations on a single image (sample-level), (ii) strict rotations on the entire dataset (dataset-level), and (iii) adaptive transformations on the dataset (adaptive dataset-level). Detailed settings and numerical results are provided in Section \ref{sec:exp-symmetry}.
		
	Fig.~\ref{symmetry_show}(a) shows the bar charts of the metric $\varepsilon_{\mathcal{G}}$ across three scenarios under five different regularizers $R$. A smaller $\varepsilon_{\mathcal{G}}$ indicates higher feature stability under the sampled transformations. As observed, across all evaluated regularizers, the dataset-level measurements (blue bars) consistently yield lower error values than single-sample measurements (gray bars), and the adaptive configuration (red bars) achieves the lowest values. 	The consistent results across various regularizers support the applicability of the dataset-level formulation to different local feature functionals.
	Taking the TV operator as an example, Fig.~\ref{symmetry_show}(b) visualizes the line charts of three scenarios. The gray curve fluctuates sharply, indicating that the single image cannot yield stable symmetry metric values (i.e., cannot provide reliable geometric priors), while the blue (consistent rotation on entire dataset) and red (adaptive transformations on dataset) curves both achieve more stable values. This highlights the importance of defining symmetry metric at the dataset level and the advantage of adaptive transformation groups.

	\textcolor{black}{\textbf{Remarks about the Proposed Constraint.} Constructing an effective regularizer $R$ is crucial for obtaining optimal solutions to image restoration inverse problems. Consequently, when utilizing $R$ to formulate data symmetry, the symmetry definition must avoid imposing overly rigid restrictions on its choice. Otherwise, it would severely interfere with the underlying optimization modeling. A fundamental advantage of the proposed dataset-level symmetry constraint lies in the remarkable flexibility it affords in selecting this characteristic functional. While existing works have made significant progress in deriving model equivariance from the invariance of $R$ \cite{celledoni2021equivariant}, this rigid requirement typically demands $R$ to be absolutely invariant across arbitrary individual samples.  Such a stringent condition severely restricts the choice of admissible regularizers, often failing to accommodate the complex priors necessary for yielding high-quality restoration results. 
By shifting the symmetry constraint to the dataset level, our formulation fundamentally relaxes this bottleneck: it requires neither the strict invariance of $R$ nor the perfect geometric symmetry of individual images. Consequently, this formulation naturally accommodates a broad and highly practical class of local-feature regularizers. Within our framework, common local feature extractors (particularly convolution-based ones) usually satisfy the relaxed constraint in \eqref{symmetry3}. Since these descriptors are already proven to be highly effective regularizers for inverse problems, satisfying our symmetry constraint does not compromise their empirical restorative efficacy. }

	\subsection{Restoration Model under Symmetry Constraint}
	% \subsection{Continuous Problem Formulation under Symmetry Constraint}\label{model for dataset}
	Based on the established relaxed symmetry constraint, we formulate the corresponding restoration models in the continuous and discrete domains\footnote{\textcolor{black}{Please refer to Appendix \ref{app:supp:discretization-framework} for the systematic treatment of the continuous-discrete modeling details.}}, respectively. 
	
	\textbf{Continuous Problem Formulation.} 
	Building upon the dataset-dependent symmetry constraint $\Omega_\varepsilon$, we can construct the optimal restoration model at the dataset level:
	\begin{equation}\label{Modelall}
		\left\{\Phi[\hat{r}_n]\right\}_{n=1}^N = \arg\!\! \min_{\{r_n\}_{n=1}^N\in \Omega_\varepsilon}\sum_{n=1}^N \left(\min_{k}\int_{\R^2}\left(\hat{r}_n(x)-\int_{\R^2}k(t)r_n(x-t)dt\right)^{2}dx+R[r_n]\right)\!,
	\end{equation}
	where $\{r_n\}_{n=1}^N$ and $\left\{\Phi[\hat{r}_n]\right\}_{n=1}^N$ denote the tuples of the target functions and the corresponding restoration results, respectively,  $\Phi$ denotes the restoration operator.
In this paper, we assume that the inner minimization subproblem with respect to the blur kernel $k$ admits a unique optimal solution in all the following analysis, which aligns well with common practical scenarios.
	It should be noted that, unlike traditional restoration operators based on a single sample, the proposed $\Phi$  takes into account cross-sample information to a certain extent, which well overcomes the two fundamental limitations of single-sample constrained model aforementioned, i.e.,   the inability to exploit inter-sample correlations across the entire dataset, and the lack
	of ideal geometric symmetry in a single natural image.

	% \subsection{Discrete Problem Formulation under Symmetry Constraint}\label{dis_model}
	\textbf{Discrete Problem Formulation.}
	While the continuous formulation establishes a rigorous theoretical baseline, practical digital imaging inherently necessitates a discrete computational model. Consequently, we introduce the discrete restoration operator $\tilde{\Phi}$. With the discretization defined in \eqref{dis}, the dataset-level discrete restoration model is formulated as:
	\begin{equation}\label{Modelall_dis}
		\{\tilde{\Phi}(\hat{Y}_n)\}_{n=1}^N = \arg \!\!\min_{\{Y_n\}_{n=1}^N\in \tilde{\Omega}_\varepsilon} \sum_{n=1}^N (\min_{K}\|\hat{Y}_n - D(K \ast Y_n)\|_F^2 + \tilde{R}(Y_n)),
	\end{equation}
	where $D$ is the downsampling operator and $\tilde{R}(\cdot)$ is the discrete counterpart of $R$, and $K_{uv}=k(t_{uv})h^2$ are the discrete convolution weights (see \eqref{app:sampled-convolution} for details). The set $\tilde{\Omega}_\varepsilon$ represents the feasible set for the problem:
	\begin{equation}\label{symmetry3_dis}
		\tilde{\Omega}_\varepsilon \!=\! \left\{\{Y_n\}_{n=1}^N \mid \{r_n\}_{n=1}^N \!\in\! \Omega_\varepsilon~\mbox{and}~\left(Y_n\right)_{ij} = r_n\left(x_{ij}\right), \forall n \!=\! 1,2,\cdots\!, N, (i,j)\in\Delta_r \right\}\!,
	\end{equation}
	where \textcolor{black}{$\Delta_r=\{(i,j)\mid i,j=1,2,\ldots,m_r\}$ and $x_{ij}=\big((i-(m_r+1)/2)h,(j-(m_r+1)/2)h\big)^T$ are the index set and coordinates of the $m_r\times m_r$ restored-image grid with mesh size $h$.}
	% Frameworks \eqref{Modelall} and  \eqref{Modelall_dis} enables us to quantify the discrepancy between the continuous operator $\Phi$  (ideal) and its discrete approximation $\tilde{\Phi}$  (practical)  through the analysis of discretization errors, which will be extensively discussed in Section \ref{sec:eq-analysis}.
	As shown in Fig.~\ref{roadmap}, the two constrained inverse problems established in \eqref{Modelall} and \eqref{Modelall_dis} will enable us to quantify the equivariance properties of the optimal operator in Section \ref{sec:eq-analysis}.

	\subsection{Deep Learning Approximation for the Optimal Operator}\label{sec:dl-assumption}
	While the discrete model formulated in \eqref{Modelall_dis} provides a mathematical definition of the image restoration task, directly solving this optimization problem via traditional iterative algorithms usually involves a high computational cost. Moreover, classical solvers heavily rely on hand-crafted regularizers, which struggle to fully capture the complex distributions of natural images. 
	In contrast, recent studies have shown that data-driven deep learning methods can well avoid the aforementioned difficulties, and directly learn $\Psi^*$ through the framework in \eqref{DL-Model}. 
	
	\textbf{Theoretical Linkage Assumption.}
	\textcolor{black}{
		The empirical success of deep networks demonstrates their powerful function approximation capabilities, which are supported by universal approximation results \cite{hornik1989multilayer} and training convergence analyses of over-parameterized models \cite{allen2019convergence}. Moreover, recent studies suggest that learned image restoration models can be characterized as idealized operators minimizing an implicit objective \cite{hu2024restoration}.
		Motivated by these studies, we reasonably assume the trained supervised network is an approximation to the optimal restoration operator defined by the variational formulation, leading to the following assumption.	
		\begin{assumption}[Ideal Approximation]\label{assum}
			Consider a finite dataset $\{\hat{Y}_n,\hat{X}_n\}_{n=1}^{N}$. Let $\tilde{\Phi}$ denote the optimal restoration operator defined by the variational model \eqref{Modelall_dis}, and let $\Psi^*$ denote the supervised restoration network adequately trained through \eqref{DL-Model}.
			We assume that the following approximation holds for each observed input $\hat{Y}_n$ and its transformations under the considered transformation groups. Specifically, with a common bound $\eta\geq0$:
			\begin{equation}\label{assu}
				\left\| \tilde{\Phi}(\hat{Y}_n) - \Psi^*(\hat{Y}_n)\right\|_F \leq \eta,
				\qquad
				\forall  n=1,\ldots,N,
			\end{equation}
		\end{assumption}
	}

		\textcolor{black}{This assumption connects the optimization problem solution with modern deep learning practice. Building on this assumption and the equivariance properties of the inverse problem solver established in Section \ref{sec:eq-analysis}, we can further derive upper risk bounds for deep networks, as detailed in Section \ref{sec:risk-analysis}.
		For conciseness and readability, we present the risk analysis for $\eta = 0$ in the later sections.
		The corresponding bounds and proofs for nonzero approximation error ($\eta \neq 0$) are provided in Appendix~{\ref{app:supp:nonzero-eta}}.}

	\textbf{Definition of Equivariant Network.}
	%When restricting the feasible region to an equivariant subset $\Lambda_{EQ}\subset \Lambda$: 
We consider the following sample-adaptive equivariant network set:
	\begin{equation}\label{eqNet}
		\Lambda_{EQ} = \{\Psi\mid \Psi(\tilde{\pi}_{A_n}(\hat{Y}_n)) = \tilde{\pi}_{A_n}(\Psi(\hat{Y}_n)), \forall A_n\in G_n, n=1,2,\cdots, N\},
	\end{equation}
	the deep learning framework \eqref{DL-Model} characterizes the training process of equivariant networks,
	where $\tilde{\pi}_{A_n}$ denotes the transformation action in the discrete domain, which satisfies:
	\begin{equation}\label{disTrans}
		(\tilde{\pi}_{A_n}(\hat{Y}_n))_{ij}=\pi_{A_n}[\hat{r}_n](x_{ij}).
	\end{equation}
In this paper, we will investigate the properties of the learned equivariant result $\Psi_{EQ}^*$ in this scenario. Specifically, we analyze its empirical and expected risks and provide a theoretical comparison against the standard optimal network $\Psi^*$.

It should be noted that in \eqref{eqNet}, the transformation group $G_n$ is allowed to vary with different data indices $n$, this is because the objective of our study includes equivariant networks whose transformation group adapts to the data sample. In general, $G_n$ is restricted to a single, fixed group for all samples, in which case our theoretical framework  degenerates to previous network architectures.	
	
	\section{Equivariance Analysis of Restoration Operator}\label{sec:eq-analysis}
	
	In this section, we analyze the equivariance of the optimal restoration operators derived from the optimization problems in \eqref{Modelall} and \eqref{Modelall_dis}, establishing the relationship between the operator's equivariance and the data symmetry defined in Section~\ref{sec:symmetry}. 
	
	\subsection{Equivariance Analysis in Continuous Domain}
	To facilitate this analysis, we first provide the following lemmas and assumption, while the proofs of lemmas are provided in the appendices.
	
	\begin{lemma}\label{varphi}
		For continuous image functions $\hat{r}(x), r(x)\in C^\infty(\mathbb{R}^2)$,  assuming the following filter optimization admits a unique solution:
		\begin{equation}\label{phi}
			\varphi[\hat{r},r]=\arg \min_{k}\int_{\R^2}\left(\hat{r}(x)-\int_{\R^2}k(t)r(x-t)dt\right)^{2}dx,
		\end{equation}
		then it satisfies $\pi_A\left[\varphi[\hat{r},r]\right]=\varphi[\pi_A[\hat{r}],\pi_A[r]]$, where $\pi_A$ denotes how the transformation acts on images, i.e., $\pi_A[r](x) = r\left(A^{-1}x\right)$, with $\vert{}\det A\vert{}=1$.
	
	\end{lemma}

\begin{lemma}[Norm-Preserving]\label{norm-preserve}
		For a function $r(x)\in C^\infty(\mathbb{R}^2)$ and the transformation $A \in G$ satisfying $|\det(A)| = 1$, the $L_2$ norm is invariant under the group action $\pi_A$. That is:
		\begin{equation}
			\|r(x)\|_2^2=\|\pi_{A}[r](x)\|_2^2.
		\end{equation}
	\end{lemma}

We then introduce the following fundamental assumption regarding the local quadratic growth condition\cite{bonnans2013perturbation, fernandez2022augmented}:
\begingroup
\color{black}	
\begin{qgassumption}[Local Quadratic Growth]\label{ass:qg} For  dataset $\{\hat{r}_n\}_{n=1}^N$,
    define the continuous objective functional of the restoration problem in \eqref{Modelall}  as:
    \begin{equation}\label{F_conti}
        F[r_n]=\min_k\int_{\mathbb{R}^2}
        \left(\hat{r}_n(x)-\int_{\mathbb{R}^2}k(t)r_n(x-t)dt\right)^2dx+R[r_n].
    \end{equation}
     Let $r_n^*=\Phi[\hat r_n]$ such that $\{r_n^*\}_{n=1}^N$ is the optimal solution of the problem. We assume the inner optimization problem with respect to $k$ admits a unique solution, and there exists a constant $\mu>0$ and a neighborhood $\mathcal{N}$ around $\{r_n^*\}_{n=1}^N$ such that the $\mu$-quadratic growth condition 
    \[
        \sum_{n=1}^{N}\bigl(F[r_n]-F[r_n^*]\bigr)
        \geq\mu\sum_{n=1}^{N}\|r_n-r_n^*\|_2^2
    \]
    holds for every feasible collection $\{r_n\}_{n=1}^N\in \mathcal{N}$. {The neighborhood $\mathcal{N}$ contains continuous representations of the selected discrete minimizing collection ($\{r_n\}_{n=1}^N, s.t., r_n(x_{ij})=(\tilde\Phi(\hat Y_n))_{ij}$), $\left\{\pi_{A_n}^{-1}\!\left[\Phi[\pi_{A_n}[\hat r_n]]\right]\right\}_{n=1}^{N}$ for all $A_n\in G_n$,  and the selected semi-discrete minimizing collection\footnote{\textcolor{black}{The explicit definition of the semi-discrete approximation $\tilde{F}$ is provided in the appendices. By construction, it closely approximates the continuous objective $F$. We assume that this neighborhood inclusion condition holds uniformly for the minimizers of the semi-discrete approach across all considered discretization mesh sizes $h$.}}, for all $0<h\leq h_0$.}
\end{qgassumption}

It is worth noting that this local quadratic growth condition is a standard and intuitive simplifying assumption widely used to characterize the local behavior of objective functions. Analyzing the theoretical properties of the restoration operator under this premise is sufficient to provide us with valuable heuristic insights into the intrinsic characteristics of the model.
Based on the preceding lemmas and Assumption~\ref{ass:qg}, we now establish the equivariance error bound for $\Phi[\cdot]$ under the relaxed symmetry constraint \eqref{symmetry3}.	

\endgroup	
	
	\begin{theorem}[Equivariance Error of Restoration Operator]\label{Error-IR}
		Consider a degraded image dataset $\{\hat{r}_n\}_{n=1}^{N}$ and the corresponding restoration problem defined in (\ref{Modelall}), $\Phi[\cdot]$ denotes the optimal restoration operator and $\Omega_\varepsilon$ is the feasible set defined under the relaxed symmetry constraint in \eqref{symmetry3}. 
		{If the $\mu$-quadratic growth condition stated in Assumption~\ref{ass:qg} holds for $\{\pi_{A_n}[\hat{r}_n]\}_{n=1}^N$, $\forall A_n\in G_n,n=1,2,\cdots,N$, the following equivariance error bound holds:}
		\begin{equation}
			\frac{1}{N}\sum_{n=1}^N\|\Phi[\pi_{A_n}[\hat{r}_n]]-\pi_{A_n}[\Phi[\hat{r}_n]]\|_2^2\leq \frac{2\varepsilon}{\mu},
		\end{equation}
		where $\pi_{A_n}$ denotes how the transformation $A_n$ acts on image function, i.e., 
		$\pi_{A_n}[r_n](x) \!\!=\! r_n\left(A_n^{-1}x\right), \forall x\in\mathbb{R}^2$.
	\end{theorem}

	\begin{proof}
		
		By the definition of $\Phi$ in \eqref{Modelall}, we have:
		\begin{equation}\label{f_pi2}
			\begin{aligned}
				\left\{\Phi[\pi_{A_n}[\hat{r}_n]]\right\}_{n=1}^N
				\!=\! \arg\min_{\{r_n\}\in\Omega_\varepsilon}
				\sum_{n=1}^N \!\left(\!\min_k\!\!\int_{\mathbb{R}^2}
				\!\left(\hat{r}_n(A_n^{-1}x)\!-\!\!\int_{\mathbb{R}^2}\!\!k(t)r_n(x\!-\!t)dt\right)^2 \!dx\!+\!R[r_n]\right).
			\end{aligned}
		\end{equation}
\begingroup
\color{black}	
		Let $\bar{r}_n=\pi_{A_n}^{-1}[r_n]$, $\bar{t}=A_n^{-1}t$, and $\bar{x}=A_n^{-1}x$. For any $\{r_n\}_{n=1}^{N}\in\Omega_\varepsilon$ and $B_n,C_n\in G_n$, we have:
			\begin{equation}
				\begin{split}
					\frac{1}{N}\left|\sum_{n=1}^{N}R[\pi_{B_n}[\bar{r}_n]]-\sum_{n=1}^{N}R[\pi_{C_n}[\bar{r}_n]]\right|
					=\frac{1}{N}\left|\sum_{n=1}^{N}R[\pi_{B_nA_n^{-1}}[r_n]]-\sum_{n=1}^{N}R[\pi_{C_nA_n^{-1}}[r_n]]\right|
					\leq\varepsilon.
				\end{split}
			\end{equation}
			The last inequality follows from \eqref{symmetry3}, since $B_nA_n^{-1},C_nA_n^{-1}\in G_n$. Thus $\{\bar{r}_n\}_{n=1}^{N}\in \Omega_\varepsilon$. Applying the same argument with $A_n^{-1}$ in place of $A_n$ gives the converse, so the change of variables preserves the feasible set.
\endgroup		
		
		Define $\pi_{A_n}\circ\{r_n\}_{n=1}^N=\{\pi_{A_n}[r_n]\}_{n=1}^N$, and using $\vert{}\det(A_n)\vert{}=1$ stated in Assumption~\ref{ass:adaptive_group}, it follows that:
		\begin{equation}\label{f_pi3}
			\begin{split}
				& \arg \!\min_{\{r_n\}\in \Omega_\varepsilon}\sum_{n=1}^N \!\!\left(\!\int_{\R^2}\!\!\!\left(\!\!\hat{r}_n(A_n^{-1}x)\!-\!\!\!\int_{\R^2}\!\!\varphi[\pi_{A_n}[\hat{r}_n], \!r_n](t)r_n(x\!-\!t)dt\!\!\right)^{2}\!\!\!dx\!+\!R[r_n]\right)\\
				=& \pi_{A_n}\!\!\circ\!\arg \!\min_{\{\bar{r}_n\}\in \Omega_\varepsilon}\!\sum_{n=1}^N \!\!\Big(\!\!\int_{\R^2}\!\!\!\left(\!\!\hat{r}_n(\!A_n^{-1}\!x)\!\!-\!\!\!\int_{\R^2}\!\!\!\varphi[\pi_{A_n}\![\hat{r}_n], \!\pi_{A_n}\![\bar{r}_n]](t)\bar{r}_n\!(\!A_n^{-1}\!(x\!\!-\!t))dt\!\!\right)^2\!\!\!\!dx\!\!+\!\!R[\pi_{A_n}\![\bar{r}_n]]\!\Big)\\
				=& \pi_{A_n}\!\!\circ\!\arg\! \min_{\{\bar{r}_n\}\in \Omega_\varepsilon}\!\sum_{n=1}^N \!\Big(\int_{\R^2}\!\!\!\left(\!\!\hat{r}_n(A_n^{-1}x)\!\!-\!\!\!\int_{\R^2}\!\!\!\!\!\pi_{A_n}[\varphi[\hat{r}_n, \bar{r}_n]](t)\bar{r}_n(\!A_n^{-1}\!(\!x\!-\!t))dt\!\right)^2\!\!\!dx\!+\!R[\pi_{A_n}[\bar{r}_n]]\Big)\\
				=& \pi_{A_n}\!\!\circ\!\arg \min_{\{\bar{r}_n\}\in \Omega_\varepsilon}\!\sum_{n=1}^N \!\Big(\int_{\R^2}\!\!\left(\!\!\hat{r}_n(\bar{x})\!-\!\!\int_{\R^2}\varphi[\hat{r}_n, \bar{r}_n](\bar{t})\bar{r}_n(\bar{x}-\bar{t})d\bar{t}\right)^2\!d\bar{x}+\!R[\pi_{A_n}[\bar{r}_n]]\Big)\\
				=& \pi_{A_n}\!\!\circ\arg \min_{\{\bar{r}_n\}\in \Omega_\varepsilon}\left\{\sum_{n=1}^N F[\bar{r}_n]+\sum_{n=1}^N\left(R[\pi_{A_n}[\bar{r}_n]]-R[\bar{r}_n]\right)\right\}.
			\end{split}
		\end{equation}

Therefore, we have:
		\begin{equation}\label{ccc}
			\begin{aligned}
				\left\{\pi_{A_n}^{-1}[\Phi[\pi_{A_n}[\hat{r}_n]]]\right\}_{n=1}^N
				= \arg\min_{\{\bar{r}_n\}\in\Omega_\varepsilon}
				\left\{\sum_{n=1}^NF[\bar{r}_n]
				+\sum_{n=1}^N\left(R[\pi_{A_n}[\bar{r}_n]]-R[\bar{r}_n]\right)\right\}.
			\end{aligned}
		\end{equation}
		
		By the definition of `$\arg \min$', substituting $\pi_{A_n}^{-1}[\Phi[\pi_{A_n}[\hat{r}_n]]]$ into the objective function of the above method and combining it with the optimality condition of $\Phi[\hat{r}_n]$, we obtain:
		\begin{equation}\label{ddd}
			\begin{split}
				&\sum_{n=1}^N F\left[\pi_{A_n}^{-1}[\Phi[\pi_{A_n}[\hat{r}_n]]]\right] + \sum_{n=1}^N\left(R\left[\pi_{A_n}[\pi_{A_n}^{-1}[\Phi[\pi_{A_n}[\hat{r}_n]]]]\right]- R\left[\pi_{A_n}^{-1}[\Phi[\pi_{A_n}[\hat{r}_n]]]\right]\right) \\
				\leq & \sum_{n=1}^N F[\Phi[\hat{r}_n]] + \sum_{n=1}^N\left(R[\pi_{A_n}[\Phi[\hat{r}_n]]]- R[\Phi[\hat{r}_n]]\right).
			\end{split}
		\end{equation}
		Then we get:
		\begin{equation}\label{ddd2}
			\begin{split}
				&\sum_{n=1}^N \left( F[\pi_{A_n}^{-1}[\Phi[\pi_{A_n}[\hat{r}_n]]]]-F[\Phi[\hat{r}_n]] \right)\\
				\leq& \sum_{n=1}^N\left(R[\pi_{A_n}[\Phi[\hat{r}_n]]]\!-\! R[\Phi[\hat{r}_n]]\right) - \sum_{n=1}^N\left(R[\pi_{A_n}[\pi_{A_n}^{-1}[\Phi[\pi_{A_n}[\hat{r}_n]]]]]\!-\!R[\pi_{A_n}^{-1}[\Phi[\pi_{A_n}[\hat{r}_n]]]]\right)\\
				\leq& \l| \sum_{n=1}^N\!\left(R[\pi_{A_n}[\Phi[\hat{r}_n]]] \!-\! R[\Phi[\hat{r}_n]]\right)\r| \!\!+\!\! \l| \sum_{n=1}^N\!\left(R[\pi_{A_n}[\pi_{A_n}^{-1}[\Phi[\pi_{A_n}[\hat{r}_n]]]]]\!-\! R[\pi_{A_n}^{-1}[\Phi[\pi_{A_n}[\hat{r}_n]]]]\right)\r|\\
				\leq & N \varepsilon + N \varepsilon 
				=  2N \varepsilon.
			\end{split}
		\end{equation}
		Since $[\Phi[\hat{r}_n]]$ and $ \left[\pi_{A_n}^{-1}[\Phi[\pi_{A_n}[\hat{r}_n]]]\right]\in \Omega_{\varepsilon}$ satisfy the relaxed symmetry constraint defined in \eqref{symmetry3}, the final inequality can thus be derived by taking $B_n$ to be the identity transformation.
		
		By Assumption~\ref{ass:qg}, the solutions in \eqref{ccc} belong to the quadratic growth neighborhood of $\{\Phi[\hat r_n]\}_{n=1}^N$. Applying the stated inequality and the norm-preserving property stated in Lemma \ref{norm-preserve} , we have:
		\begin{equation}\label{quadratic2}
			\begin{split}
				\frac{1}{N}\sum_{n = 1}^N\|\Phi[\pi_{A_n}[\hat{r}_n]]\!-\!\pi_{A_n}[\Phi[\hat{r}_n]]\|_2^2 
				=& \frac{1}{N}\sum_{n = 1}^N\|\pi_{A_n}^{-1}[\Phi[\pi_{A_n}[\hat{r}_n]]]\!-\!\Phi[\hat{r}_n]\|_2^2 \\ 
				\leq& \frac{1}{N\mu}\sum_{n=1}^N {\Big(}F\left[\pi_{A_n}^{-1}[\Phi[\pi_{A_n}[\hat{r}_n]]]\right]\!-\!F\left[\Phi[\hat{r}_n]\right]{\Big)} 
				\leq \frac{2\varepsilon}{\mu}.
			\end{split}
		\end{equation}
		The proof is completed.
	\end{proof}

Theorem~\ref{Error-IR} reveals a fundamental property in the continuous domain: the equivariance of the optimal restoration operator is intrinsically governed by data symmetry. Specifically, the equivariance error of the restoration operator is upper bounded by the dataset symmetry error $\varepsilon$.

	\subsection{Equivariance Analysis in Discrete Domain}\label{sec:discrete-analysis} 
	While the continuous domain analysis provides fundamental theoretical results, practical applications inevitably involve discretization. To extend our conclusions to the discrete domain, it is imperative to account for errors introduced during the discretization process.

\begingroup\color{black}
For the discretization analysis, let $D_r,D_k,D_R>0$
denote fixed support-radius bounds for the continuous
observed and restored images ($\hat r$ and $r$), the blur kernel $k$,
and the regularization kernel $k^R$, respectively.
These bounds also apply to their admissible transformed versions.
Let $h>0$ denote the mesh size of the restored-image,
blur-kernel, and regularization quadrature grids.
The observed-image grid has mesh size $sh$, where $s>0$
is the fixed downsampling factor.
We use fixed square sampling windows satisfying
\begin{equation}\label{fixed-windows}
    m_rh=msh=2(D_r+D_k),\qquad
    ph=2D_k,\qquad
    m_Rh=2(D_r+D_R),
\end{equation}
where $m$ and $m_r$ are the numbers of grid points in each
coordinate of the observed and restored images, respectively,
$p$ and $m_R$ are the corresponding numbers for the blur-kernel
grid and the regularization quadrature grid.
\par\endgroup

	We first present the following important lemmas, which provide an upper bound on the discretization error between the continuous integral formulation and its corresponding discrete Riemann sum. The detailed proofs are provided in the appendices.

\begin{lemma}[Discretization Error \cite{xie2022fourier}]\label{tan}
Let $f\in C^2(\mathbb R^2)$ satisfy $f(x)=0$ for $\|x\|\geq D$, where $D>0$ is fixed, and
\begin{equation}\label{conditionLemma}
    \max_{a=1,2}|\partial_{aa}f(x)|\leq C,\qquad x\in\mathbb R^2.
\end{equation}
On the cell-centered grid of $[-D,D]^2$ with $Mh=2D$, one has
\begin{equation}\label{Lemma33}
    \left|\int_{\mathbb R^2}f(x)\,dx-
    \sum_{(i,j)\in\Delta_M}f(x_{ij})h^2\right|\leq 2CD^2h^2.
\end{equation}
Here $\Delta_M=\{1,\ldots,M\}^2$ and
$x_{ij}=((i-(M+1)/2)h,(j-(M+1)/2)h)^T$.
\end{lemma}

\begingroup\color{black}
\begin{lemma}[Discretization Error of Regularization Term]\label{error-R}
		Consider a convolution-based local feature regularization functional $R$ {and its discretization on the fixed regularization window in \eqref{fixed-windows}. The continuous functional is defined as:}
		\begin{equation}\label{R-conv}
			{R[r]=\frac{1}{m_R^2h^2}\int_{\mathbb{R}^2}
				\rho\!\left(\int_{\mathbb{R}^2}k^{R}(x)r(x-y)dx\right)dy,}
		\end{equation}
		{where $\rho:\mathbb{R}\to[0,\infty)$ is continuous, satisfies $\rho(0)=0$, and is twice continuously differentiable except on at most $M_\rho$ points, with $M_\rho$ a nonnegative integer independent of $h$; $r\in C^2(\mathbb{R}^2)$ is the input image function, and $k^R\in C^2(\mathbb{R}^2)$ is the convolution kernel function.}
		{For a discrete image $Y$ with admissible representation $r$ as in Appendix~\ref{app:supp:semi-discrete-setting}, define}
		\textcolor{black}{
			\begin{equation}\label{R-dis}
				{\tilde{R}(Y)=\frac{1}{s^2m_R^2h^2}
					\sum_{(\hat{i},\hat{j})\in\Delta_R}
					\rho\!\left(\sum_{(i,j)\in\Delta_R} k^{R}(x_{ij})
					r(x_{ij}-x_{\hat{i}\hat{j}})h^2\right),}
			\end{equation}
			{where $s>0$ is the sampling rate,}
			$\Delta_R=\{(i,j)\mid i,j=1,2,\ldots,m_R\}$ and $x_{ij}=\big((i-(m_R+1)/2)h,(j-(m_R+1)/2)h\big)^T$, and $\forall\|x\|\geq D_R,~k^R(x)=0, \forall\|x\|\geq D_r,~r(x)=0.$}
		Assume $r(x)$, $k^R(x)$, and their gradients and Hessian matrices are bounded, i.e., 
		$\|r(x)\|\leq F_r,\allowbreak \|k^R(x)\|\leq F_k,\allowbreak \|\nabla {r}(x)\|\leq G_r,\allowbreak \|\nabla k^R(x)\| \leq G_k,\allowbreak \|\nabla^2 {r}(x)\| \leq H_r,\allowbreak \|\nabla^2 k^R(x) \|\leq H_k$.
		{Assume $0<h\leq h_0$ and $|\rho'(z)|\leq G_\rho$, $|\rho''(z)|\leq H_\rho$ at every smooth point with $|z|\leq4F_kF_r(D_r+D_R)^2$, with all bounds independent of $h$.}
		Then the discretization error satisfies:
		\begin{equation}\label{Lemma35}
			{\left|R[r]-(sh)^2\tilde{R}(Y)\right|\leq C_{\mathrm{reg}}h^2,}
		\end{equation}
		\begingroup
		where the constant $C_{\mathrm{reg}}$ is independent of $h$, determined by the fixed support radii, the stated function and derivative bounds, and $M_\rho$. Its explicit formulation is given in \eqref{app:supp:reg-constant} in the appendices.
		\endgroup
	\end{lemma}
%The normalization satisfies $m_R^2h^2=4(D_r+D_R)^2$, the fixed area of the regularization window. The support of $k^R(x)r(x-y)$ in $y$ is contained in the disk of radius $D_r+D_R$, so no response is truncated. In particular,
%\[
% (sh)^2\tilde R(Y)=\frac1{m_R^2}
% \sum_{(\hat i,\hat j)\in\Delta_R}
% \rho\!\left(\sum_{(i,j)\in\Delta_R}k^R(x_{ij})r(x_{ij}-x_{\hat i\hat j})h^2\right).
%\]
%Thus $\tilde R$ remains a functional of $Y$ under the prescribed representation convention, whereas $R$ acts on the continuous image. Neither the definition nor the error estimate requires $\rho'(0)=0$.
%

	In the subsequent theoretical analysis, we consider the regularizers $R$ and $\tilde{R}$ in \eqref{R-conv} and \eqref{R-dis} under the conditions of Lemma~\ref{error-R}.\footnote{\textcolor{black}{{Examples satisfying these conditions include $\rho(z)=z^2$, $\rho(z)=\sqrt{1+z^2}-1$, and $\rho(z)=|z|$.} Finite sums of such functionals can be treated by adding their error bounds.}} {Within this class, the specific choice of the regularizer affects only the constant factors in the regularization discretization error, without changing its $O(h^2)$ order or the form of the subsequent estimates under their stated assumptions.
		{The estimate uses the fixed support bounds, the first and second derivative bounds on the smooth pieces of $\rho$, and the corresponding image and kernel bounds.}
		Moreover, the convolution-based formulation accommodates different local feature penalties through the choices of $k^R$ and $\rho$, making it a representative setting for the analysis.}
\par\endgroup

	Under the discretization conditions detailed in Appendix~\ref{app:supp:semi-discrete-setting}, we now establish the equivariance error bound for the optimal restoration operator in the discrete domain.

\begin{theorem}[Equivariance Error of Discrete Operator]\label{Error-dis}
Consider the image restoration problem defined in \eqref{Modelall_dis}
with feasible set $\tilde{\Omega}_\varepsilon$ in \eqref{symmetry3_dis}
and observed dataset $\{\hat{Y}_n\}_{n=1}^{N}$.
\textcolor{black}{Let $Y_n$, $K$, $\tilde{\Phi}(\cdot)$, and $R$ represent
the restored image, the blur kernel, the optimal restoration operator,
and the continuous regularization functional, respectively, where $R$
is defined in \eqref{R-conv} with kernel $k^R$ and response function $\rho$.
\textcolor{black}{Let $\hat r_n$ be an admissible continuous representation
of $\hat Y_n$. Assume Assumption~\ref{ass:adaptive_group} and the conditions
of Lemma~\ref{error-R} hold. For every compatible mesh $0<h\leq h_0$
and every choice of $A_n\in G_n$, $n=1,\ldots,N$,
Assumption~\ref{ass:qg} holds for $\{\pi_{A_n}[\hat r_n]\}_{n=1}^{N}$
with the same constant $\mu>0$.}
Assume there exist fixed constants
$F_i,G_i,H_i>0$, $i=1,2,3$, such that the following support
and boundedness conditions hold for all such meshes and
transformations, every feasible continuous blur kernel $k$
in the inner minimizations, every $n=1,\ldots,N$, and all
$x\in\mathbb R^2$:
		\begin{equation}\label{bound_condition}
			\begin{split}
				& |r(x)| \leq F_1, ~|k(x)|\leq F_2, ~|k^R(x)|\leq F_3,\\
				&\|\nabla r(x)\| \leq G_1, ~\|\nabla k(x)\|\leq G_2, ~\|\nabla k^R(x)\|\leq G_3, \\
				& \|\nabla^2 r(x) \| \leq H_1, ~\|\nabla^2 k(x) \|\leq H_2, ~\|\nabla^2 k^R(x) \|\leq H_3,\\
				&\forall\|x\|\geq D_r,~r(x)=0, \forall\|x\|\geq D_k,~k(x)=0, \forall\|x\|\geq D_R,~k^R(x)=0,
			\end{split} 
		\end{equation}
		where $r$ ranges over $\pi_{A_n}[\hat r_n]$,
$\Phi[\pi_{A_n}[\hat r_n]]$, and the corresponding admissible
continuous representations of
$\tilde{\pi}_{A_n}(\tilde{\Phi}(\hat Y_n))$ and
$\tilde{\Phi}(\tilde{\pi}_{A_n}(\hat Y_n))$.}
Here $\nabla$ and $\nabla^2$ denote the gradient and Hessian,
respectively.
The grid sizes and fixed support radii satisfy
\eqref{fixed-windows}.
		Then, for any transformation $A_n \in G_n$, the equivariance error satisfies:
		\begingroup
		\begin{equation}\label{final}
			\frac{1}{Nm_r^2}\sum_{n=1}^N \left\| \tilde{\Phi}\big(\tilde{\pi}_{A_n}(\hat{Y}_n)\big) - \tilde{\pi}_{A_n}\big(\tilde{\Phi}(\hat{Y}_n)\big) \right\|_F^2 \leq Ch^2 + \hat{C}\varepsilon,
		\end{equation}
		\endgroup
		where $\varepsilon$ denotes the symmetry error, $m_r$ is the number of grid points in each
coordinate of the restored image,
$\tilde{\pi}_{A_n}$ denotes how the transformation $A_n$ acts on image $\hat{Y}_n$, $\hat C=\nicefrac{3}{2\mu(D_r+D_k)^2}$,  $
 C=\nicefrac{(4\mu+12(1+s^2)(1+4D_k^2(F_2+G_2+H_2))^2)(F_1^2+G_1^2+F_1H_1)}{\mu} +\nicefrac{3C_{\mathrm{reg}}}{\mu(D_r+D_k)^2}$,
and $C_{\mathrm{reg}}$ is defined in Lemma \ref{error-R}.
	\end{theorem}

\begingroup\color{black}

	\begin{proof}[Proof Sketch]
		{	We first rescale the discrete objective in \eqref{Modelall_dis} by $(sh)^2$, which preserves its minimizers while aligning its scale with the continuous formulation. By defining an intermediate semi-discrete objective that discretizes only the data fidelity term, we establish an $O(h^2)$ discrepancy from the fully continuous objective at the continuous minimizers and the selected continuous representations of the discrete minimizers. Furthermore, Lemma~\ref{error-R} bounds the discretization error of the regularizer, $\vert{}R[r]-(sh)^2\tilde R(Y)\vert{}$, yielding a total discrepancy of $O(h^2)$ between the continuous and scaled discrete objectives. Leveraging the optimality condition and Assumption~\ref{ass:qg}, we can then bound the squared $L^2$ distance between the continuous minimizers and their discrete counterparts. Combining this with Theorem~\ref{Error-IR} controls their equivariance error. Finally, Lemma~\ref{tan} converts this continuous $L^2$ estimate into the pixel-averaged squared Frobenius norm by dividing it by the fixed output-window area $m_r^2h^2=4(D_r+D_k)^2$. The detailed proof is provided in the appendices.}	
	\end{proof}
\par\endgroup
	
	Theorem \ref{Error-dis} indicates that the optimal restoration operator maintains robust equivariance in the discrete domain corresponding to practical applications. Its equivariance error is strictly upper-bounded by a controllable limit that depends on the discretization mesh size $h$ and the data symmetry error $\varepsilon$.
Consequently, under the stated uniform assumptions and fixed physical windows, mesh refinement reduces the discretization contribution to the bound. The bound tends to zero when $h,\varepsilon\to0$.

\section{Risk and Generalization Analysis for Equivariant Networks}\label{sec:risk-analysis}
	
	\subsection{Empirical Risk of Equivariant Networks}\label{em-risk}
	
	The analysis in the above sections establishes the inherent equivariance of the restoration operator  $\tilde{\Phi}$ achieved from the dataset-level inverse model. This crucial finding provides a compelling justification for employing equivariant networks to approximate this operator.  
	In this section, we investigate the empirical risk of equivariant networks, aiming to explore the factors that influence it. 
	
	Consider a training dataset $D_N = \{(\hat{X}_n, \hat{Y}_n) \mid n = 1,2,\cdots N\}$, where $\hat{X}_n$ and $\hat{Y}_n$ represent the paired high quality ground truth and degraded observations, respectively. \textcolor{black}{The pixel-averaged empirical risk $\EM$ of a deep network $\Psi$ is defined as:}
	\begin{equation}\label{emp_risk}
		\EM\left(\Psi\right) = \textcolor{black}{\frac{1}{Nm_r^2}}\sum_{n=1}^{N}\|\hat{X}_n -\Psi(\hat{Y}_n)\|_F^2.
	\end{equation}
	Then the training of equivariant networks is to achieve $\Psi^*_{EQ} = \mathop{\arg\min}_{\Psi\in \Lambda_{EQ}} \EM(\Psi)$, where $\Lambda_{EQ}$ is defined in \eqref{eqNet}.
	Before presenting our major conclusion, we first provide the following preliminary conclusion\footnote{The proof of Lemma \ref{bar_f} is provided in the appendices.}.
%\begin{lemma}\label{bar_f}
%	\textcolor{black}{\hypersetup{allcolors=red}For each $n$, let $\mathcal{H}(Y_n) = G_n$ define the mapping from an input image to its adaptive symmetry group. Let $A_n^t$ denote the $t$-th element of $\mathcal{H}(Y_n)$, where $T = |\mathcal{H}(Y_n)|$, and the group action $\tilde{\pi}_{A_n^t}$ denotes how the transformation $A_n^t$ acts on an image. Then, for any mapping $\Psi$ between these spaces, let $\bar{\Psi}(\hat{Y}_n)=\frac{1}{T}\sum_{t=1}^{T}\tilde{\pi}_{A_n^t}[\Psi[\tilde{\pi}_{A_n^t}^{-1}]](\hat{Y}_n)$. Then $\bar{\Psi}$ is equivariant with respect to $G_n$, in the sense that for any $t=1,\cdots,T$:}
%	% \vspace{-3mm}
%	\begin{equation}
%		\textcolor{black}{\bar{\Psi}(\tilde{\pi}_{A_n^t}(Y_n)) = \tilde{\pi}_{A_n^t}(\bar{\Psi}(Y_n)).}
%	\end{equation}
%\end{lemma}

\begin{lemma}\label{bar_f}
    \textcolor{black}{
    Under Assumption~\ref{ass:adaptive_group} with $G_n=\mathcal H_G(\hat Y_n)$ and $T = |G_n|, \forall n = 1,2,\cdots,N$, for any mapping $\Psi$ between the image spaces, define
    $\bar{\Psi}(\hat{Y}_n)=\frac{1}{T}\sum_{A\in\mathcal H_G(\hat{Y}_n)}
    \tilde{\pi}_{A}\bigl(\Psi(\tilde{\pi}_{A}^{-1}(\hat{Y}_n))\bigr)$.
    Then, for any $A\in G_n$, we have:}
    \begin{equation}
        \textcolor{black}{
        \bar{\Psi}(\tilde{\pi}_{A}(\hat Y_n))
        =\tilde{\pi}_{A}(\bar{\Psi}(\hat Y_n)).
        }
    \end{equation}
\end{lemma}

\textcolor{black}{
The operation $\bar\Psi$ in Lemma~\ref{bar_f} is a sample-adaptive variant of the widely known Reynolds operator \cite{puny2021frame} and provides a simple and intuitive way to construct an equivariant mapping directly from any standard network $\Psi$. 
Under Assumption~\ref{ass:adaptive_group}, this sample-adaptive Reynolds construction defines a single mapping through the common group-assignment rule $\mathcal H_G$. By applying this construction independently to each sample $\hat{Y}_n$ and its correlated group $G_n$, the resulting mapping $\bar{\Psi}$ is strictly equivariant with respect to the corresponding $G_n$ for each $\hat{Y}_n$, i.e., $\bar{\Psi}\in \Lambda_{EQ}$. However, explicitly executing this construction requires a priori knowledge of the optimal $G_n$ for every sample, which is typically intractable in real-world applications. Therefore, rather than serving as a practical algorithmic step, this mechanism is employed purely as a theoretical construct. It effectively projects a general model onto an analytically tractable mapping within $\Lambda_{EQ}$, thereby serving as a crucial mathematical bridge to compare the adequately trained standard network $\Psi^*$ and the optimal equivariant network $\Psi^*_{EQ}$.
} Building upon this, we can deduce the following conclusion: 

	\begin{theorem}[Empirical Risk of Equivariant Network]\label{ER}
		Consider the image restoration problem defined in (\ref{Modelall_dis}) with feasible set $\tilde{\Omega}_\varepsilon$ \eqref{symmetry3_dis} and observed dataset $\{\hat{Y}_n\}_{n=1}^{N}$.
		Suppose the assumptions of Theorem~\ref{Error-dis} hold, with the same constants, for every transformed dataset $\{\tilde\pi_{A_n}(\hat Y_n)\}_{n=1}^{N}$, $A_n\in G_n$, $n=1,\ldots,N$.
		Then the empirical risks of the standard and equivariant networks satisfy
		\begingroup
		\color{black}
		\begin{equation}\label{eq:empirical-bound}
			\EM(\Psi^*_{EQ})
			\leq\EM(\Psi^*)+C_h h+C_\varepsilon\sqrt{\varepsilon},
		\end{equation}
		\endgroup
		where $\varepsilon$ denotes the symmetry error,  $h$ denotes the mesh size, with $0<h<h_0$, $\Psi^*\in\Lambda$ and $\Psi^*_{EQ}\in\Lambda_{EQ}$ denote the adequately trained standard and equivariant networks, respectively, with $\Psi^*_{EQ}$ minimizing $\EM$ over $\Lambda_{EQ}$,            $C_h=2\sqrt{C\,\EM(\Psi^*)}+Ch, 
            C_\varepsilon=2\sqrt{\hat C\,\EM(\Psi^*)}+\hat C\sqrt{\varepsilon}$,
        where $C$ and $\hat C$ are the constants defined in Theorem~\ref{Error-dis}. 
	\end{theorem}

\begin{proof}
		\color{black}
%		Since each $\bar{\Psi}_n$ (Lemma~\ref{bar_f}) is equivariant with respect to $G_n$, the sample-wise outputs $\{\bar{\Psi}_n(\hat Y_n)\}_{n=1}^N$ are realized by an equivariant candidate on the training set. As $\Psi^*_{EQ}$ minimizes the empirical risk over $\Lambda_{EQ}$, we have:
		{Let $\bar{\Psi}(\hat{Y}_n)=\frac{1}{T}\sum_{A\in\mathcal H_G(\hat{Y}_n)}
    \tilde{\pi}_{A}\bigl(\Psi^*(\tilde{\pi}_{A}^{-1}(\hat{Y}_n))\bigr)$. By Lemma~\ref{bar_f}, $\bar\Psi\in\Lambda_{EQ}$.
     Since $\Psi^*_{EQ}$ minimizes the empirical risk over $\Lambda_{EQ}$, we have:}
		\begin{equation}\label{eq:risk-candidate}
			\EM(\Psi^*_{EQ})\leq{\frac1{Nm_r^2}}\sum_{n=1}^N\|\hat X_n-\bar\Psi(\hat Y_n)\|_F^2.
		\end{equation}
		Let $e_n=\hat X_n-\Psi^*(\hat Y_n)$ and $d_n=\Psi^*(\hat Y_n)-\bar\Psi(\hat Y_n)$, and write $E=\EM(\Psi^*)={\frac1{Nm_r^2}}\sum_{n=1}^N\|e_n\|_F^2$ and $D={\frac1{Nm_r^2}}\sum_{n=1}^N\|d_n\|_F^2$. Expanding the squared residual and applying the Cauchy--Schwarz inequality yields:
		\begin{equation}\label{eq:risk-expansion}
\begin{split}
			&\EM(\Psi^*_{EQ})\leq{\frac1{Nm_r^2}}\sum_{n=1}^N\|e_n+d_n\|_F^2
			=E+D+{\frac2{Nm_r^2}}\sum_{n=1}^N\langle e_n,d_n\rangle_F\\
			&\leq E + D +\! {\frac2{Nm_r^2}}\sum_{n=1}^N \|e_n\|_F\cdot\|d_n\|_F \leq E+D+{\frac2{Nm_r^2}}\sqrt{\sum_{n=1}^N \|e_n\|_F^2}\cdot\sqrt{ \sum_{n=1}^N\|d_n\|_F^2  }  \\
            &\leq E+D+2\sqrt{ED}.
\end{split}
		\end{equation}

Since each assigned group is finite of order $T$ with linear discrete actions, we have $\bar{\Psi}(\hat{Y}_n)=\frac1T\sum_{t=1}^T\tilde\pi_{A_n^t}[\Psi^*[\tilde\pi_{A_n^t}^{-1}]](\hat{Y}_n)$,  where $A_n^t$ denotes the $t$-th element of $\mathcal{H}_G(\hat Y_n)$. 
Using Jensen's inequality and applying Theorem~\ref{Error-dis} for each fixed $t$ (justified below), we have:
		\begin{equation}\label{eq:risk-jensen}
			\begin{aligned}
				D &={\frac1{Nm_r^2}}\sum_{n=1}^N\left\|\Psi^*(\hat Y_n)-\bar\Psi(\hat Y_n)\right\|_F^2 ={\frac1{Nm_r^2}}\sum_{n=1}^N\left\|\frac1T\sum_{t=1}^T\left[\Psi^*(\hat Y_n)-\tilde\pi_{A_n^t}\left(\Psi^*(\tilde\pi_{A_n^t}^{-1}(\hat Y_n))\right)\right]\right\|_F^2\\
				&\leq{\frac1{Nm_r^2}}\sum_{n=1}^N\frac1T\sum_{t=1}^T\left\|\Psi^*(\hat Y_n)-\tilde\pi_{A_n^t}\left(\Psi^*(\tilde\pi_{A_n^t}^{-1}(\hat Y_n))\right)\right\|_F^2\\
				&=\frac1T\sum_{t=1}^T{\frac1{Nm_r^2}}\sum_{n=1}^N\left\|\Psi^*(\hat Y_n)-\tilde\pi_{A_n^t}\left(\Psi^*(\tilde\pi_{A_n^t}^{-1}(\hat Y_n))\right)\right\|_F^2
				\leq Ch^2+\hat C\varepsilon.
			\end{aligned}
		\end{equation}
Here, the first inequality is Jensen's inequality for the convex squared norm. Letting $\tilde Y_n=\tilde\pi_{A_n^t}^{-1}(\hat Y_n)$, so that $\hat Y_n=\tilde\pi_{A_n^t}(\tilde Y_n)$, and using $\Psi^*=\tilde\Phi$ (Assumption~\ref{assum}, $\eta=0$), the transformed-dataset hypothesis of Theorem~\ref{ER} permits applying Theorem~\ref{Error-dis} to $\{\tilde Y_n\}_{n=1}^N$ with the same constants, bounding the inner dataset average by $Ch^2+\hat C\varepsilon$ for every such $t$.
Substituting this into \eqref{eq:risk-expansion} and using $\sqrt{x+y}\leq\sqrt x+\sqrt y$ for $x,y\geq0$, we obtain:
		\[
		\begin{aligned}
			\EM(\Psi^*_{EQ})
			&\leq E+Ch^2+\hat C\varepsilon+2\sqrt E\sqrt{Ch^2+\hat C\varepsilon}\leq E+Ch^2+\hat C\varepsilon+2\sqrt{C\,E}\,h+2\sqrt{\hat C\,E}\,\sqrt\varepsilon\\
			&=E+\bigl(Ch+2\sqrt{C\,E}\bigr)h+\bigl(\hat C\sqrt\varepsilon+2\sqrt{\hat C\,E}\bigr)\sqrt\varepsilon=E+C_h h+C_\varepsilon\sqrt\varepsilon.
		\end{aligned}
		\]
		The proof is completed. 
\end{proof}

Theorem~\ref{ER} bounds the empirical risk of an equivariant network in terms of the discretization mesh size $h$ and the data symmetry error $\varepsilon$. With uniformly bounded coefficients, the empirical penalty $C_h h + C_\varepsilon\sqrt\varepsilon$ diminishes as $h, \varepsilon \to 0$\footnote{\textcolor{black}{It is worth noting that while the coefficients $C_h$ and $C_\varepsilon$ inherently depend on $h$ and $\varepsilon$ (as defined in Theorem~\ref{ER}), they remain strictly bounded as $h, \varepsilon \to 0$. Specifically, their limits are governed by the baseline empirical risk $\EM(\Psi^*)$, which is a finite and well-optimized constant representing the inherent fitting error of the adequately trained standard network. Consequently, this internal dependence does not dominate the penalty term, ensuring that the overall empirical penalty stably diminishes in our asymptotic analysis.}}. This implies two key insights: First, for high-resolution images (small $h$) with strong inherent symmetry (small $\varepsilon$), the equivariant constraint does not severely compromise the network's fitting capacity. Notably, when $G_n$ is uniform across all samples, this provides a rigorous theoretical guarantee for the fitting capability of most existing equivariant networks. Second, a flexible, sample-adaptive group formulation can actively reduce the dataset symmetry error $\varepsilon$, thereby securing a tighter empirical risk bound. This insight directly motivates our sample-adaptive network design to further enhance the model's fitting capacity.

\subsection{Generalization Theory of Equivariant Networks}
While empirical risk measures the fitting error on training data, the ultimate metric of a model's efficacy is its expected risk on unseen data. The generalization gap quantifies this discrepancy. 
Formally, for a deep network $\Psi$ trained on the dataset $D_N$ drawn from a joint distribution $\mathcal{P}$, the generalization error is typically defined as \cite{jakubovitz2019generalization}:
\begin{equation}
    GE(\Psi,D_N)=\mathcal{L}_{exp}(\Psi)-\EM(\Psi),
\end{equation}
where $\mathcal{L}_{exp}(\Psi)=\mathbb{E}_{(\hat{X}, \hat{Y})\sim \mathcal{P}}[\|\hat{X} - \Psi(\hat{Y})\|_F^2/m_r^2]$ represents the expected risk under the true data distribution $\mathcal{P}$, and $\EM$ is the empirical risk defined in \eqref{emp_risk}.

\begingroup
\color{black}
For equivariant networks, embedding symmetry constraints restricts the hypothesis space, which can lead to tighter generalization bounds.
Extensive prior studies have rigorously established generalization bounds for networks equivariant to fixed finite or compact groups \cite{sannai2021improved, petrache2023approximation}.
Based on these foundations, rather than focusing on a rigorous end-to-end derivation of generalization bounds, we provide a high-level analysis under idealized assumptions to combine with our preceding findings and further clarify the properties of equivariant networks. 

Specifically, our analysis is based on a typical recent class of adaptive equivariant architectures \cite{tan2026image}, which makes it possible to adaptively control the transformation group via a sample-dependent parameter $w_n$. We suppose the \textbf{ideal hypernetwork assumption} holds, i.e., these parameters are provided by an ideal, pre-trained hypernetwork $\mathcal{H}$. This hypernetwork remains fixed before sampling and consistently outputs the correct $w_n$ for each input image (other necessary technical conditions are detailed in Appendix~\ref{app:supp:gen-assumptions}). Under this ideal assumption, for the complete composite network combining the hypernetwork and the equivariant backbone, $\Psi_{EQ} \in \Lambda_{EQ}$, we readily obtain the following trivial conclusion:

\begin{lemma}\label{eq-ge}
    Under the ideal hypernetwork assumption and conditions (A1)--(A5) detailed in Appendix~\ref{app:supp:gen-assumptions}, for sample size $N\geq1$ and $0<\delta<1$, with probability at least $1-\delta$, 
    \begin{equation}\label{eq:GE_standard}
        |GE(\Psi_{EQ},D_N)|\leq U_N/\sqrt T+V_{N,\delta},
        \qquad \forall\Psi_{EQ}\in\Lambda_{EQ}.
    \end{equation}
%    \begin{equation}\label{eq:GE_CNN}
%        |GE(\Psi,D_N)|\leq U_N+V_{N,\delta},
%        \qquad \forall\Psi\in\Lambda_{\rm CNN}^{\rho_{\rm par}}.
%    \end{equation}
    {For the pixel-averaged loss, $U_N$ and $V_{N,\delta}$ are $O(\sqrt{\log(N+1)/N})$ and their definitions are given in the appendices. The statistical result is stated for each fixed resolution, with the same constants when the architecture and parameter bounds are uniform.} 
\end{lemma}

This conclusion intuitively implies that the generalization capability of the model improves (i.e., the error bound decreases) as the group size $T$ increases, owing to the stronger parameter sharing mechanism enforced across transformation directions. 

Particularly, for non-sample-adaptive scenarios, which represent the vast majority of existing equivariant network designs, the ideal hypernetwork assumption is no longer required. In such cases, the aforementioned conclusion naturally holds under standard general conditions where $\Lambda_{EQ} \subset \Lambda$. In the extreme case where $T=1$, the equivariant network $\Psi_{EQ}$ simply degenerates into an ordinary CNN, yielding the following standard bound:
\begin{equation}\label{eq:GE_CNN}
    |GE(\Psi,D_N)|\leq U_N+V_{N,\delta},
    \qquad \forall\Psi\in\Lambda.
\end{equation}

\textbf{Expected Risk Bound and Symmetry Alignment.}
The total expected performance of a model is bounded by the sum of its empirical risk and generalization error. 
%{\hypersetup{allcolors=red}\textcolor{black}{When applying both bounds to the prescribed network classes, we require the Reynolds candidate to belong to the same equivariant class, as detailed in Appendix~\ref{app:supp:gen-risk-connection}.}} 
By combining the empirical risk bound (Theorem \ref{ER}) with the generalization bound (Lemma \ref{eq-ge}), we can derive the expected risk upper bound for the optimal equivariant network $\Psi_{EQ}^*$:
\begin{equation}\label{eq:Expected_EQ}
    \mathcal L_{\mathrm{exp}}(\Psi^*_{EQ})\leq
    \EM(\Psi^*) + C_h h + C_\varepsilon \sqrt{\varepsilon} + U_N/\sqrt T + V_{N,\delta}.
\end{equation}
Similarly, by incorporating the standard generalization bound for ordinary CNNs in \eqref{eq:GE_CNN}, we obtain the corresponding expected risk upper bound for the standard network $\Psi^*$:
\begin{equation}\label{eq:Expected_Std}
    \mathcal L_{\mathrm{exp}}(\Psi^*)\leq\EM(\Psi^*) + U_N + V_{N,\delta}.
\end{equation}

It is crucial to note that directly comparing theoretical upper bounds does not yield a strict, computable criterion for the actual expected risks in practice. However, such a comparison serves as a valuable qualitative tool to gain high-level intuition into when an equivariant network might hold a theoretical advantage. Specifically, the expected risk upper bound for $\Psi^*_{EQ}$ becomes strictly tighter than that of the standard network when:
\begin{equation}\label{eq:trade_off_condition}
    C_h h + C_\varepsilon \sqrt{\varepsilon} < U_N(1-T^{-1/2}).
\end{equation}

Importantly, inequality \eqref{eq:trade_off_condition} provides heuristic guidance rather than a directly computable performance threshold. Its precise evaluation would require the uniform regularity, quadratic-growth, and network-parameter constants appearing in the analysis, together with the rigorous assumptions needed to combine the two bounds. However, the profound significance of this inequality is not numerical verification, but rather acting as a conceptual vehicle to clarify the balance between the fitting cost of equivariance and its parameter-sharing benefit. For example, across datasets with comparable resolution and sample size, stronger inherent symmetry (a smaller $\varepsilon$) suggests a more favorable trade-off. 

Ultimately, this qualitative analysis implies that keeping the symmetry mismatch as small as possible is essential to fully harnessing the generalization advantages of group equivariance. This theoretical intuition fundamentally motivates the jointly trained sample-adaptive hypernetwork in Section~\ref{sec:network-design}, providing a principled design direction without strictly guaranteeing the ordering of actual test errors.

\textbf{Remark 1 (Generality to Traditional Equivariant Networks).} Our framework naturally accommodates conventional equivariant networks by assuming a fixed group ($G_n = G$ for all $n$). In this scenario, \eqref{eq:Expected_EQ} provides a rigorous foundation for these methods: as long as the dataset symmetry under $G$ is sufficiently strong to satisfy \eqref{eq:trade_off_condition}, the generalization benefit from coefficient sharing dominates the empirical penalty of symmetry mismatch. This strictly yields a smaller expected risk bound than standard CNNs, theoretically justifying the empirical success of existing fixed-group equivariant architectures.

\textbf{Remark 2 (Practical Interpretation of Theoretical Assumptions).} {Our mathematical assumptions naturally align with the physical and computational realities of image restoration. (i) First, the $\mu$-quadratic growth in Assumption~\ref{ass:qg} characterizes the local stability of the inverse problem's solution. While strict mathematical adherence might be challenging to guarantee in highly complex real-world scenarios, a properly designed regularizer $R$ typically shapes a well-conditioned energy landscape where this condition is approximately satisfied.  Even when strictly viewed  as an idealized approximation, the proposed theoretical framework still yields conclusions that closely align with practical behaviors, providing sufficient illuminating insights for the field. 
(ii) Second, the boundedness conditions in \eqref{bound_condition} are defined exclusively on the underlying continuous image and kernel functions before sampling, with fixed bounds on the functions listed in Theorem~\ref{Error-dis} for all compatible meshes and transformations.
Regarding the spatial discretization, our theory evaluates the risk bound concerning the mesh size $h$ under  fixed support radii $D_r,D_k,D_R$.}

\endgroup

\section{Network Design Based on Theoretical Guidance}\label{sec:network-design}
	The analysis in Section~\ref{sec:risk-analysis} shows that a smaller data symmetry error tightens the empirical risk upper bound.
%	However, as demonstrated in Fig.~\ref{symmetry_show}, using a single consistent transformation group for the entire dataset often leads to suboptimal dataset symmetry error, 
%	whereas a sample-adaptive transformation group family $\{G_n\}_{n=1}^N$ can significantly reduce this error. 
%	Therefore, a sample-adaptive transformation group family is expected to achieve better performance.
%	Inspired by this, we propose a novel transformation-group sample-adaptive equivariant network framework.
	Meanwhile, Fig.~\ref{symmetry_show} shows that sample-adaptive transformation groups achieve a smaller empirical symmetry metric $\varepsilon_{\mathcal G}$ than a common group, indicating more stable dataset-averaged feature responses under the sampled transformations. Together with the theoretical analysis, this observation motivates our sample-adaptive equivariant network design.
	
	\subsection{Adaptive Transformation Group Family and Filter Parametrization}
	We first introduce the following transformation-group-parameterized equivariant convolution framework \cite{tan2026image}, which is started by constructing the following learnable transformation group:
	\begin{equation}\label{Sw-group}
		G_w = \{D_w A D_w ^{-1}\},
	\end{equation}
	where $A =
	\begin{bmatrix}
		\cos \theta\! & \!\sin \theta \\
		-\sin \theta\! &  \!\cos \theta
	\end{bmatrix}$ is the rotation matrix, \textcolor{black}{$\theta=2\pi k/T$, $k=0,1,\ldots,T-1$, and $T$ is the finite group order, consistently with \eqref{metric} and Lemma~\ref{eq-ge}.} $D_w$ is a learnable affine matrix defined as:
	\begin{equation}\label{D}
		\begin{split}
			D_w\!=\!\!\begin{bmatrix}
				\cos\! \alpha \!&\! \sin \alpha \\
				\!-\sin \alpha \!&\! \cos \alpha
			\end{bmatrix}\!\begin{bmatrix}
				s_{x} \!&\!\! 0 \\
				0 \!&\!\! s_y 
			\end{bmatrix}
			\!\!=\!\!\begin{bmatrix}
				s_x \cos \alpha \!&\! s_y \sin \alpha \\
				-s_x \sin \alpha \!&\! s_y \cos \alpha
			\end{bmatrix},
		\end{split}
	\end{equation}
	where ${w}=\left[{\alpha}, {s_x}, {s_y}\right]$ denotes adaptive parameters, with ${s_x}$ and ${s_y}$ representing adaptive scaling factors along two orthogonal axes, and ${\alpha}$ determines the principal orientation{\footnote{\textcolor{black}{In practice, the coefficients of the inverse coordinate transformation are learned directly and used to evaluate the adaptive filter bases, avoiding explicit matrix inversion. Implementation details are provided in Appendix \ref{app:supp:hypernet-init}.}}}.
	
	Then,  the transformation learnable filter $\varphi$ can be achieved by further adopting the following filter parametrization framework \cite{shen2020pdo, weiler2018learning, xie2022fourier}:
	\begin{equation}\label{TL-filter}
		\varphi(A^{-1}D_w^{-1}x)=\sum_k v_k\varphi_k(A^{-1}D_w^{-1}x),
	\end{equation}
	where $\varphi_k$ is the $k$-th filter basis function and $v_k$ is its learnable coefficient. %\textcolor{black}{This network construction uses reference pixel coordinates. Its fixed filter grids are distinct from the physical quadrature grids in \eqref{fixed-windows}; the convention is detailed in Appendix~\ref{app:supp:network-discrete-grid}.} 

	\subsection{Transformation Adaptive Equivariant Convolutions}
	Consequently, the transformation adaptive equivariant convolution framework can be constructed with filter in the framework \eqref{TL-filter}. 
	Specifically, taking construction of the learnable input layer of equivariant convolution as an example \cite{tan2026image, xie2022fourier},  the  convolution operation 
	$\Psi_w$, which  maps an input $r\in C^\infty(\R^2)$ to a feature map, can be formulated as:
	\begin{equation}\label{new_input_Conv}
		\Psi_w[r](x,A) =  \int_{\R^2}\varphi_{in}\left(A^{-1}\!D_w^{-1} \delta \right)r(x-\delta)d\sigma(\delta),
	\end{equation}
	where $\sigma$ is the measure on $\R^2$, $A$ denotes the rotation transformation, 
	which is also the index of group dimension, $D_w$ is the adaptive transformation defined in \eqref{D}, and $\varphi_{in}$ is the parameterized filter as defined in (\ref{TL-filter}). 
	Under this framework, we can conveniently achieve sample-level adaptive adjustment of the transformation group by adjusting 
	$w$ for each sample.
	
	This formulation can be readily generalized to the intermediate and output layers of the equivariant convolution framework, and their discretization can be easily performed following the strategy of previous work \cite{tan2026image, xie2022fourier}. 
	Due to space constraints, we do not elaborate on the details of the intermediate layers, 
	output layers, or the discretization in the main text. Interested readers are referred to the appendices, while a conceptual overview of the complete adaptive equivariant architecture is provided in Fig.~\ref{network}(c).
	
	%\subsection{Construction of the Equivariant Backbone}
	%
	%The proposed framework supports construction in a plug-and-play manner. Specifically, by seamlessly substituting all standard convolution operators in existing restoration networks with the TL-Conv defined above, we convert a standard architecture into its transformation-learnable equivariant counterpart without altering the moverall architecture of the backbone, as illustrated in Fig. \ref{network}(c).
	%
	%In a naive setting, the affine transformation parameters $w = [\alpha, s_x, s_y]^T$ of the TL-Conv layers could be set as global learnable parameters, optimized over the entire dataset via backpropagation. However, as demonstrated by our theoretical analysis (Theorem \ref{ER}), the optimal symmetry group is intrinsically sample-dependent. Forcing a single set of global parameters across a diverse dataset fails to minimize the symmetry error $\varepsilon$. This necessitates the introduction of a Hypernetwork to achieve instance-level adaptivity.
	
	\begin{figure}
		\centering
		\includegraphics[width=1\linewidth]{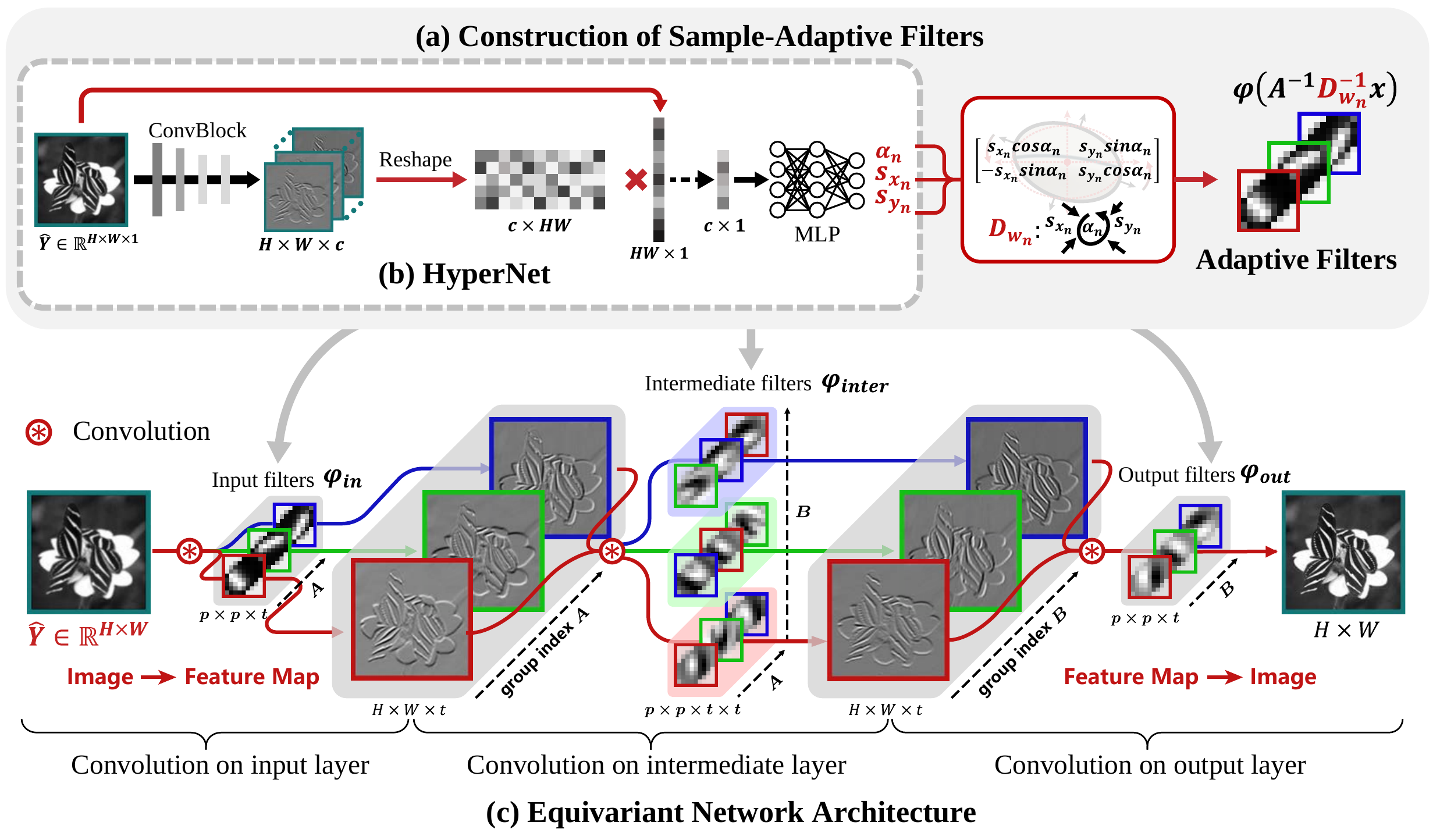}
		\vspace{-8mm}
		\caption{ Illustration of (a) adaptive transformation group family and filter parametrization;
			(b) the hypernetwork framework; 
			(c) transformation adaptive equivariant convolutions.
		}
		\label{network}
		\vspace{-9mm}
	\end{figure}
	
	\subsection{Sample-Adaptive Symmetry via Hypernetwork}

	To achieve the sample adaptive transformation, we introduce a hypernetwork for predicting $w_n$ in the transformation adaptive equivariant convolutions, as illustrated in Fig.~\ref{network}(b). The hypernetwork offers a unified framework for handling varying $w_n$, avoiding the need to learn each sample's transformation individually. This unified mechanism also enables the proposed method to generalize conveniently to new unseen images without the need for per-sample retraining.

	Specifically, we follow the local image  information extraction architecture proposed in  prior work \cite{fu2022kxnet}, for designing this hypernetwork.
	The input $\hat{Y}_n$ is first converted to grayscale, and a lightweight convolutional encoder is then adopted to extract local features, which is in the shape of $c\times H\times W$, with $c$ denoting the channel number. 
	After aligning the grayscale input with the feature map by average pooling, both are flattened along the spatial dimensions, which are in shape $c\times HW$ and $HW\times 1$, respectively. Perform matrix multiplication on the above expanded contents and divide by $HW$ to achieve weighted local feature extraction \cite{fu2022kxnet}. Then the output $c$-dimensional vector is finally mapped into the 3-dimensional transformation parameter $w_n=[\alpha_n, s_{x_n}, s_{y_n}]$,  through an MLP.
	
	To ensure numerical stability and optimal convergence, we adopt a two-stage training strategy. First, the equivariant restoration backbone is pre-trained with $w_n$ as a shared learnable parameter across all samples, initialized to $[0,1,1]$. Then, it is integrated with a hypernetwork to adaptively adjust $w_n$ for each sample during end-to-end training.

% BEGIN APPROVED-REVIEW 07-scope-paragraph
{\color{black}
\textbf{Scope of the theoretical guarantees.}
The backbone is equivariant for fixed transformation parameters. For the complete adaptive network, a sufficient additional condition is that the hypernetwork predict the same parameters along each selected transformation orbit (see Appendix~\ref{app:supp:gen-composite-equivariance}). Our jointly trained hypernetwork is not constrained to satisfy this condition exactly. Likewise, Lemma~\ref{eq-ge} assumes a hypernetwork fixed before sampling. These results motivate the sample-adaptive design, while the restoration experiments assess the practical model.
}
% END APPROVED-REVIEW 07-scope-paragraph

	\section{Experimental Results}\label{sec:experiment}
	This section validates our theoretical framework through a systematic series of experiments. We first quantitatively verify the proposed dataset-level symmetry, which forms the mathematical cornerstone of our modeling. Subsequently, we evaluate the proposed Sample-Adaptive Equivariant Network across diverse image restoration tasks, including super-resolution, denoising, and single-image deraining. 
	
	\subsection{Empirical Validation of Dataset Symmetry}\label{sec:exp-symmetry}
	In Section \ref{sec:symmetry}, we introduce the definition of symmetry on dataset level and the relaxed symmetry constraint that better aligns with real-world data scenarios. Before presenting the restoration results, we detail the experimental settings and quantitative results for the empirical validation of dataset-level symmetry.
	
	\textbf{Experimental Setup.}
	We conduct the analysis on the DIV2K dataset \cite{agustsson2017ntire}, evaluating data symmetry using the metric $\varepsilon_{\mathcal{G}}$ defined in \eqref{metric}, 
which measures the range of the dataset-averaged local feature responses over the sampled transformations.

	For this empirical comparison, we use representative local-feature functionals associated with the TV operator, second-order TV, Sobel, Laplacian, and Prewitt filters as $R$. \textcolor{black}{Specifically, we average the absolute responses of these filters across valid locations, which aligns with the theoretical regularization formulation in Lemma \ref{error-R}, where the non-linear response function is set as $\rho(z)=\vert{}z\vert{}$.} The discrete form $R[\pi_{A_n^t}[r_n]]$ can be efficiently calculated via the process illustrated in Fig.~\ref{operation}.

	% We set the adaptive transformation group family as $\mathcal{G} = \{ G_{w_n} \}_{n=1}^N$, where $G_w$ is defined in \eqref{Sw-group}. 
	% Besides, we sample the continuous rotation group at 32 uniformly discrete angles (i.e., $T=32$ with a step size of $\frac{\pi}{16}$), and quantitatively compare data symmetry under three distinct scenarios:
	In this symmetry measurement, we use $T=32$ equally spaced rotation angles, with an angular step of $\pi/16$. We quantitatively compare data symmetry under three scenarios:
	(i) Strict symmetry of single image: As the baseline case, we randomly select a single image and apply strict rotation transformation $A^t$ to it, i.e., $\pi_{A^t}[r]$.
	(ii) Strict symmetry of entire dataset: 
	We apply the identical strict rotations $A^t$ to every sample $r_n$ in dataset $\mathcal{D}$. This setting examines global statistical invariance under a consistent dataset-level group action.
	(iii)  Adaptive symmetry of entire dataset: 
	We apply a sample-adaptive transformation for each sample, i.e. $\pi_{A_n^t}[r_n]$, where $A_n^t\in G_{w_n}$ and the transformation group family is $\mathcal{G} = \{G_{w_n}\}_{n=1}^N$, with $G_w$ defined in \eqref{Sw-group}. 
	
	It is worth noting that scenario (iii) requires optimizing the affine parameters $w_n$ for each sample, which necessitates a well-defined loss function.
	Although it is natural to directly use $\varepsilon_{\mathcal G}$ as the optimization objective, 
the non-smooth $\max$ and $\min$ operators
 can complicate optimization. Therefore, we instead optimize $w_n$ via gradient descent by minimizing the variance of the feature responses across different rotation angles. This proxy encourages more stable feature responses across rotation angles.
		\begin{figure}
		\centering
		\includegraphics[width=0.8\linewidth]{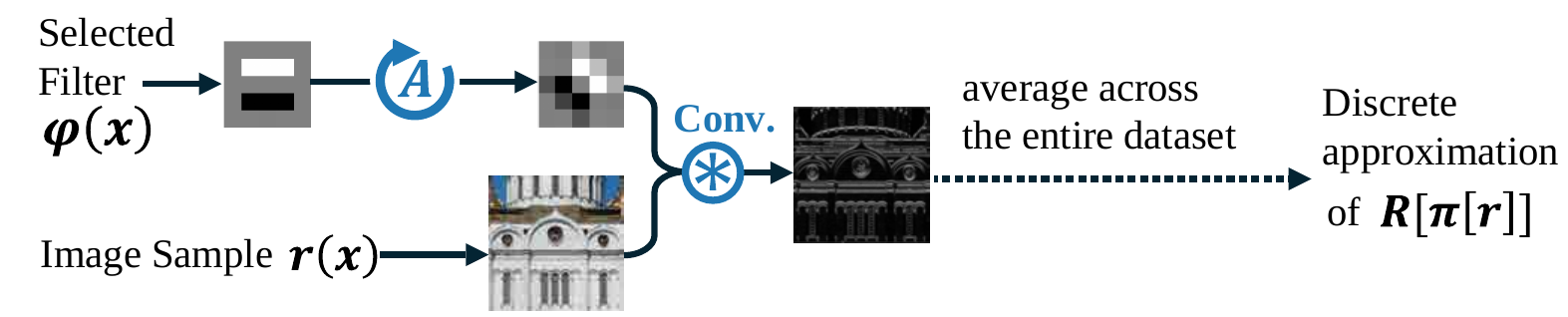}
		\vspace{-3mm}
		\caption{The process of feature extraction from images using classical differential kernels.}
		\label{operation}\vspace{-9mm}
	\end{figure}
	\begin{table}[t]
		\centering
		\small
		\caption{Quantitative evaluation of data symmetry under various regularizers. The reported values denote the empirical symmetry metric $\varepsilon_{\mathcal{G}}$ (defined in \eqref{metric}), where a lower value indicates stronger dataset-level symmetry under the sampled transformations.}
		\vspace{-2mm}
		\setlength{\tabcolsep}{12pt}
		\begin{tabular}{ccccccc}
			\toprule
			\multicolumn{2}{c}{Configuration} & \multirow{2}{*}{TV} & \multirow{2}{*}{Second-order TV} & \multirow{2}{*}{Sobel} & \multirow{2}{*}{Laplacian} & \multirow{2}{*}{Prewitt} \\
			\cmidrule(r){1-2} 
			Data & Trans. & & & & & \\
			\midrule
			sample & strict & \textcolor{black}{0.422} & \textcolor{black}{0.838} & \textcolor{black}{5.737} & \textcolor{black}{2.808} & \textcolor{black}{1.250} \\
			dataset & strict & \textcolor{black}{0.039} & \textcolor{black}{0.602} & \textcolor{black}{0.326} & \textcolor{black}{2.180} & \textcolor{black}{0.186} \\
			\rowcolor{gray!15}
			dataset & adaptive & \textcolor{black}{\textbf{0.007}} & \textcolor{black}{\textbf{0.182}} & \textcolor{black}{\textbf{0.280}} & \textcolor{black}{\textbf{0.478}} & \textcolor{black}{\textbf{0.044}} \\
			\bottomrule
		\end{tabular}
		\label{Tab:Symmetry_Validation}
		\vspace{-4mm}
	\end{table}
	
	As shown in Table \ref{Tab:Symmetry_Validation}, for various $R$, the dataset-level $\varepsilon_{\mathcal{G}}$ is consistently and significantly smaller than the single-sample one, and the adaptive-transformation $\varepsilon_{\mathcal{G}}$ is even lower than that under strict transformations. The consistent results across the evaluated regularizers support the applicability of the proposed dataset-level symmetry metric with sample-adaptive transformations to real image data.
	
	\subsection{Image Super-Resolution}
	Single image super-resolution (SR) aims to reconstruct high-resolution images from low-resolution observations, a fundamental task requiring the accurate recovery of high-frequency details.
	
	\begin{table}[t]
		\centering % \setlength{\tabcolsep}{40pt}
		\footnotesize
		\caption{The SR ($\times 2$) results of different competing methods on four benchmark datasets.}
		\centering \setlength{\tabcolsep}{2.7pt}
		\begin{tabular}{c c c c c c c c c c c}
			\toprule
			% \hline
			\multirow{2}{*}{Baseline} & \multirow{2}{*}{Method} & \multicolumn{2}{c}{Urban100 \cite{huang2015single}} & \multicolumn{2}{c}{B100 \cite{martin2001database}} & \multicolumn{2}{c}{Set14 \cite{zeyde2010single}} & \multicolumn{2}{c}{Set5 \cite{bevilacqua2012low}} & \multirow{2}{*}{\#Param.}\\
			\cmidrule(r){3-4}\cmidrule(r){5-6}\cmidrule(r){7-8}\cmidrule(r){9-10}
			& & PSNR & SSIM & PSNR & SSIM & PSNR & SSIM & PSNR & SSIM \\
			\midrule
			\multirow{9}{*}{EDSR\cite{lim2017enhanced}} 
			& CNN & 32.170 & 0.9297 & 32.201 & 0.9029 & 33.617 & 0.9198 & 38.063 & 0.9621 & 21.85M\\
			& G-CNN \cite{cohen2016group} & 32.269 & 0.9306 & 32.209 & 0.9027 & 33.606 & 0.9197 & 38.041 & 0.9621 & 7.09M\\
			& E2-CNN \cite{weiler2018learning} & 31.992 & 0.9283 & 32.148 & 0.9024 & 33.513 & 0.9191 & 37.912 & 0.9618  & 5.36M\\
			& Har-Net \cite{Worrall_2017_CVPR} & 31.986 & 0.9281 & 32.161 & 0.9025 & 33.534 & 0.9192 & 37.917 & 0.9619  & 6.30M\\
			& PDO-eConv \cite{shen2020pdo} & 30.912 & 0.9159 & 31.871 & 0.8985 & 33.169 & 0.9161 & 37.631 & 0.9606 & 4.73M\\
			& F-Conv \cite{xie2022fourier} & 32.357 & 0.9313 & 32.239 & 0.9034 & 33.698 & 0.9205 & 38.061 & 0.9622  & 8.92M\\
			& B-Conv \cite{xie2025rotation} & 32.489 & 0.9330 & 32.263 & 0.9038 & 33.732 & 0.9208 & 38.119 & 0.9625 & 8.92M\\
			& TL-Conv \cite{tan2026image} & 32.554 & 0.9337 & 32.275 & 0.9039 & 33.803 & 0.9217 & 38.128 & 0.9625 & 8.92M\\
			\rowcolor{gray!15}
			& SA-Conv & \textbf{32.673} & \textbf{0.9345} & \textbf{32.298} & \textbf{0.9041} & \textbf{33.833} & \textbf{0.9219} & \textbf{38.175} & \textbf{0.9626} & 8.96M\\
			\midrule
			\multirow{9}{*}{RDN\cite{zhang2018residual}} 
			& CNN & 32.230 & 0.9303 & 32.215 & 0.9031 & 33.630 & 0.9204 & 38.066 & 0.9623 & 22.12M\\
			& G-CNN \cite{cohen2016group} & 32.237 & 0.9304 & 32.217 & 0.9032 & 33.678 & 0.9204 & 38.050 & 0.9623 & 5.61M\\
			& E2-CNN \cite{weiler2018learning} & 32.095 & 0.9294 & 32.178 & 0.9028 & 33.565 &  0.9200 & 37.964 & 0.9620 & 4.31M\\
			& Har-Net\cite{Worrall_2017_CVPR}  & 32.492 & 0.9333 & 32.260 & 0.9038 & 33.763 & 0.9218 & 38.089 & 0.9624 & 4.76M\\
			& PDO-eConv \cite{shen2020pdo} & 29.550 & 0.8994 & 31.429 & 0.8933 & 32.511 & 0.9103 & 36.824 & 0.9566  & 2.86M\\
			& F-Conv \cite{xie2022fourier} & 32.514 & 0.9331 & 32.267 & 0.9038 & 33.758 & 0.9217 & 38.130 & 0.9624 & 7.71M\\
			& B-Conv \cite{xie2025rotation} & 32.486 & 0.9330 & 32.245 & 0.9035 & 33.731 & 0.9210 & 38.085 & 0.9624 & 7.71M\\
			& TL-Conv \cite{tan2026image} & 32.522 & 0.9331 & 32.253 & 0.9036 & 33.769 & 0.9211 & 38.111 & 0.9624 & 7.71M\\
			\rowcolor{gray!15}
			& SA-Conv & \textbf{32.611} & \textbf{0.9341} & \textbf{32.278} & \textbf{0.9039} & \textbf{33.845} & \textbf{0.9218} & \textbf{38.148} & \textbf{0.9625} & 7.75M\\
			\midrule
			\multirow{9}{*}{RCAN\cite{zhang2018image}} 
			& CNN & 32.359 & 0.9316 & 32.240 & 0.9034 & 33.703 & 0.9202 & 38.088 & 0.9624 & 12.37M\\
			& G-CNN \cite{cohen2016group} & 32.396 & 0.9324 & 32.243 & 0.9036 & 33.726 & 0.9210 & 38.053 & 0.9623 & 3.21M\\
			& E2-CNN \cite{weiler2018learning} & 31.550 & 0.9238 & 32.035 & 0.9009 & 33.361 & 0.9179 & 37.789 & 0.9613 & 2.02M\\
			& Har-Net \cite{Worrall_2017_CVPR} & 31.557 & 0.9237 & 32.060 & 0.9012 & 33.364 & 0.9177 & 37.786 & 0.9611  & 2.71M\\
			& PDO-eConv \cite{shen2020pdo} & 31.082 & 0.9181 & 31.916 & 0.8994 & 33.178 & 0.9163 & 37.680 & 0.9608 & 1.68M\\
			& F-Conv \cite{xie2022fourier} & 32.691 & 0.9341 & 32.293 & 0.9039 & 33.784 & 0.9217 & 38.160 & 0.9625 & 4.41M\\
			& B-Conv \cite{xie2025rotation} & 32.537 & 0.9328 & 32.274 & 0.9037 & 33.773 & 0.9217 & 38.083 & 0.9623 & 4.41M\\
			& TL-Conv \cite{tan2026image} & 32.700 & 0.9343 & 32.303 & 0.9042 & 33.771 & 0.9213 & 38.141 & 0.9625 & 4.41M\\
			\rowcolor{gray!15}
			& SA-Conv & \textbf{32.803} & \textbf{0.9353} & \textbf{32.319} & \textbf{0.9043} & \textbf{33.893} & \textbf{0.9223} & \textbf{38.162} & \textbf{0.9626} & 4.44M\\
			\bottomrule
			% \hline
		\end{tabular}
		\label{SR}
		\vspace{-5mm}
	\end{table}
	
	\textbf{Network Architecture Settings.}
	Following Xie et al. \cite{xie2022fourier}, we evaluate the $\times 2$ scale SR task using three existing baselines, including EDSR \cite{lim2017enhanced}, RDN \cite{zhang2018residual}, and RCAN \cite{zhang2018image}. To assess the impact of equivariant structures, we compare the proposed transformation sample-adaptive equivariant network against a standard CNN baseline and a comprehensive set of equivariant methods, including G-CNN \cite{cohen2016group}, E2-CNN \cite{weiler2019general}, Har-Net \cite{Worrall_2017_CVPR}, PDO-eConv \cite{shen2020pdo}, F-Conv \cite{xie2022fourier}, B-Conv \cite{xie2025rotation}, and TL-Conv \cite{tan2026image}.
	In each experiment, the standard convolution operators are replaced by their respective equivariant counterparts.
	
	\begin{figure}[t]
		\centering
		% \vspace{-3mm}
		\includegraphics[width=1.0\linewidth]{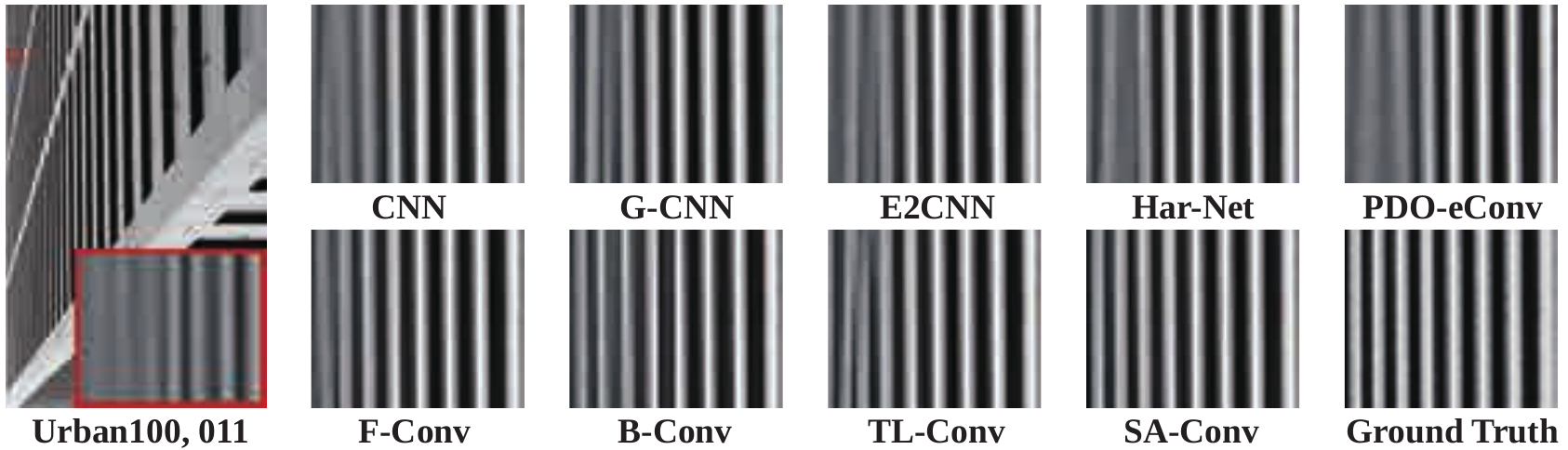}
		\vspace{-5mm}
		\caption{Visual comparison of $\times 2$ SR results from various methods on img011 of Urban100.} 
		\label{Fig:SR}
		\vspace{-8mm}
	\end{figure}
	\begin{table}[t]
		\centering 
		\small
		\vspace{-2mm}
		\caption{The SR ($\times 2$) results of different competing methods on RSSCN7 \cite{zou2015deep} dataset.}
		\vspace{-2mm}
		\centering \setlength{\tabcolsep}{3pt}
		\begin{tabular}{c c c c c c c c c c}
			\toprule
			\multirow{2}{*}{Method} & 
			\multicolumn{3}{c}{EDSR\cite{lim2017enhanced}} & \multicolumn{3}{c}{RDN\cite{zhang2018residual}} & \multicolumn{3}{c}{RCAN\cite{zhang2018image}} \\
			\cmidrule(r){2-4}\cmidrule(r){5-7}\cmidrule(r){8-10}
			& PSNR & SSIM & \#Param. & PSNR & SSIM & \#Param. & PSNR & SSIM & \#Param.\\
			\midrule
			CNN & 34.078 & 0.9128 & 21.85M & 34.021 & 0.9115 & 22.12M & 34.288 & 0.9169 & 12.37M\\
			G-CNN \cite{cohen2016group} &  34.186 & 0.9150 & 7.09M & 33.866 & 0.9083 & 5.61M & 34.350 & 0.9182 & 3.21M\\
			E2-CNN \cite{weiler2018learning} & 33.802 & 0.9070 & 5.36M & 33.766 & 0.9061 & 4.31M & 33.555 & 0.9020 & 2.02M\\
			Har-Net \cite{Worrall_2017_CVPR} & 33.892 & 0.9089 & 6.30M & 33.780 & 0.9066 & 4.76M & 33.770 & 0.9064 & 2.71M\\
			PDO-eConv \cite{shen2020pdo} & 33.535 & 0.9009 & 4.73M & 33.240 & 0.8947 & 2.86M & 33.486 & 0.9001 & 1.68M\\
			F-Conv \cite{xie2022fourier} & 34.079 & 0.9126 & 8.92M & 33.931 & 0.9095 & 7.71M & 34.357 & 0.9182 & 4.41M\\
			B-Conv \cite{xie2025rotation} & 34.180 & 0.9150 & 8.92M & 34.043 & 0.9119 & 7.71M & 34.573 & 0.9222 & 4.41M\\
			TL-Conv \cite{tan2026image} & 34.191 & 0.9151 & 8.92M & 34.140 & 0.9139 & 7.71M & 34.602 & 0.9229 & 4.41M\\
			\rowcolor{gray!15}
			SA-Conv & \textbf{34.206} & \textbf{0.9155} & 8.96M & \textbf{34.216} & \textbf{0.9154} & 7.75M & \textbf{34.610} & \textbf{0.9231} & 4.44M\\
			\bottomrule
			% \hline
		\end{tabular}
		\label{RS-SR}
		\vspace{-1mm}
	\end{table}
	\begin{figure}[ht!]
		\centering
		\vspace{-2mm}
		\includegraphics[width=1.0\linewidth]{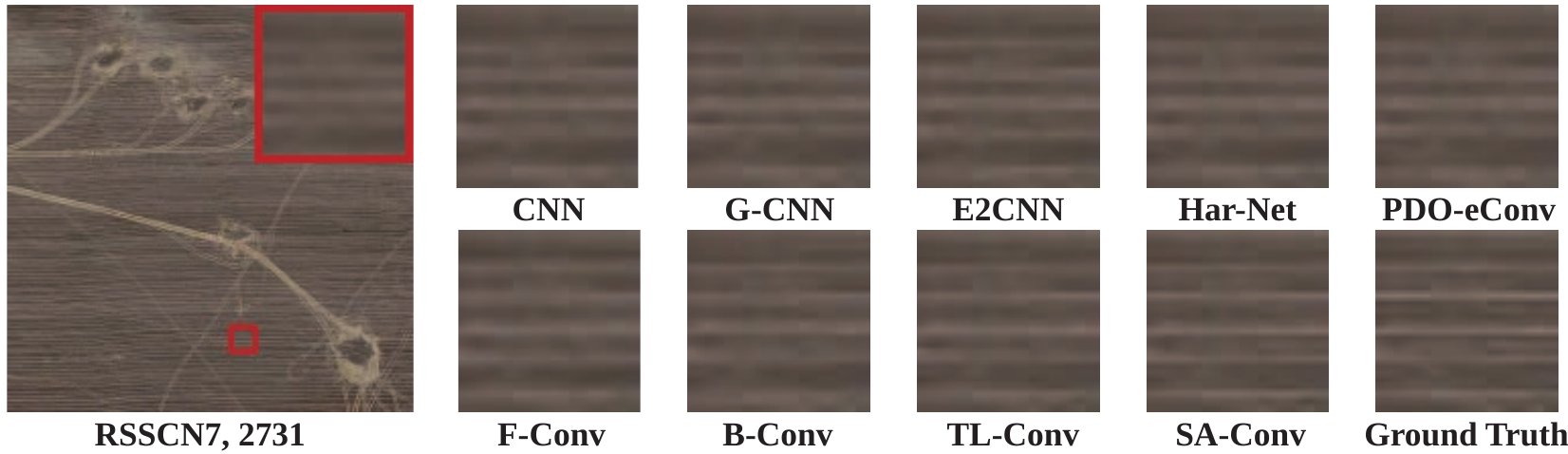}
		\vspace{-6mm}
		\caption{Visual comparison of $\times 2$ SR results from various methods on img2731 of RSSCN7.} 
		\label{Fig:RS-SR}
		\vspace{-9mm}
	\end{figure}
	
	For the G-CNN-based network, which is restricted to the $p4$ group ($\nicefrac{k\pi}{2}$ rotations, $k=1,2,3,4$), we reduce the number of channels in each residual block to $\nicefrac{1}{4}$ of the original network to maintain comparable memory costs. For the other equivariant methods, including SA-Conv, we employ the $p8$ group to capture richer symmetries and set the channel number as $\nicefrac{1}{8}$ to the original networks in each layer to keep memory consumption similar.
	
	\textbf{Datasets and Training Settings.}
	To validate the generalization and robustness of the proposed method across diverse visual domains, we conduct experiments on two distinct types of imagery.
	For natural images, following previous works \cite{lim2017enhanced, zhang2018residual,zhang2018image, xie2022fourier}, models are trained on 900 images from the DIV2K dataset \cite{timofte2017ntire} and tested on four benchmarks, including Urban100 \cite{huang2015single}, B100 \cite{martin2001database}, Set14 \cite{zeyde2010single}, and Set5 \cite{bevilacqua2012low}. 
	For remote sensing images, we use the RSSCN7 dataset \cite{zou2015deep} to assess the method’s effectiveness on overhead images exhibiting arbitrary orientations. It contains 2,800 images distributed across seven categories, each with four different scale images of $400 \times 400$ pixels collected by various sensors under different weather conditions. We randomly split the data into 2,600 training and 200 testing images.
	
	\begin{table}
		\centering
		\small
		\caption{The denoising results of different competing methods on synthesized testing datasets.}
		\vspace{-1mm}
		\setlength{\tabcolsep}{3.5pt}
		\begin{tabular}{c c c c c c c c c c}
			\toprule
			\multirow{2}{*}{Method}  & \multicolumn{2}{c}{Urban100 \cite{huang2015single}} & \multicolumn{2}{c}{BSD100 \cite{martin2001database}} & \multicolumn{2}{c}{Set14 \cite{zeyde2010single}} & \multicolumn{2}{c}{Set5 \cite{bevilacqua2012low}}  & \multirow{2}{*}{\#Param.}\\
			\cmidrule(r){2-3}\cmidrule(r){4-5} \cmidrule(r){6-7}\cmidrule(r){8-9}
			& PSNR & SSIM & PSNR & SSIM & PSNR & SSIM & PSNR & SSIM  \\
			\midrule
			CNN & 30.565 & 0.8877 & 29.562 & 0.8171 & 29.864 & 0.8112 & 31.838 & 0.8807 & 19.49M\\
			G-CNN \cite{cohen2016group}  & 30.749 & 0.8914 & 29.602 & 0.8199 & 29.888 & 0.8111 & 31.908 & 0.8825 & 4.73M\\
			E2-CNN \cite{weiler2018learning}  & 30.661 & 0.8891 & 29.575 & 0.8177 & 29.792 & 0.8086 & 31.902 & 0.8817 & 3.00M\\
			Har-Net \cite{Worrall_2017_CVPR} & 30.787 & 0.8914 & 29.618 & 0.8199 & 29.902 & 0.8118 & 31.934 & 0.8822 & 3.94M\\
			PDO-eConv \cite{shen2020pdo} & 29.549 & 0.8628 & 29.189 & 0.7998 & 29.502 & 0.8030 & 31.466 & 0.8703 & 2.37M\\
			F-Conv \cite{xie2022fourier}  & 31.049 & 0.8962 & 29.658 & 0.8211 & 30.021 & 0.8134 & 31.975 & 0.8836 & 6.56M\\
			B-Conv \cite{xie2025rotation}  & 30.974 & 0.8924 & 29.612 & 0.8167 & 29.986 & 0.8113 & 31.941 & 0.8806 & 6.56M\\
			TL-Conv \cite{tan2026image}  & 31.059 & 0.8956 & 29.671 & 0.8210 & 30.051 & 0.8145 & 32.020 & 0.8838 & 6.56M\\
			\rowcolor{gray!15}
			SA-Conv  & \textbf{31.155} &\textbf{0.8979} & \textbf{29.690}  & \textbf{0.8212} & \textbf{30.111} & \textbf{0.8153} & \textbf{32.036} & \textbf{0.8842} & 6.60M\\
			\bottomrule
		\end{tabular}
		\vspace{-2mm}
		\label{DeNoise}
	\end{table}
	\begin{figure}
		\centering
		% \vspace{1mm}
		\includegraphics[width=1.0\linewidth]{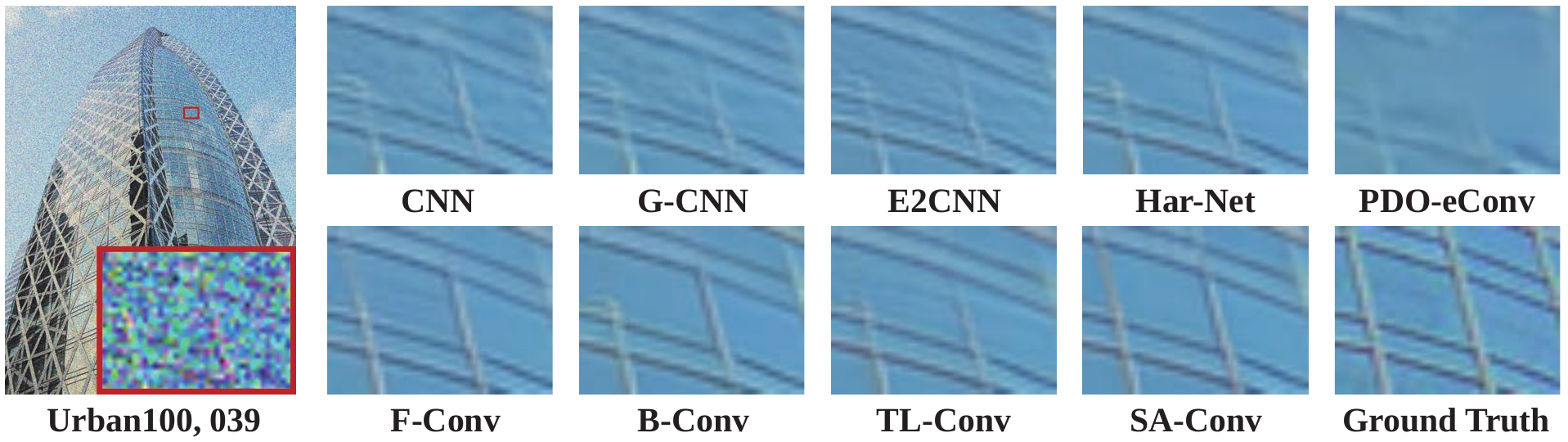}
		\vspace{-7mm}
		\caption{Visual results of Denoising from various methods on img039 of Urban100.} 
		\label{Fig:DeNoise}
		\vspace{-7mm}
	\end{figure}
	
	All models are trained using the Adam optimizer on an NVIDIA RTX 4090 GPU. Following the original settings, we set the batch size to 16 for all networks. The training patch sizes are set to $64 \times 64$ for RDN based methods and $96 \times 96$ for EDSR and RCAN variants. Performance is quantified using PSNR and SSIM.
	
	\textbf{Quantitative and Qualitative Comparison.} 
	Tables \ref{SR} and  \ref{RS-SR}  present the $\times 2$ SR results on natural images and remote sensing images, respectively.  It can be observed that the proposed SA-Conv consistently achieves superior results compared to the competing methods in these two experiments. 
	Notably, on the detail-rich natural image dataset Urban100, SA-Conv (EDSR) outperforms the baseline CNN by 0.503 dB in PSNR, and surpasses the  transformation learnable equivariant convolution TL-Conv (not sample-adaptive) by 0.119 dB. SA-Conv (RCAN) achieves the best PSNR among all methods, reaching 32.803 dB. Similarly, on remote sensing images, SA-Conv (RDN) exceeds TL-Conv by 0.076 dB, while SA-Conv (RCAN) significantly outperforms the baseline CNN with a PSNR gain of 0.322 dB.
	
	Visual comparisons are provided in Fig.\ref{Fig:SR} and Fig.\ref{Fig:RS-SR}  (take RCAN-based methods as an example). Standard baselines and strict equivariant methods tend to produce aliasing artifacts or over-smoothed textures in complex repetitive patterns. In contrast, SA-Conv reconstructs sharper edges and more faithful geometric details, effectively mitigating structural distortions.

		\begin{figure}
		\centering
		\vspace{-2mm}
		\includegraphics[width=1.0\linewidth]{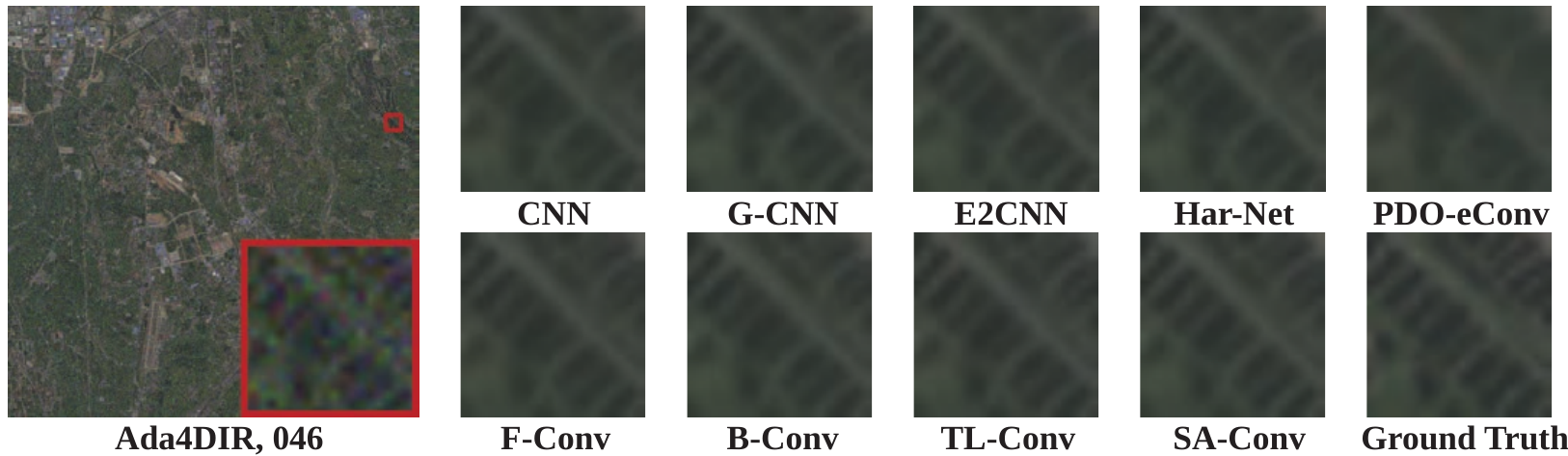}
		\vspace{-6mm}
		\caption{Visual results of Denoising from various methods on img046 of Ada4DIR.} 
		\label{Fig:RS-DeNoise}
		\vspace{-10mm}
	\end{figure}
	
	\subsection{Image Denoising}
	Image denoising is a classic restoration task that aims to remove noise from corrupted observations while preserving high-frequency structures.
	
	\textbf{Network Architecture Settings.}
	Adhering to prior works \cite{fu2024rotation}, we adopt a unified ResNet architecture as the common backbone for all competing methods. Specifically, the backbone is configured with 16 residual blocks, each containing 256 channels.
	We evaluate the proposed SA-Conv against the standard CNN and multiple equivariant methods, including G-CNN \cite{cohen2016group}, E2-CNN \cite{weiler2019general}, Har-Net \cite{Worrall_2017_CVPR}, PDO-eConv \cite{shen2020pdo}, F-Conv \cite{xie2022fourier}, B-Conv \cite{xie2025rotation}, and TL-Conv \cite{tan2026image}.
	Consistent with SR configuration, we apply identical channel reduction strategies to ensure strictly comparable memory overhead across all models.

	\begin{wraptable}{r}{0.5\textwidth} 
		\centering
		\small
		\vspace{-7mm}
		\caption{The denoising results of different competing methods on Ada4DIR dataset \cite{lihe2025ada4dir}.}
		\setlength{\tabcolsep}{1pt} 
		\begin{tabular}{cccc}
			\toprule
			\multirow{2}{*}{Method} & \multicolumn{2}{c}{Ada4DIR-Denoising \cite{lihe2025ada4dir}} & \multirow{2}{*}{\#Param.}\\
			\cmidrule(r){2-3}
			& PSNR & SSIM \\
			\midrule
			CNN & 37.021 & 0.8999 & 19.49M\\
			G-CNN\cite{cohen2016group} & 37.042 & 0.8999 & 4.73M\\
			E2CNN\cite{weiler2018learning} & 37.016 & 0.8993 & 3.00M\\
			Har-Net\cite{Worrall_2017_CVPR} & 36.998 & 0.8989 & 3.94M\\
			PDO-eConv\cite{shen2020pdo} & 36.687 & 0.8935 & 2.37M\\
			F-Conv\cite{xie2022fourier} & 37.046 & 0.8998 & 6.56M\\
			B-Conv\cite{xie2025rotation} & 37.064 & 0.9009 & 6.56M\\
			TL-Conv\cite{tan2026image} & 37.073 & 0.9013 & 6.56M\\
			\rowcolor{gray!15}
			SA-Conv & \textbf{37.104} & \textbf{0.9016} & 6.60M\\
			\bottomrule
		\end{tabular}
		\label{RS-DeNoise}
		% \vspace{-3mm}
	\end{wraptable}
	
	\textbf{Datasets and Training Settings.}
	For natural image denoising, models are trained on 900 DIV2K samples \cite{timofte2017ntire} with Gaussian noise ($\sigma = 50$) and evaluated on Urban100 \cite{huang2015single}, B100 \cite{martin2001database}, Set14 \cite{zeyde2010single}, and Set5 \cite{bevilacqua2012low} benchmarks. For remote sensing, we adopt the Ada4DIR denoising
	subset \cite{lihe2025ada4dir} (5,130 training and 270 testing images), where the noise standard deviations are randomly sampled from $[5, 35]$. Training hyperparameters remain consistent with the SR task, utilizing $48 \times 48$ cropped patches.
	
	\textbf{Quantitative and Qualitative Comparison.}
	Table \ref{DeNoise} shows the denoising performance on natural image benchmarks. Consistent with SR, most equivariant methods demonstrate superior noise suppression compared to the standard CNN baseline. Overall, SA-Conv achieves the best performance among all architectures, reaching 31.155 dB. These results once 
	again demonstrate that equivariant networks can achieve performance gains over standard CNNs. When sample-adaptive transformations are further introduced to better align network equivariance with data symmetry, the performance advantage becomes more pronounced. 
	The visualization results in Fig.~\ref{Fig:DeNoise} show that the proposed method is able to preserve richer detail features.
	Similarly, for the remote sensing denoising task (Table \ref{RS-DeNoise}), where noise levels vary substantially, SA-Conv still achieves the best performance (37.104 dB). As illustrated in Fig.~\ref{Fig:RS-DeNoise}, while competing methods often struggle to balance noise removal and texture preservation, thereby blurring fine details, SA-Conv produces noticeably cleaner images.

	\begin{wraptable}{r}{0.5\textwidth}
		\centering
		\small
		\vspace{-8mm}
		\caption{The rain removal results of different competing methods on Rain100L \cite{yang2019joint}.}
		\vspace{-2mm}
		\setlength{\tabcolsep}{9pt}
		\begin{tabular}{ccc}
			\toprule
			\multirow{2}{*}{Method} & \multicolumn{2}{c}{Rain100L \cite{yang2019joint}} \\
			\cmidrule(r){2-3}
			& PSNR & SSIM \\
			\midrule
			Input & 26.90 & 0.8384 \\
			DSC\cite{luo2015removing} & 27.34 & 0.8494 \\
			% GMM\cite{li2016rain} & 29.05 & 0.8717 \\
			JCAS\cite{gu2017joint}  & 28.54 & 0.8524 \\
			Clear\cite{fu2017clearing} &30.24 & 0.9344 \\
			DDN\cite{fu2017removing} & 32.38 & 0.9258 \\
			RESCAN\cite{li2018recurrent} & 38.52& 0.9812 \\
			PReNet\cite{ren2019progressive} & 37.45 & 0.9790 \\
			SPANet\cite{wang2019spatial} & 35.33 & 0.9694 \\
			JORDER\_E\cite{yang2019joint} & 38.59 & 0.9834 \\
			SIRR\cite{wei2019semi} & 32.37 & 0.9258 \\
			IDT\cite{xiao2022image} & 35.42 & 0.9674 \\
			\midrule
			RCDNet\cite{wang2020model} & 40.00 & 0.9860 \\
			FConv-RCDNet \cite{xie2022fourier} & 40.28 & 0.9867 \\
			BConv-RCDNet \cite{xie2025rotation} & 40.25 & 0.9863 \\
			TLEQ-RCDNet \cite{tan2026image} & 40.36 & 0.9869 \\
			\rowcolor{gray!15}
			SA-RCDNet & \textbf{40.54} & \textbf{0.9873} \\
			\bottomrule
		\end{tabular}
		% \vspace{-5mm}
		\label{DeRain}
	\end{wraptable}
	
	\subsection{Single Image Rain Removal}
	To further verify the generalizability of our method, we extend our evaluation to single image rain removal, a task that requires eliminating rain streaks while preserving scene details.

	\textbf{Network Architecture Settings.}
	We adopt RCDNet \cite{wang2020model}, a state-of-the-art deep unfolding network, as our baseline. RCDNet unfolds the iterative proximal gradient descent algorithm, employing learnable proximal operators (implemented via ResNet) to jointly estimate rain features and the clean background.
	Since natural scenes often exhibit rich symmetric structures, introducing equivariance into the background estimation module can help preserve such structural features.
	We construct equivariant variants of RCDNet by replacing the standard convolutions in the proximal operators with several state-of-the-art equivariant convolutions: F-Conv \cite{xie2022fourier}, B-Conv \cite{xie2025rotation}, TL-Conv \cite{tan2026image}, and the proposed SA-Conv.
	All equivariant models adopt the $p8$ group and reduce channel number to $\nicefrac{1}{8}$ of the original network to maintain comparable memory usage.
	
	\begin{figure}
		\centering
		% \vspace{-3mm}
		\includegraphics[width=1.0\linewidth]{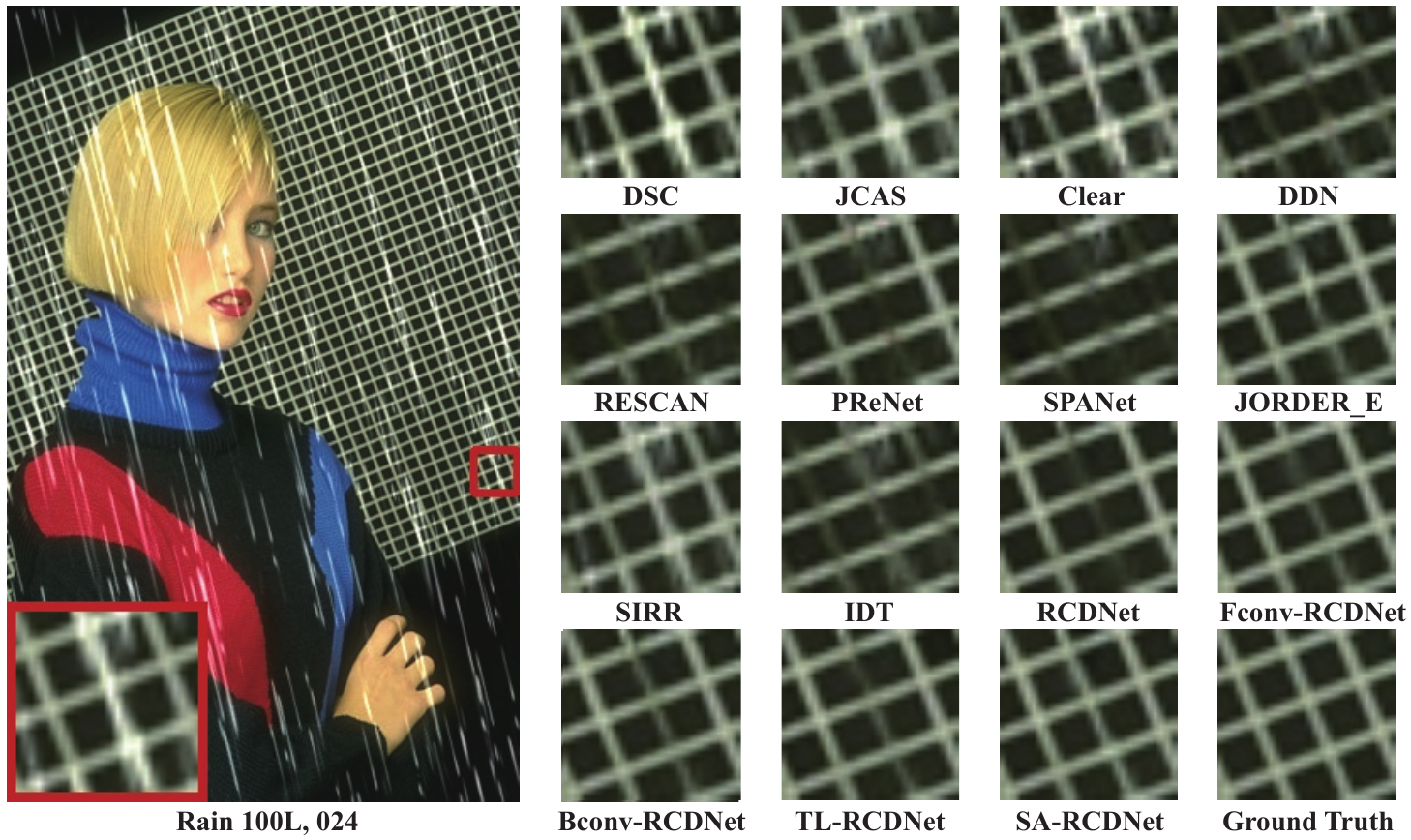}
		\vspace{-6mm}
		\caption{Visual results of Rain removal from various methods on img024 of Rain100L.} 
		\label{Fig:DeRain}
		\vspace{-9mm}
	\end{figure}
	
	\textbf{Datasets and Training Settings.}
	Experiments are conducted on the widely used Rain100L \cite{yang2019joint}, which is the synthesized data
	set with one type of rain streaks. 
	To ensure a rigorous comparison, all training settings and loss functions remain consistent with \cite{wang2020model}.
	
	\textbf{Quantitative and Qualitative Comparison.}
	As shown in Table \ref{DeRain}, the baseline RCDNet is already a high-performance model (PSNR with 40.00 dB). Integrating symmetry priors into the background estimation module steadily improves performance.
	Notably, the proposed SA-RCDNet achieves the best PSNR of 40.54 dB, surpassing the original RCDNet and TLEQ-RCDNet by 0.54 dB and 0.18 dB, respectively.
	Visual results in Fig.~\ref{Fig:DeRain} demonstrate that SA-RCDNet eliminates rain streaks more completely. While competing models may erase background details, our method faithfully preserves the structural integrity of the rain-free scene, validating the benefit of adaptive symmetry in background estimation.
	
	\subsection{Discussion}
	It is worth noting that: (i) TL-Conv, as an equivariant network, also learns its transformations, but it can only learn a consistent transformation group for the entire dataset. It is easy to see that under such learnable transformations, the data symmetry is stronger than that under strict rotation transformations, yet weaker than that under the proposed sample-adaptive transformations. Overall, SA-Conv achieves the best performance, while TL-Conv generally improves on other strict equivariant methods, consistent with the role of symmetry alignment in our empirical risk analysis. (ii) Moreover, we can also observe that most equivariant methods outperform CNN baselines across all backbones. This observation is consistent with the benefit of incorporating symmetry priors that match these datasets.
	
	\textcolor{black}{Furthermore, while the preceding evaluations confirm the efficacy of SA-Conv across standard benchmarks, examining its behavior under progressively weakened data symmetry provides further empirical evidence for the benefits of symmetry alignment. To this end, we conduct a controlled experiment by applying directional deformations of increasing strength to the images, systematically measuring both the quantified symmetry error and the resulting restoration performance. The detailed settings and results, provided in Appendix~{\ref{app:supp:controlled-symmetry}}, demonstrate that even as data symmetry progressively weakens, dynamically adapting to individual samples enables SA-Conv to consistently maintain a distinct performance advantage.}

	\section{Conclusions}\label{sec:conclusions}
	In this paper, we establish a novel framework to formally analyze the relationship between data symmetry and network equivariance. Recognizing the limitations of modeling symmetry strictly on individual samples, we introduce the first quantifiable definition of dataset-level symmetry, which naturally leads to the dataset-constrained restoration inverse problem. 
	Through the analysis of this inverse problem, we prove that the optimal restoration operator satisfies approximate equivariance, with its equivariance error bounded by the data symmetry error and the discretization mesh size. 
	\textcolor{black}{Moreover, by integrating network risk analysis with a generalization framework, we reveal how the symmetry constraints imposed on equivariant models can tighten the generalization bound and explicitly quantify the impact of network equivariance on the overall expected risk upper bound, providing a theoretical motivation for the proposed sample adaptive network design.
	}
	
	This conclusion not only explains the source of the advantage of equivariant networks, but also highlights the importance of selecting the transformation group. It also provides a qualitative explanation for the advantages of traditional CNNs over fully connected networks in image processing tasks: 
	The translation equivariance of CNNs closely aligns with the inherent translation symmetry of image data, suggesting a favorable balance between empirical fitting and generalization when the symmetry mismatch is minimal. Notably, this holds even when individual images do not strictly obey such symmetry, as the network exploits the statistical prevalence of symmetry across the dataset rather than enforcing it on each isolated sample.

	Guided by these findings, we propose the Sample-Adaptive Equivariant Network, which uses a hypernetwork to predict sample-specific transformation parameters, dynamically adapting geometric constraints to each input. Extensive experiments across diverse tasks support our theoretical analysis and demonstrate the significant superiority of the proposed method.
	
	In the future, we intend to generalize this symmetry-aware framework into non-Euclidean domains such as manifolds and graphs. Such an extension will naturally facilitate complex restoration tasks on point clouds and PDE-governed physical systems. Ultimately, we aim to develop a unified symmetry-driven variational theory that accommodates abstract transformation groups and diverse data modalities.

% The former supplementary material is part of this single arXiv document.
% Main-text numbering is retained; appendix equations/theorems use A.1, etc.
\clearpage
\appendix
\numberwithin{figure}{section}
\numberwithin{table}{section}
% Appendices for main.tex. Compile main.tex, not this input file.
\begingroup
\color{black}
\section{Discretization Framework and Assumptions}\label{app:supp:discretization-framework}\label{app:supp:semi-discrete-setting}
This section collects the sampling conventions and assumptions used to compare the continuous, semi-discrete, and discrete restoration problems. We first describe the grids and their continuous representations, then define the discrete operations and optimization problems. The final subsection states the regularity and stability conditions needed in the proofs.

\subsection{Images, Kernels, and Sampling Grids}\label{app:supp:discretization-grids}
Consider $N$ observed images, indexed by $n=1,\ldots,N$. The real-valued functions $\hat r_n,r_n:\mathbb R^2\to\mathbb R$ represent the observed and restored images, respectively. Their sampled arrays are $\hat Y_n\in\mathbb R^{m\times m}$ and $Y_n\in\mathbb R^{m_r\times m_r}$. The blur kernel is $k:\mathbb R^2\to\mathbb R$, with discrete coefficient array $K\in\mathbb R^{p\times p}$. $k^R:\mathbb R^2\to\mathbb R$ denotes the kernel used by the regularizer. We suppress the sample index when discussing a single image. Kernel minimizations are performed separately for each sample, as in the main text.

Let $h>0$ be the spatial mesh size and $s>0$ the fixed downsampling factor. The four grids are summarized below. Each grid size is the number of points in one coordinate direction.
\begin{center}
	\vspace{1.5mm}
\setlength{\tabcolsep}{20pt}      % 列间距，默认约 6pt，可调 10/12/15pt
\renewcommand{\arraystretch}{1.2} % 行高，可选
\small
\begin{tabular}{@{}lllll@{}}
\toprule
Object & Samples & Index set & Grid size & Mesh size\\
\midrule
Observed image $\hat r$ & $\hat Y$ & $\Delta$ & $m$ & $sh$\\
Restored image $r$ & $Y$ & $\Delta_r$ & $m_r$ & $h$\\
Blur kernel $k$ & $K$ & $\Delta_p$ & $p$ & $h$\\
Regularization kernel $k^R$ & $k^R(x_{ij})$ & $\Delta_R$ & $m_R$ & $h$\\
\bottomrule
\end{tabular}
\end{center}
\vspace{1.5mm}
The index sets are
\[
 \Delta=\{1,\ldots,m\}^2,\quad
 \Delta_r=\{1,\ldots,m_r\}^2,\quad
 \Delta_p=\{1,\ldots,p\}^2,\quad
 \Delta_R=\{1,\ldots,m_R\}^2.
\]
All grids are centered at the origin and use cell centers. Retaining the notation of the main text, the coordinates are
\[
\begin{aligned}
 x_{ij}&=\big((i-(m+1)/2)sh,(j-(m+1)/2)sh\big)^T,
 &&(i,j)\in\Delta,\\
 x_{ij}&=\big((i-(m_r+1)/2)h,(j-(m_r+1)/2)h\big)^T,
 &&(i,j)\in\Delta_r,\\
 t_{uv}&=\big((u-(p+1)/2)h,(v-(p+1)/2)h\big)^T,
 &&(u,v)\in\Delta_p,\\
 x_{ij}&=\big((i-(m_R+1)/2)h,(j-(m_R+1)/2)h\big)^T,
 &&(i,j)\in\Delta_R.
\end{aligned}
\]
where $T$ in the superscript denotes vector transpose. The symbol $x_{ij}$ is used for different grids, whose index set specifies the coordinates in each formula. Sampling, with the cell-area weight included in $K$, gives
\begin{equation}\label{app:dis}
 \hat Y_{ij}=\hat r(x_{ij})\quad((i,j)\in\Delta),\quad
 Y_{ij}=r(x_{ij})\quad((i,j)\in\Delta_r),\quad
 K_{uv}=k(t_{uv})h^2\quad((u,v)\in\Delta_p).
\end{equation}
The regularization kernel is sampled at the points indexed by $\Delta_R$.

\subsection{Physical Support and Grid Refinement}\label{app:supp:discretization-windows}
Let $D_r,D_k,D_R>0$ be fixed support-radius bounds for the images, blur kernels, and regularization kernels, respectively. With $\|x\|$ denoting the Euclidean norm of $x\in\mathbb R^2$, the support conditions are
\[
\begin{aligned}
 \hat r(x)=r(x)&=0 &&\text{for }\|x\|\geq D_r,\\
 k(x)&=0 &&\text{for }\|x\|\geq D_k,\\
 k^R(x)&=0 &&\text{for }\|x\|\geq D_R.
\end{aligned}
\]
These bounds also apply to the admissible transformed functions used in the analysis. We choose square sampling windows whose side lengths satisfy
\begin{equation}\label{app:supp:fixed-windows}
 m_rh=msh=2(D_r+D_k),\qquad ph=2D_k,\qquad
 m_Rh=2(D_r+D_R).
\end{equation}
Thus the observed and restored images use the same physical window at different resolutions. The blur window contains the support of $k$, and the regularization window contains the response $y\mapsto\int k^R(x)r(x-y)\,dx$, which vanishes for $\|y\|\geq D_r+D_R$. Likewise, the blur response vanishes for $\|x\|\geq D_r+D_k$. All function evaluations use the zero values outside the stated supports; the quadrature sums do not use periodic wrapping.

The equalities in \eqref{app:supp:fixed-windows} specify the window convention. We consider compatible integer grid sizes, for example by halving $h$ and doubling the grid sizes from an initial compatible grid. The radii and $s$ remain fixed. Consequently, $h\to0$ means denser sampling of fixed physical windows, and the continuous functions remain expressed in the same physical coordinates.

\subsection{Continuous Representations and Transformations}\label{app:supp:discretization-representations}
An admissible continuous representation of $Y$ is a smooth function $r$ in the prescribed image class whose values on $\Delta_r$ equal $Y$. Observed images are represented in the same way on $\Delta$. The class used in the continuous and semi-discrete problems is fixed under grid refinement and is preserved by the transformations considered below.

We assume that the sampled convolution, regularizer, and transformations are determined by the image samples under this representation convention. In particular, the values needed at shifted or transformed locations, such as $r(x_{ij}-t_{uv})$, $r(x_{ij}-x_{\hat i\hat j})$, and $r(A^{-1}x_{ij})$, should be consistent with that convention. If several admissible representations have the same samples, they must give the same discrete operations. The continuous functional $R[r]$ may still depend on the continuous representation. These are compatibility assumptions on the representation class: smoothness alone does not determine values away from a finite sampling grid.

For sample $n$, let $G_n$ be its prescribed finite group of invertible $2\times2$ transformation matrices with $T$ elements, where $T$ denotes the group order. For $A\in G_n$, assume $|\det A|=1$ and define the continuous action by
\[
 \pi_A[r](x)=r(A^{-1}x).
\]
The corresponding discrete action on a restored image satisfies
\begin{equation}\label{app:sampling-identities}
 \big(\tilde\pi_A(Y)\big)_{ij}=r(A^{-1}x_{ij}),
 \qquad (i,j)\in\Delta_r.
\end{equation}
For an observed image, the same convention gives
\[
 \big(\tilde\pi_A(\hat Y)\big)_{ij}=\hat r(A^{-1}x_{ij}),
 \qquad (i,j)\in\Delta.
\]
The prescribed representation class and these identities make the discrete transformations consistent with the continuous group action, including composition and inversion. This compatibility is required even when transformed points do not lie on the original grid.

The group-assignment map $\mathcal H_G$ associates a sample with its group and agrees on corresponding continuous and discrete representations. As in Assumption~\ref{ass:adaptive_group}, it satisfies
\[
 \mathcal H_G(r_n)=\mathcal H_G(Y_n)=G_n,\qquad
 \mathcal H_G(\pi_A[r_n])=\mathcal H_G(\tilde\pi_A(Y_n))=G_n
 \quad(A\in G_n).
\]
The same requirements apply to observed inputs. For each dataset under consideration, the group family is held fixed throughout the restoration optimization.

\subsection{Convolution and Regularization}\label{app:supp:discretization-operations}
The continuous blur response at $x\in\mathbb R^2$ is $\int_{\mathbb R^2}k(t)r(x-t)\,dt$, where $t$ is the integration variable and $dt$ denotes area measure. Let $D$ denote the downsampling operator in the discrete model. The notation $D(K*Y)$ is interpreted through the quadrature convention
\begin{equation}\label{app:sampled-convolution}
\begin{aligned}
 \big[D(K*Y)\big]_{ij}
 &=\sum_{(u,v)\in\Delta_p}K_{uv}r(x_{ij}-t_{uv})\\
 &=\sum_{(u,v)\in\Delta_p}k(t_{uv})r(x_{ij}-t_{uv})h^2,
 \qquad (i,j)\in\Delta.
\end{aligned}
\end{equation}
where $x_{ij}$ is an observed-image coordinate. The coefficients $K_{uv}=k(t_{uv})h^2$ include the blur-grid cell area $h^2$. Applying quadrature to the data fidelity integral on the observed-image grid introduces the further cell-area weight $(sh)^2$:
\[
 \int_{\mathbb R^2}\left(\hat r(x)-\int_{\mathbb R^2}k(t)r(x-t)\,dt\right)^2dx
 \ \approx\ (sh)^2\sum_{(i,j)\in\Delta}
 \big(\hat Y_{ij}-[D(K*Y)]_{ij}\big)^2.
\]
The approximation error is estimated in the subsequent proofs.

For regularization, let $\rho:\mathbb R\to[0,\infty)$ be a scalar penalty applied to the local response produced by $k^R$. Its regularity conditions are stated in Section~\ref{app:supp:discretization-assumptions}. The continuous functional and its discrete counterpart are
\begin{equation}\label{app:supp:regularizer-definitions}
\begin{aligned}
 R[r]&=\frac{1}{m_R^2h^2}\int_{\mathbb R^2}
 \rho\!\left(\int_{\mathbb R^2}k^R(x)r(x-y)\,dx\right)dy,\\
 \tilde R(Y)&=\frac{1}{s^2h^2m_R^2}
 \sum_{(\hat i,\hat j)\in\Delta_R}
 \rho\!\left(\sum_{(i,j)\in\Delta_R}
 k^R(x_{ij})r(x_{ij}-x_{\hat i\hat j})h^2\right).
\end{aligned}
\end{equation}
In this display, $x,y\in\mathbb R^2$ are integration variables, and all sampled coordinates belong to the regularization grid. The pairs $(i,j)$ and $(\hat i,\hat j)$ index the inner and outer sums, respectively. The shifted image values are supplied by the representation convention in Section~\ref{app:supp:discretization-representations}, so $\tilde R$ acts on $Y$.

The continuous normalization is by the fixed area $m_R^2h^2=4(D_r+D_R)^2$. Replacing both integrals by their quadrature sums gives
\[
 (sh)^2\tilde R(Y)=\frac{1}{m_R^2}
 \sum_{(\hat i,\hat j)\in\Delta_R}
 \rho\!\left(\sum_{(i,j)\in\Delta_R}
 k^R(x_{ij})r(x_{ij}-x_{\hat i\hat j})h^2\right).
\]
Thus the quadrature approximation to $R[r]$ is $(sh)^2\tilde R(Y)$. This scaling places the regularizer on the same scale as the data fidelity term when the full discrete objective is multiplied by $(sh)^2$.

\subsection{Feasible Sets and Restoration Objectives}\label{app:supp:discretization-objectives}
Let $\varepsilon\geq0$ bound the variation of the dataset-averaged regularizer under the prescribed transformations. Within the admissible continuous image class, the relaxed symmetry constraint is
\[
\begin{aligned}
 \Omega_\varepsilon=\Bigg\{\{r_n\}_{n=1}^N\ \Bigg|\ 
 \frac1N\left|\sum_{n=1}^N R[\pi_{A_n}[r_n]]
 -\sum_{n=1}^N R[\pi_{B_n}[r_n]]\right|\leq\varepsilon,
  A_n,B_n\in G_n, \forall n=1,\ldots,N\Bigg\}.
\end{aligned}
\]
The inequality is required for every choice of $A_n=A_n^{t_1}$ and $B_n=A_n^{t_2}$, with the same indices $t_1,t_2$ used for all samples, following the convention in the main text. Together with closure of the representation class under these actions, this makes $\Omega_\varepsilon$ invariant under these samplewise transformations.

The discrete feasible set is the sampling image of this continuous set:
\[
\begin{aligned}
 \tilde\Omega_\varepsilon=\Big\{\{Y_n\}_{n=1}^N\ \Big|\ 
 \exists \{r_n\}_{n=1}^N\in\Omega_\varepsilon
 \text{ with }(Y_n)_{ij}=r_n(x_{ij}),
 n=1,\ldots,N,\quad(i,j)\in\Delta_r\Big\}.
\end{aligned}
\]
The admissible discrete blur kernels likewise consist exactly of the samples of the admissible continuous kernels multiplied by $h^2$. These correspondences are used when comparing optimization problems and apply to both original and transformed datasets.

For a fixed observed dataset, define the continuous and semi-discrete per-image objectives by
\begin{equation}\label{app:supp:continuous-semidiscrete-objectives}
\begin{aligned}
 F[r_n]&=\min_k\int_{\mathbb R^2}
 \left(\hat r_n(x)-\int_{\mathbb R^2}k(t)r_n(x-t)\,dt\right)^2dx+R[r_n],\\
 \tilde F[r_n]&=(sh)^2\min_k\sum_{(i,j)\in\Delta}
 \left(\hat r_n(x_{ij})-\sum_{(u,v)\in\Delta_p}
 k(t_{uv})r_n(x_{ij}-t_{uv})h^2\right)^2+R[r_n].
\end{aligned}
\end{equation}
Both minimizations over $k$ use the same admissible continuous kernel class. The semi-discrete objective retains the continuous regularizer $R[r_n]$ and discretizes only the data fidelity term. The full discrete objective is
\begin{equation}\label{app:supp:discrete-objective-definition}
 J_h(\{Y_n\}_{n=1}^N)=\sum_{n=1}^N
 \left[\min_K\|\hat Y_n-D(K*Y_n)\|_F^2+\tilde R(Y_n)\right],
\end{equation}
where $\|\cdot\|_F$ is the matrix Frobenius norm and the minimization is over the admissible discrete kernels.

The operators $\Phi$, $\bar\Phi$, and $\tilde\Phi$ return selected minimizing collections of $\sum_nF[r_n]$, $\sum_n\tilde F[r_n]$, and $J_h$, respectively. The first two optimize over $\Omega_\varepsilon$, while the third optimizes over $\tilde\Omega_\varepsilon$. Their outputs are continuous functions for $\Phi$ and $\bar\Phi$, and image arrays for $\tilde\Phi$. The notation $\Phi[\hat r_n]$, $\bar\Phi[\hat r_n]$, or $\tilde\Phi(\hat Y_n)$ denotes the $n$th component of the corresponding jointly optimized collection; dependence on the rest of the dataset is implicit.

Since $(sh)^2>0$, scaling the complete discrete objective preserves its minimizing collections:
\begin{equation}\label{app:scaled-argmin}
 \operatorname*{arg\,min}_{\{Y_n\}\in\tilde\Omega_\varepsilon}J_h(\{Y_n\})
 =\operatorname*{arg\,min}_{\{Y_n\}\in\tilde\Omega_\varepsilon}(sh)^2J_h(\{Y_n\}).
\end{equation}
For corresponding images and representations, the fidelity terms of $\sum_n\tilde F[r_n]$ and $(sh)^2J_h$ agree. Their regularizers are $R[r_n]$ and $(sh)^2\tilde R(Y_n)$, whose discrepancy is controlled by the regularization quadrature estimate. The comparison concerns distinct optimization problems; the discrete minimizing images are not defined as samples of the continuous minimizers.

\subsection{Regularity and Stability Assumptions}\label{app:supp:discretization-assumptions}
The representation and feasible-set conditions above specify how the continuous and discrete problems are related. The error estimates also require uniform regularity and local stability, as stated in the main text. We collect these requirements here for completeness.

\textit{Uniform function bounds.}
The admissible image and kernel functions have continuous derivatives through order two.
\begingroup
For each compatible mesh and every choice of $A_n\in G_n$, $n=1,\ldots,N$, let $r_n$ and $r_n^A$ be the selected admissible continuous representations of the original and transformed-input discrete minimizers, respectively, so that
\[
\begin{aligned}
 r_n(x_{ij})&=\bigl(\tilde\Phi(\hat Y_n)\bigr)_{ij},\\
 r_n^A(x_{ij})&=\bigl(\tilde\Phi(\tilde\pi_{A_n}(\hat Y_n))\bigr)_{ij},
 \qquad (i,j)\in\Delta_r.
\end{aligned}
\]
These collections belong to $\Omega_\varepsilon$ and the respective quadratic-growth neighborhoods described below. By \eqref{app:sampling-identities}, the continuous representation used for $\tilde\pi_{A_n}(\tilde\Phi(\hat Y_n))$ is $\pi_{A_n}[r_n]$.
In the notation of Theorem~\ref{Error-dis}, there are fixed constants $F_j,G_j,H_j>0$, $j=1,2,3$, such that, for every compatible mesh $0<h\leq h_0$, every choice of $A_n\in G_n$, every feasible continuous blur kernel $k$, and all $x\in\mathbb R^2$,
\[
\begin{aligned}
 |r|&\leq F_1,& |k|&\leq F_2,& |k^R|&\leq F_3,\\
 \|\nabla r\|&\leq G_1,&
 \|\nabla k\|&\leq G_2,& \|\nabla k^R\|&\leq G_3,\\
 \|\nabla^2r\|&\leq H_1,&
 \|\nabla^2k\|&\leq H_2,& \|\nabla^2k^R\|&\leq H_3,
\end{aligned}
\]
where, for each $n=1,\ldots,N$,
\[
 r\in\bigl\{\pi_{A_n}[\hat r_n],\ \Phi[\pi_{A_n}[\hat r_n]],\ \pi_{A_n}[r_n],\ r_n^A\bigr\}.
\]
The same functions satisfy the support conditions in Section~\ref{app:supp:discretization-windows}, and $h_0>0$ is a fixed upper mesh bound.
\par\endgroup
Here $\nabla$ is the gradient and $\nabla^2$ the Hessian, and their norms are the Euclidean and induced matrix operator norms, respectively. The constants $F_j,G_j,H_j$, for $j=1,2,3$, bound function values, gradients, and Hessians for the images, blur kernels, and regularization kernels. Stronger smoothness assumptions used in the continuous-domain lemmas remain as stated there.

\textit{Scalar penalty.}
The function $\rho$ is continuous, takes nonnegative values, satisfies $\rho(0)=0$, and is $C^2$ away from at most $M_\rho$ points, where $M_\rho$ is a nonnegative integer independent of $h$. At each smooth point in the response interval, assume
\[
 |\rho'(z)|\leq G_\rho,\qquad |\rho''(z)|\leq H_\rho,
 \qquad |z|\leq4F_kF_r(D_r+D_R)^2.
\]
The variable $z$ is the scalar convolution response, and $G_\rho,H_\rho$ are derivative bounds independent of $h$. To match Lemma~\ref{error-R}, the image bounds are denoted there by $(F_r,G_r,H_r)=(F_1,G_1,H_1)$ and the regularization-kernel bounds by $(F_k,G_k,H_k)=(F_3,G_3,H_3)$. In particular, $F_k$ in this response interval bounds $k^R$. No derivative value is assigned at a nonsmooth point of $\rho$.

\textit{Minimizers and local quadratic growth.}
The minimizers invoked in the comparisons are assumed to exist. The uniqueness assumption on the inner continuous blur-kernel problem is the one stated in the main text. For a fixed observed dataset, let $r_n^*=\Phi[\hat r_n]$ denote the selected continuous minimizing collection. There are a constant $\mu>0$ and a neighborhood $\mathcal N$ of this collection such that
\[
 \sum_{n=1}^N\big(F[r_n]-F[r_n^*]\big)
 \geq\mu\sum_{n=1}^N\|r_n-r_n^*\|_2^2,
 \qquad \{r_n\}_{n=1}^N\in\Omega_\varepsilon\cap\mathcal N,
\]
where $\|f\|_2^2=\int_{\mathbb R^2}|f(x)|^2\,dx$ is the squared continuous $L^2$ norm. As required by Assumption~\ref{ass:qg}, $\mathcal N$ contains the selected semi-discrete minimizing collection, admissible continuous representations of the selected discrete minimizing collection, and
\[
 \left\{\pi_{A_n}^{-1}\!\left[\Phi[\pi_{A_n}[\hat r_n]]\right]\right\}_{n=1}^N,
 \qquad A_n\in G_n.
\]
In this expression, $\pi_{A_n}^{-1}=\pi_{A_n^{-1}}$ is the inverse group action. The same requirements hold around the selected continuous minimizer for each transformed dataset, with a common $\mu$ independent of the mesh and the transformations. The neighborhood inclusions are assumed uniformly for $0<h\leq h_0$. They are stability assumptions used to convert objective discrepancies into distances between minimizers.
\par\endgroup

	\section{Proof of Lemma~\ref{varphi}}
	\begin{proof}
		For continuous image functions $\hat{r}(x),r(x)\in C^\infty(\mathbb{R}^2)$, the unique minimizing filter satisfies:
		\begin{equation}
			\begin{split}
				\varphi[\pi_A[\hat{r}],\pi_A[r]]
				= & \arg \min_{k}\int_{\R^2}\left(\pi_A[\hat{r}](x)-\int_{\R^2}k(t)\pi_A[r](x-t)dt\right)^{2}dx\\
				= & \arg \min_{k}\int_{\R^2}\left(\hat{r}(A^{-1}x)-\int_{\R^2}k(t)r(A^{-1}x-A^{-1}t)dt\right)^{2}dx.
			\end{split}
		\end{equation}
		Let $\hat{k}(t)=\pi_A^{-1}[k](t)=k(At)$, $\hat{t}=A^{-1}t$, and $\hat{x}=A^{-1}x$. This substitution is invertible, with $k(t)=\hat k(A^{-1}t)$. Since $|\det(A)|=1$, we have:
		\begin{equation}
			\begin{split}
				\varphi[\pi_A[\hat{r}],\pi_A[r]] 
				= & \arg \min_{k}\int_{\R^2}\left(\hat{r}(\hat{x})-\int_{\R^2}k(A\hat{t})r(\hat{x}-\hat{t})d\hat{t}\right)^{2}d\hat{x} \\
				= & \pi_A \arg \min_{\hat{k}}\int_{\R^2}\left(\hat{r}(\hat{x})-\int_{\R^2}\hat{k}(\hat{t})r(\hat{x}-\hat{t})d\hat{t}\right)^{2}d\hat{x}\\
				= & \pi_A[\varphi[\hat{r},r]].
			\end{split}
		\end{equation}
		The proof is completed.
	\end{proof}

	\section{Proof of Lemma~\ref{norm-preserve}}\label{app:supp:norm-preserve}

	\begin{proof}
		According to the definition of $L_2$ norm, we can deduce:
		\begin{equation}
			\|\pi_{A}[r](x)\|_2^2=\int_{\R^2}\left(\pi_{A}[r](x)\right)^2 dx=\int_{\R^2}\left(r(A^{-1}x)\right)^2 dx.
		\end{equation}
		Let $\bar{x}=A^{-1}x$. Since $|\det(A)|=1$, we have
		\begin{equation}
			\|\pi_{A}[r](x)\|_2^2=\int_{\R^2}\left(r(\bar{x})\right)^2 d\bar{x}=\|r(x)\|_2^2.
		\end{equation}
	\end{proof}

	\section{Proof of Lemma~\ref{tan}}
	\label{app:proof-tan}
	
	Before proving Lemma 3.4, we state several necessary lemmas, continuing from the framework established by Xie et al. \cite{xie2022fourier}.

\begingroup
\begin{lemma}\label{app:Hessian} 
		For a smooth function $f:\R^2\to\R$, if the following results are satisfied:
		\begin{equation}
			\begin{split}
				\max\left\{\left|\frac{\partial^2f(x)}{\partial x_1^2} \right|, \left|\frac{\partial^2f(x)}{\partial x_2^2} \right|  \right\}
				\leq C, ~~
				\forall \|x\|\geq D, ~f(x)=0,
			\end{split}
		\end{equation}
		where $D>0$ is a fixed support radius, then we have:
		\begin{equation}
			\begin{split}
				\max&\left\{\left|\frac{d^2}{d x_2^2}\int_{\R}f(x)dx_1 \right|, \left|\frac{d^2}{d x_1^2}\int_{\R}f(x)dx_2 \right|  \right\}
				\leq 2CD.
			\end{split}
		\end{equation}
		
	\end{lemma}
	%It is noteworthy that a more general conclusion can be obtained when $\A$ is taken as the identity matrix. Specifically, for a given function $f(x)$, we can derive upper bounds for its associated second-order partial derivatives and the second-order derivatives of the related functions defined by integrals.
\par\endgroup
\begin{lemma}\label{app:linear} 
		For any $f: \R\to\R$, let $l$ denote the linear interpolation function between points $x = a$ and $x=b$, i.e.,
		\begin{equation}
			l(x) = \frac{(x-a)f(b)+(b-x)f(a)}{b-a}.
		\end{equation}
		If the second derivative $f''$ is continuous in $[a,b]$ and $|f''(x)|\leq C$, then it holds that:
		\begin{equation}\label{app:Lemma4}
			\left|f(x)-l(x)\right|\leq\frac{C}{8}(a-b)^2.
		\end{equation}
	\end{lemma}
	
	The proofs of the above lemmas can be found in \cite{xie2022fourier}. Now let us prove Lemma 3.4 in the main text.

\begingroup
\begin{proof} 
		{Let $t_k=(k-(M+1)/2)h$ for $k\in\mathbb{Z}$, where $Mh=2D$, so that $x_{ij}=(t_i,t_j)^T$. The range condition implies that $f(x)=0$ whenever either coordinate lies outside $(t_0,t_{M+1})$.}
		
		1) We first prove that for any $ n\in \mathbb{Z}$, it holds that:
		{
			\begin{equation}\label{app:3_1}
				\left| \int_{\R}f\left( \begin{bmatrix} x_1 \\  t_n  \end{bmatrix}\right)dx_1\!\!-\!\!\sum_{k\in\mathbb{Z}} f\left( \begin{bmatrix} t_k \\  t_n  \end{bmatrix}\right)h \right| \leq \frac{C(M+1)}{8}h^3.
			\end{equation}
		}
		{By Lemma \ref{app:linear}, we can obtain that $\forall k,n\in\mathbb{Z}, t_k\leq x_1 < t_{k+1}$, it holds that:}
		{
			\begin{equation}\label{app:3_2}
				\left| f\left( \begin{bmatrix} x_1 \\  t_n  \end{bmatrix}\right)- l_{kn}\left( \begin{bmatrix} x_1 \\  t_n  \end{bmatrix} \right)  \right| \leq \frac{C}{8}h^2,
			\end{equation}
		}
		{where $l_{kn}$ is the linear interpolation function of $f$ between points $x = \begin{bmatrix} t_k \\  t_n  \end{bmatrix}$ and $x=\begin{bmatrix} t_{k+1} \\  t_n  \end{bmatrix}$, i.e.,}\small
		{
			\begin{equation}\label{app:Lx}
				\begin{split}
					l_{kn}(x) = \frac{(x_1-t_k)f\left( \begin{bmatrix} t_{k+1} \\  t_n  \end{bmatrix}\right)+(t_{k+1}-x_1)f\left( \begin{bmatrix} t_k \\  t_n  \end{bmatrix}\right)}{h}.
				\end{split}
			\end{equation}
		}
		Thus,
		{
			\begin{equation}\label{app:3_3}
				\begin{split}
					&\left| \int_{t_k}^{t_{k+1}}\l(f\left( \begin{bmatrix} x_1 \\  t_n  \end{bmatrix}\right) \r)dx_1 - \int_{t_k}^{t_{k+1}}\l(l_{kn}\left( \begin{bmatrix} x_1 \\  t_n  \end{bmatrix}\right)  \r)dx_1  \right| \\
					=&\left| \int_{t_k}^{t_{k+1}}\l(f\left( \begin{bmatrix} x_1 \\  t_n  \end{bmatrix}\right) -l_{kn}\left( \begin{bmatrix} x_1 \\  t_n  \end{bmatrix}\right)  \r)dx_1  \right|\\
					\leq & \int_{t_k}^{t_{k+1}}\l|f\left( \begin{bmatrix} x_1 \\  t_n  \end{bmatrix}\right) -l_{kn}\left( \begin{bmatrix} x_1 \\  t_n  \end{bmatrix}\right)  \r|dx_1  \\
					\leq & \int_{t_k}^{t_{k+1}}  \frac{C}{8}h^2  dx_1 = \frac{C}{8}h^3.
				\end{split}
			\end{equation}
		}
		Besides, we have:
		{
			\begin{equation}\label{app:3_4}
				\begin{split}
					&\left| \int_{\R}f\left( \begin{bmatrix} x_1 \\  t_n  \end{bmatrix}\right)dx_1-\sum_{k\in\mathbb{Z}} f\left( \begin{bmatrix} t_k \\  t_n  \end{bmatrix}\right)h \right|\\
					= & \left| \int_{\R}f\left( \begin{bmatrix} x_1 \\  t_n  \end{bmatrix}\right)dx_1-\sum_{k\in\mathbb{Z}} \frac{h}{2}\l( f\left( \begin{bmatrix} t_k \\  t_n  \end{bmatrix}\right)+  f\left( \begin{bmatrix} t_{k+1} \\  t_n  \end{bmatrix}\right) \r)\right|  \\
					= & \left| \int_{\R}f\left( \begin{bmatrix} x_1 \\  t_n  \end{bmatrix}\right)dx_1-\sum_{k\in\mathbb{Z}}  \int_{t_k}^{t_{k+1}}\l(l_{kn}\left( \begin{bmatrix} x_1 \\  t_n  \end{bmatrix}\right)  \r)dx_1 \right|  \\
					\leq &\sum_{k\in\mathbb{Z}}\left| \int_{t_k}^{t_{k+1}}f\left( \begin{bmatrix} x_1 \\  t_n  \end{bmatrix}\right)dx_1-  \int_{t_k}^{t_{k+1}}\l(l_{kn}\left( \begin{bmatrix} x_1 \\  t_n  \end{bmatrix}\right)  \r)dx_1 \right|.
				\end{split}
			\end{equation}
		}
		{Since $\forall \|x\|\geq D,~ f(x)=0$, there are at most $M+1$ values of $k$ (namely, $k=0,1,\cdots,M$), s.t.,}
		{
			\begin{equation}\label{app:3_5}
				\begin{split}
					&\left| \int_{t_k}^{t_{k+1}}\l(f\left( \begin{bmatrix} x_1 \\  t_n  \end{bmatrix}\right) \r)dx_1 -\int_{t_k}^{t_{k+1}}\l(l_{kn}\left( \begin{bmatrix} x_1 \\  t_n  \end{bmatrix}\right)  \r)dx_1  \right|\neq 0,
				\end{split}
			\end{equation}
		}
		while all of them satisfy (\ref{app:3_3}). Then, by (\ref{app:3_3}), (\ref{app:3_4}) and (\ref{app:3_5}), we can easily achieve  (\ref{app:3_1}).
		
		2) We then prove
		{
			\begin{equation}\label{app:3_6}
				\begin{split}
					\left| \int_{\R^2}f\left( x \right)dx-\sum_{n\in\mathbb{Z}}\int_{\R} f\left( \begin{bmatrix} x_1 \\  t_n  \end{bmatrix}\right)dx_1h \right|
					\leq \frac{CD(M+1)}{4}h^3.
				\end{split}
			\end{equation}
		}
		Let $F(x_2) = \int_{\R}f\left( \begin{bmatrix} x_1 \\  x_2  \end{bmatrix}\right)dx_1$, and then
		{
			\begin{equation}\label{app:3_7}
				\int_{\R^2}f\left( x \right)dx = \int_{\R}F\left( x_2 \right)dx_2.
			\end{equation}
		}  
		{By Lemma \ref{app:Hessian}, we can obtain that $\l|F''(x_2)\r|\leq 2CD$. Let $L_n(x_2)$ be the linear interpolation function of $F(x_2)$, between points $x_2 =t_n $ and $x_2=t_{n+1}$, and then for any $x_2\in [t_n,t_{n+1}]$, by Lemma \ref{app:linear}, we have:}
		{
			\begin{equation}\label{app:3_8}
				\left| F(x_2)-L_n(x_2) \right|\leq \frac{CD}{4}h^2.
			\end{equation}
		}
		Thus, similar to (\ref{app:3_3}), we can deduce that:
		{
			\begin{equation}\label{app:3_9}
				\begin{split}
					\l|  \int_{t_n}^{t_{n+1}}F(x_2)dx_2 - \int_{t_n}^{t_{n+1}}L_n(x_2)dx_2  \r|
					\leq \int_{t_n}^{t_{n+1}}\l|F(x_2)-L_n(x_2)\r|dx_2
					\leq \frac{CD}{4}h^3.
				\end{split}
			\end{equation}
		}
		Then, similar to (\ref{app:3_4}), we can deduce that:
		{
			\begin{equation}\label{app:3_10}
				\begin{split}
					&\left| \int_{\R^2}f\left( x \right)dx-\sum_{n\in\mathbb{Z}}\int_{\R} f\left( \begin{bmatrix} x_1 \\  t_n  \end{bmatrix}\right)dx_1h \right|\\
					= &\left| \int_{\R}F(x_2)dx_2-\sum_{n\in\mathbb{Z}}F(t_n)h \right|\\
					= &\left| \int_{\R}F(x_2)dx_2-\sum_{n\in\mathbb{Z}}\frac{1}{2}(F(t_n)+F(t_{n+1}))h \right|\\
					\leq & \sum_{n\in\mathbb{Z}}\left| \int_{t_n}^{t_{n+1}}F(x_2)dx_2-\frac{1}{2}(F(t_n)+F(t_{n+1}))h \right|.
				\end{split}
			\end{equation}
		}
		{Since $\forall \|x\|\geq D,~ f(x)=0$, there are at most $M+1$ values of $n$ (namely, $n=0,1,\cdots,M$), s.t.,}
		{
			\begin{equation}\label{app:3_11}
				\left| \int_{t_n}^{t_{n+1}}F(x_2)dx_2-\frac{1}{2}(F(t_n)+F(t_{n+1}))h \right| \neq 0,
			\end{equation}
		}
		while all of them satisfy (\ref{app:3_9}). Then, by (\ref{app:3_9}), (\ref{app:3_10}) and (\ref{app:3_11}), we can easily achieve  (\ref{app:3_6}).
		
		3) By (\ref{app:3_1}) and (\ref{app:3_6}), we can obtain:
		{
			\begin{equation}\label{app:3_12}
				\begin{split}
					&\left|\int_{\R^2}f(x)dx - \sum_{(i,j)\in\Delta_M}f(x_{ij})h^2\right| \\
					= & \l| \int_{\R^2}f(x)dx -\sum_{k,n\in\mathbb{Z}} f\l(\begin{bmatrix} t_k \\  t_n  \end{bmatrix}\r)h^2 \r|\\
					\leq &  \l| \int_{\R^2}f(x)dx -\sum_{n\in\mathbb{Z}} \int_{\R}f\l(\begin{bmatrix} x_1 \\  t_n  \end{bmatrix}\r)dx_1 h\r| +
					\sum_{n\in\mathbb{Z}}\l|\int_{\R}f\l(\begin{bmatrix} x_1 \\  t_n  \end{bmatrix}\r)dx_1 - \sum_{k\in\mathbb{Z}} f\l(\begin{bmatrix} t_k \\  t_n  \end{bmatrix}\r)h    \r|h,
				\end{split}
			\end{equation}
		}
		where, by (\ref{app:3_6}),
		{
			\begin{equation}\label{app:3_13}
				\begin{split}
					\l| \int_{\R^2}f(x)dx -\sum_{n\in\mathbb{Z}} \int_{\R}f\l(\begin{bmatrix} x_1 \\  t_n  \end{bmatrix}\r)dx_1 h\r| 
					\leq \frac{CD(M+1)}{4}h^3.
				\end{split}
			\end{equation}
		}
		{Besides, since $\forall \|x\|\geq D,~ f(x)=0$, there are at most $M$ values of $n$ (namely, $n=1,2,\cdots,M$), s.t.,}
		{
			\begin{equation}\label{app:3_14}
				\l|\int_{\R}f\l(\begin{bmatrix} x_1 \\  t_n  \end{bmatrix}\r)dx_1 - \sum_{k\in\mathbb{Z}} f\l(\begin{bmatrix} t_k \\  t_n  \end{bmatrix}\r)h    \r|\neq 0.
			\end{equation}
		}
		Thus, by (\ref{app:3_1}), we have:
		{
			\begin{equation}\label{app:3_15}
				\begin{split}
					\sum_{n\in\mathbb{Z}}\l|\int_{\R}f\l(\begin{bmatrix} x_1 \\  t_n  \end{bmatrix}\r)dx_1 - \sum_{k\in\mathbb{Z}} f\l(\begin{bmatrix} t_k \\  t_n  \end{bmatrix}\r)h    \r|h 
					\leq  \frac{CM(M+1)}{8}h^4 \leq \frac{CD(M+1)}{4}h^3.
				\end{split}
			\end{equation}
		}
		{Combining (\ref{app:3_12}), (\ref{app:3_13}) and (\ref{app:3_15}), and using $Mh=2D$ and $h\leq2D$, we obtain the bound $CD(M+1)h^3/2=CD(2D+h)h^2/2\leq2CD^2h^2$, which completes the proof.}
	\end{proof}
\par\endgroup
\section{Proof of Lemma~\ref{error-R}}
	\label{app:proof-error-R}
	
\begingroup\color{black}
\begin{proof}
		Let $f(x)=k^R(x)r(x-y)$, then we have:
		\begin{equation}
			\begin{split}
				\nabla^2 \!\!f(x)\!\!=\!\!\nabla^2 k^R(x)r(\!x\!-\!y)
				\!\!+\!\! \nabla k^R\!(x)(\nabla \!r(x\!-\!y))^T
				\!\!\!+\!\!\!\nabla r(x\!-\!y)(\nabla k^R\!(x))^T
				\!\!\!+\! k^R(x)\nabla^2 \!r(x\!\!-\!\!y).
			\end{split}
		\end{equation}
		From the known boundedness conditions:
		\begin{equation}
			\begin{split}
				& \|r(x)\|\leq F_r, ~\|k^R(x)\|\leq F_k, \\
				& \|\nabla {r}(x)\|\leq G_r, ~\|\nabla k^R(x)\| \leq G_k, \\
				& \|\nabla^2 {r}(x)\| \leq H_r, ~\|\nabla^2 k^R(x) \|\leq H_k,
			\end{split}
		\end{equation}
		we can deduce that:
		\begin{equation}
			\|\nabla^2 f(x)\| \leq H_k F_r+2G_r G_k+F_k H_r.
		\end{equation}
		
Let $L=2(D_r+D_R)=m_Rh$, the fixed side length of the regularization window, and denote the sampled inner response by
\[
 u_h(y)=\sum_{(i,j)\in\Delta_R}k^R(x_{ij})r(x_{ij}-y)h^2.
\]
Lemma~\ref{tan}, applied on $[-L/2,L/2]^2$, gives uniformly in $y$
\begin{equation}\label{app:R-temp}
 \left|\int_{\mathbb R^2}k^R(x)r(x-y)\,dx-u_h(y)\right|
 \leq 2(H_kF_r+2G_rG_k+F_kH_r)(D_r+D_R)^2h^2.
\end{equation}
Both responses vanish for $\|y\|\geq D_r+D_R$. Moreover,
$\int|k^R|\leq F_kL^2$ and $\sum_{\Delta_R}|k^R(x_{ij})|h^2\leq F_kL^2$,
so the responses lie in the fixed interval of Lemma~\ref{error-R}.
Continuity of $\rho$ and the bound on its first derivative on each smooth piece make it $G_\rho$-Lipschitz on this interval. Splitting the error using \eqref{R-dis} gives
\begin{equation}\label{app:supp:reg-split}
\begin{aligned}
 |R[r]-(sh)^2\tilde R(Y)|
 &\leq \frac1{L^2}\int_{[-L/2,L/2]^2}
 \left|\rho\!\left(\int k^R(x)r(x-y)\,dx\right)-\rho(u_h(y))\right|\,dy\\
 &\quad+\frac1{L^2}\left|\int_{[-L/2,L/2]^2}\rho(u_h(y))\,dy
       -h^2\sum_{(\hat i,\hat j)\in\Delta_R}\rho(u_h(x_{\hat i\hat j}))\right|.
\end{aligned}
\end{equation}
By \eqref{app:R-temp}, the first term is at most
$2G_\rho(H_kF_r+2G_rG_k+F_kH_r)(D_r+D_R)^2h^2$.

For the second term we account explicitly for the finitely many nonsmooth values of $\rho$.
Differentiation of the finite sum shows that, for $a=1,2$,
\[
 \partial_a u_h(y)=-\sum_{\Delta_R}k^R(x_{ij})\partial_a r(x_{ij}-y)h^2,
 \qquad
 \partial_{aa}u_h(y)=\sum_{\Delta_R}k^R(x_{ij})\partial_{aa} r(x_{ij}-y)h^2.
\]
Consequently every coordinate-line restriction $u$ of $u_h$ satisfies
$|u'|\leq G$ and $|u''|\leq H$, where the temporary constants
\[
 G=F_kG_rL^2,\qquad H=F_kH_rL^2
\]
are independent of $h$. At each exceptional value $a$ in the interior of the response interval, the one-sided derivatives of $\rho$ exist because $|\rho''|\leq H_\rho$ on each adjacent piece. Define
\[
 c_a=\rho'_+(a)-\rho'_-(a),\qquad
 q(z)=\rho(z)-\sum_a c_a(z-a)_+,
\]
where $(z-a)_+=\max\{z-a,0\}$ and the sums range over those at most $M_\rho$ values. Since $|c_a|\leq2G_\rho$, we have $\sum_a|c_a|\leq2M_\rho G_\rho$. Then $q'$ is $H_\rho$-Lipschitz and $|q'|\leq(1+2M_\rho)G_\rho$.
Along a coordinate interval of length $L$, the derivative of $q(u)$ has total variation at most
$L[H_\rho G^2+(1+2M_\rho)G_\rho H]$.
For each hinge $v=(u-a)_+$, the function $v(t)+Ht^2/2$ is convex, since it is the maximum of the two convex functions $u(t)-a+Ht^2/2$ and $Ht^2/2$.
Since $u_h$ is $C^2$ and vanishes for $\|y\|\geq L/2$, both $u'$ and $v'$ vanish at the interval endpoints.
Hence the monotone one-sided derivative of $v(t)+Ht^2/2$ has total variation $HL$, which gives $\operatorname{Var}(v')\leq2HL$.
Here $\operatorname{Var}$ denotes total variation on that interval; derivatives at corners are understood almost everywhere, with their one-sided limits included in the variation. It follows that
\begin{equation}\label{app:supp:rho-variation}
 \operatorname{Var}((\rho\circ u)')\leq
 V:=L[H_\rho G^2+(1+6M_\rho)G_\rho H].
\end{equation}
This bound is uniform over all coordinate lines and all grids, which does not require finitely many crossings of an exceptional response value.

For an absolutely continuous one-dimensional function whose derivative has bounded variation, the composite midpoint error is at most $h^2\operatorname{Var}(v')/8$.
Indeed, integration against the distributional second derivative gives the error with a continuous piecewise quadratic Peano kernel, zero at cell endpoints and bounded by $h^2/8$. On a cell $[a,b]$ this kernel is $\tfrac12\min\{(t-a)^2,(b-t)^2\}$.
Applying this estimate successively in both coordinates and using \eqref{app:supp:rho-variation} bounds the unnormalized outer quadrature error by $LVh^2/4$.
Substituting the definitions of $L,G,H$ and enlarging the numerical coefficients, \eqref{app:supp:reg-split} proves the assertion with
\begin{equation}\label{app:supp:reg-constant}
\begin{split}
 C_{\mathrm{reg}}&=
 6(M_\rho+1)G_\rho(D_r+D_R)^2(H_kF_r+G_rG_k+F_kH_r)+4H_\rho F_k^2G_r^2(D_r+D_R)^4.
\end{split}
\end{equation}
All quantities defining this constant are independent of $h$. If $\rho(z)=|z|$, the same argument applies with $M_\rho=1$, $G_\rho=1$, and $H_\rho=0$. No value of $\rho'(0)$ is required.
\end{proof}
\par\endgroup
\section{Proof of Theorem~\ref{Error-dis}}\label{app:supp:theorem36-setting}
We use the sampling conventions and assumptions in Section~\ref{app:supp:discretization-framework}. We first estimate the semi-discrete error in $L^2$, then incorporate the regularization discretization error to control the discrete minimizers and their pixel-averaged squared Frobenius equivariance error.

	\subsection{The Analysis of Semi-Discrete Operator }\label{app:supp:semi-discrete-analysis}
	
	We begin with the convolution and data fidelity quadrature estimates needed to compare the continuous and discrete objectives.

\begingroup
\begin{lemma}\label{app:xie}
		For smooth functions $r:\R^2\to\R$ and $k: \R^2\to\R$,
		if for  $x\in \R^2$, the following conditions are satisfied:
		\begin{equation}
			\begin{split}
				|r(x)| \leq F_1,&~|k(x)|\leq F_2, \\
				\|\nabla r(x)\| \leq G_1,&~ \|\nabla k(x)\|\leq G_2, \\
				\|\nabla^2 r(x)\| \leq H_1,&~ \|\nabla^2 k(x) \|\leq H_2, \\
				\forall \|x\|\geq D_k,&~ k(x)=0
			\end{split}
		\end{equation}
		{where $ph=2D_k$ as in \eqref{fixed-windows}, } $\nabla$ and ${\nabla}^2$ denote the operators of gradient and Hessian matrix, respectively,
		then, $\forall y\in\R^2$, the following results are satisfied:
			\begin{equation}
				\begin{split}
					\left|\int_{\R^2}k\l(x\r)r\l(y-x\r)dx - 
					\sum_{{u,v}\in\Delta_p}k\l(t_{uv}\r)r\l(y-t_{uv}\r)h^2\right|\leq2CD_k^2h^2,
				\end{split}
			\end{equation}
		{where $\Delta_p=\{(u,v)\mid u,v=1,2,\ldots,p\}$, $t_{uv}=\big((u-(p+1)/2)h,(v-(p+1)/2)h\big)^T$, and $C=F_1H_2+F_2H_1+2G_1G_2$.}
	\end{lemma}
	
	Apply Lemma~\ref{tan} with $D=D_k$ to $k(x)r(y-x)$, whose second partial derivatives are bounded by $F_1H_2+F_2H_1+2G_1G_2$. This gives the stated estimate uniformly in $y$.
\par\endgroup

\begingroup
\begin{lemma}\label{app:discrete_error}
		For function $G(x)=\hat{r}(x)-\int_{\R^2}k(t)r(x-t)dt$,
		%	and $\tilde{G}(x)=\hat{r}(x)-\sum_{{u,v}\in\Delta_p}k(t_{uv})r(x-t_{uv})h^2$, 
		we define:
		\begin{equation}
			I = \int_{\R^2}\left(G(x)\right)^2 \!dx = \int_{\R^2}\left(\hat{r}(x)-\int_{\R^2}k(t)r(x-t)dt\right)^2 dx,
		\end{equation}
		and its discretization form:
		\begin{equation}
			\tilde{I} = \sum_{{i,j}\in\Delta}\!\left(\tilde{G}(x_{ij})\right)^2s^2h^2 = 
			\sum_{{i,j}\in\Delta}\left(\hat{r}(x_{ij})-\sum_{{u,v}\in\Delta_p}k(t_{uv})r(x_{ij}-t_{uv})h^2\right)^2 s^2h^2,
		\end{equation}		
where $\Delta$ and $\Delta_p$ are the observed-image and blur-kernel grids in Section~\ref{app:supp:discretization-grids}. Suppose
\[
\begin{aligned}
 &|\hat r|,|r|\leq F_1,\quad |k|\leq F_2,\qquad
 \|\nabla\hat r\|,\|\nabla r\|\leq G_1,\quad\|\nabla k\|\leq G_2,\\
 &\|\nabla^2\hat r\|,\|\nabla^2r\|\leq H_1,\quad\|\nabla^2k\|\leq H_2,\\
 &\hat r(x)=r(x)=0\quad(\|x\|\geq D_r),\qquad
 k(x)=0\quad(\|x\|\geq D_k).
\end{aligned}
\]
Then, with a constant $\bar C$ independent of $h$,
\begin{equation}\label{app:Ix}
 |I-\tilde I|\leq\bar C h^2.
\end{equation}
\end{lemma}
The support conditions make both residuals vanish for $\|x\|\geq D_r+D_k$.
Writing $C_0=F_1H_2+F_2H_1+2G_1G_2$, a valid constant is
\begin{equation}\label{app:supp:fidelity-constant}
\begin{aligned}
 \bar C=4(D_r+D_k)^2(1+4F_2D_k^2)\left[4C_0F_1D_k^2+(G_1^2+F_1H_1)(1+4F_2D_k^2)s^2\right].
\end{aligned}
\end{equation}
\par\endgroup

\begingroup
\begin{proof}
		{Let $\tilde G(x)=\hat r(x)-\sum_{(u,v)\in\Delta_p}k(t_{uv})r(x-t_{uv})h^2$. We first bound $G$, $\tilde G$, and the first two derivatives of $\tilde G$.}
		{Since $k(t)=0$ for $\|t\|\geq D_k$, both $\int_{\R^2}|k(t)|dt$ and $\sum_{(u,v)\in\Delta_p}|k(t_{uv})|h^2$ are bounded by $4F_2D_k^2$.}
		\begingroup
		\begin{equation}\label{app:bound_G}
			\begin{split}
				\left|G(x)\right|  = \left|\hat{r}(x)-\!\!\int_{\R^2}k(t)r(x-t)dt\right| 
				\leq \left|\hat{r}(x)\right| + \!\!\int_{\R^2}\left|k(t)\right|\left|r(x-t)\right|dt
				\leq F_1 \!+\! 4D_k^2F_1F_2.
			\end{split}
		\end{equation}
		\par
		\endgroup
		Similarly,
		\begingroup
		\begin{equation}\label{app:bound_G2}
			\begin{split}
				\left|\tilde{G}(x)\right| \!\!=\!\! \left|\hat{r}(x)\!-\!\!\!\!\sum_{{u,v}\in\Delta_p}\!k(t_{uv})r(x\!-\!t_{uv})h^2\right| \!\!
				\leq \!\left|\hat{r}(x)\right| \!\!+\!\!\! \sum_{{u,v}\in\Delta_p}\!\left|k(t_{uv})\right|\!\left|r(x\!-\!t_{uv})\right|h^2
				\!\!\leq \!F_1 \!\!+\! 4D_k^2F_1F_2.
			\end{split}
		\end{equation}
		\par
		\endgroup
		For the gradient and Hessian matrix of $\tilde{G}(x)$, we can deduce:
		\begingroup
		\begin{equation}\label{app:bound_Gg}
			\begin{split}
				\left|\nabla \tilde{G}(x)\right| & = \left|\nabla \hat{r}(x)-\sum_{{u,v}\in\Delta_p}k(t_{uv})\nabla r(x-t_{uv})h^2\right|\\
				& \leq  \left|\nabla \hat{r}(x)\right| + \sum_{{u,v}\in\Delta_p}\left|k(t_{uv})\right|\left|\nabla r(x-t_{uv})\right|h^2 \\
				& \leq G_1 + 4D_k^2F_2G_1. 
			\end{split}
		\end{equation}
		\par
		\endgroup
		\begingroup
		\begin{equation}\label{app:bound_Gg2}
			\begin{split}
				\left|\nabla^2 \tilde{G}(x)\right| & = \left|\nabla^2 \hat{r}(x)-\sum_{{u,v}\in\Delta_p}k(t_{uv})\nabla ^2 r(x-t_{uv})h^2\right|\\
				&\leq \left|\nabla^2 \hat{r}(x)\right| + \sum_{{u,v}\in\Delta_p}\left|k(t_{uv})\right|\left|\nabla^2 r(x-t_{uv})\right|h^2 \\
				&\leq H_1 + 4D_k^2F_2H_1.
			\end{split}
		\end{equation}
		\par
		\endgroup
		{We split $|I-\tilde I|$ into the convolution error and the outer quadrature error:}
		\begin{equation}\label{app:split}
			\begin{split}
				\left|I-\tilde{I}\right|
				\leq \underbrace{\left|\int_{\R^2}\left(G(x)\right)^2 dx - \int_{\R^2}\left(\tilde{G}(x)\right)^2 dx\right|}_{\langle 1 \rangle} 
				+ \underbrace{\left|\int_{\R^2}\left(\tilde{G}(x)\right)^2 dx - \sum_{{i,j}\in\Delta}\left(\tilde{G}(x_{ij})\right)^2s^2h^2 \right|}_{\langle 2 \rangle},
			\end{split}
		\end{equation}
		%	where $G(x)=\hat{r}(x)-\int_{\R^2}k(t)r(x-t)dt$, and $\tilde{G}(x)=\hat{r}(x)-\sum_{{u,v}\in\Delta_p}k(t_{uv})r(x-t_{uv})h^2$.
		
		For the part ⟨1⟩ of Eq. (\ref{app:split}), we have:
		\begingroup
		\begin{equation}
			\begin{split}
				\left|\int_{\R^2}\left(G(x)\right)^2 dx - \int_{\R^2}\left(\tilde{G}(x)\right)^2 dx\right|
				&= \left|\int_{\R^2}\left(G(x)-\tilde{G}(x)\right)\left(
				G(x)+\tilde{G}(x)\right) dx\right|\\
				& \leq \int_{\R^2}\left|\left(G(x)-\tilde{G}(x)\right)\left(
				G(x)+\tilde{G}(x)\right) \right|dx.
			\end{split}
		\end{equation}
		\par
		\endgroup
		By Lemma \ref{app:xie}, we know that:
			\begin{equation}
				\left|\int_{\R^2}k(t)r(x-t)dt - \sum_{{u,v}\in\Delta_p}k(t_{uv})r(x-t_{uv})h^2\right| \leq 2C_0D_k^2h^2,
		\end{equation}
		where $C_0 = F_1H_2+F_2H_1+2G_1G_2$. Then we have:
		\begingroup
		\begin{equation}\label{app:Gx}
			\begin{split}
				\left|G(x)-\tilde{G}(x)\right|
				\!=\! & \left|\hat{r}(x)\!-\!\!\int_{\R^2}k(t)r(x-t)dt-\hat{r}(x)\!+\!\!\!\sum_{{u,v}\in\Delta_p}k(t_{uv})r(x-t_{uv})h^2\right|
				\!\leq \!2C_0D_k^2h^2.
			\end{split}
		\end{equation}
		\par
		\endgroup
		{Combining \eqref{app:Gx} with the bounds \eqref{app:bound_G} and \eqref{app:bound_G2} gives:}
		\begingroup
		\begin{equation}\label{app:part1}
			\begin{split}
				&\int_{\R^2}\left|G(x)^2-\tilde{G}(x)^2\right|dx\\
				\leq{}&2C_0D_k^2h^2\,2F_1\big(1+4F_2D_k^2\big)
				\int_{\R^2}\delta(x)dx\\
				\leq{}&16C_0F_1D_k^2(1+4F_2D_k^2)(D_r+D_k)^2h^2,
			\end{split}
		\end{equation}
		\endgroup
		where
		$\delta(x)=\left\{\begin{array}{cc}
			1 & \mbox{if} ~ G(x)\neq 0 ~\mbox{or}~ \tilde{G}(x)\neq 0 \\
			0 & \mbox{otherwise}  \\
		\end{array}\right.$.
		{The vanishing conditions on $\hat{r}(x)$ and $k(t)r(x-t)$ imply that both $G$ and $\tilde{G}$ are zero for $\|x\|\geq D_r+D_k$. Consequently, $\int_{\R^2}\delta(x)dx\leq4(D_r+D_k)^2$.}

		{For part $\langle2\rangle$ of \eqref{app:split}, we bound the second derivatives of $\tilde G^2$ to apply Lemma~\ref{tan}:}
		\begin{equation}
			\begin{split}
				\frac{\partial (\tilde{G}(x))^2}{\partial x_1}=2\tilde{G}(x)\frac{\partial \tilde{G}(x)}{\partial x_1},~~~
				\frac{\partial^2 (\tilde{G}(x))^2}{\partial x_1^2}=2\left[\left(\frac{\partial \tilde{G}(x)}{\partial x_1}\right)^2+\tilde{G}(x)\frac{\partial^2 (\tilde{G}(x))}{\partial x_1^2}\right]. \\
				\frac{\partial (\tilde{G}(x))^2}{\partial x_2}=2\tilde{G}(x)\frac{\partial \tilde{G}(x)}{\partial x_2},~~~
				\frac{\partial^2 (\tilde{G}(x))^2}{\partial x_2^2}=2\left[\left(\frac{\partial \tilde{G}(x)}{\partial x_2}\right)^2+\tilde{G}(x)\frac{\partial^2 (\tilde{G}(x))}{\partial x_2^2}\right].
			\end{split}
		\end{equation}
		{Then by \eqref{app:bound_G2}, \eqref{app:bound_Gg}, and \eqref{app:bound_Gg2}, we have:}
		\begingroup
		\begin{equation}
			\begin{split}
				\max\left\{\left|\frac{\partial^2(\tilde{G}(x))^2}{\partial x_1^2}\right|,
				\left|\frac{\partial^2(\tilde{G}(x))^2}{\partial x_2^2}\right|\right\}
				\leq 2(G_1^2+F_1H_1)\big(1+4F_2D_k^2\big)^2.
			\end{split}
		\end{equation}
		\par
		\endgroup
		{According to Lemma 3.4, with $D=D_r+D_k$, $M=m$, and mesh size $sh$, we have:}
		\begingroup
		\begin{equation}\label{app:part2}
			\begin{split}
				&\left|\int_{\R^2}(\tilde{G}(x))^2dx
				-\sum_{(i,j)\in\Delta}(\tilde{G}(x_{ij}))^2s^2h^2\right|
				\leq 4(G_1^2+F_1H_1)(1+4F_2D_k^2)^2(D_r+D_k)^2s^2h^2.
			\end{split}
		\end{equation}
		Thus, combining \eqref{app:part1} and \eqref{app:part2}, we obtain
\[
 \begin{aligned}
 |I-\tilde I|\leq{}&16C_0F_1D_k^2(1+4F_2D_k^2)(D_r+D_k)^2h^2\\
 &+4(G_1^2+F_1H_1)(1+4F_2D_k^2)^2(D_r+D_k)^2s^2h^2
 =\bar C h^2.
 \end{aligned}
\]
\endgroup
		where $\bar{C}$ is the constant specified in the lemma. The proof is completed.
	\end{proof}
\par\endgroup
\begin{lemma}\label{app:min_k}
		Let $g$ and $\hat g$ attain their minima on the same feasible set. If for every $k$ in this set,
		\begin{equation}
			\left|g(k)-\hat{g}(k)\right|\leq C,
		\end{equation}
		then it holds that:
		\begin{equation}
			\left|\min_{k}g(k)-\min_{k}\hat{g}(k)\right|\leq C.
		\end{equation}
	\end{lemma}
	
	\begin{proof}
		Without loss of generality, assume $\min_k g(k)\geq\min_k\hat g(k)$, and let
		\begin{equation}
			k^*\!\!=\!\!\arg \min_{k}g(k), ~\hat{k}^*\!=\!\arg \min_{k}\hat{g}(k).
		\end{equation}
		Then we have:
		\begin{equation}
			\begin{split}
				\left|\min_{k}g(k)\!-\!\min_{k}\hat{g}(k)\right|\!\!=\!\min_{k}g(k)\!-\!\min_{k}\hat{g}(k)
				\!=\! g(k^*)\!-\!\hat{g}(\hat{k}^*)
				\!\leq g(\hat{k}^*)\!-\!\hat{g}(\hat{k}^*)
				\!\leq \!\!\left|g(\hat{k}^*)-\hat{g}(\hat{k}^*)\right|
				\!\!\leq \!C.
			\end{split}
		\end{equation}
	The proof is completed.
	\end{proof}
	
	The continuous, semi-discrete, and discrete objectives are defined in Section~\ref{app:supp:discretization-objectives}. Under the representation and stability assumptions in Section~\ref{app:supp:discretization-framework}, we now compare the continuous and semi-discrete minimizers.

\begingroup
\begin{lemma}\label{app:Error-temp}
		{For the dataset $\{\hat r_n\}_{n=1}^N$, define the \textit{semi-discrete operator} $\bar\Phi$ by discretizing only the data fidelity term:}
		\begingroup
		\begin{equation}\label{app:smi-dis}
			\begin{aligned}
				\{\bar{\Phi}[\hat{r}_n]\}_{n=1}^N
				\!=\!\arg\!\!\min_{\{r_n\}\in\Omega_\varepsilon}\!\sum_{n=1}^N
				\Bigg[\min_k\sum_{(i,j)\in\Delta}\!\!
				\left(\hat r_n(x_{ij})\!-\!\!\!\!\!
				\sum_{(u,v)\in\Delta_p}\!k(t_{uv})r_n(x_{ij}\!-\!t_{uv})h^2\right)^2\!\!(sh)^2
				\!+\!R[r_n]\Bigg]\!.
			\end{aligned}
		\end{equation}
		\endgroup
		{where $\Delta,\Delta_p,s,h$ are specified above, {$R$ is the continuous regularizer in \eqref{R-conv}}, and $\Omega_\varepsilon$ is the relaxed symmetric feasible set. Suppose the assumptions of Theorem~\ref{Error-dis} hold. Assume additionally that the image support and boundedness conditions in \eqref{bound_condition} hold for $r=\bar\Phi[\pi_{A_n}[\hat r_n]]$, with the same constants, for every compatible mesh and every choice of $A_n\in G_n$, $n=1,\ldots,N$. Then, for any $A_n\in G_n$, $n=1,\ldots,N$, with $|\det(A_n)|=1$, we have:}
		\begingroup
		\begin{equation}\label{app:error_bar}
			\frac{1}{N}\sum_{n=1}^N \|\bar{\Phi}[\pi_{A_n}[\hat{r}_n]] - \pi_{A_n}[\bar{\Phi}[\hat{r}_n]]\|_2^2 \leq {\tilde{C}h^2 + \frac{6}{\mu}\varepsilon},
		\end{equation}
		{where $\tilde C=12\bar C/\mu$ and $\bar C$ is the constant in Lemma~\ref{app:discrete_error}.}
		\endgroup
	\end{lemma}
	
	\begin{proof}
		{Let $F$ be the continuous objective functional in \eqref{F_conti}. Under Assumption~\ref{ass:qg}, the local quadratic growth condition holds uniformly for both the original and transformed datasets. The corresponding semi-discrete functional is:}
		
		\begingroup
\begingroup
		\begin{equation}\label{app:F_temp}
			\tilde F[r_n]=\min_k\sum_{(i,j)\in\Delta}
			\left(\hat r_n(x_{ij})-\sum_{(u,v)\in\Delta_p}k(t_{uv})r_n(x_{ij}-t_{uv})h^2\right)^2(sh)^2
			+R[r_n].
		\end{equation}
\endgroup
		\par
		\endgroup
		
		{The regularization terms in $F$ and $\tilde F$ coincide, so their difference consists only of the data fidelity discretization error.}
		
		{For any $r_n$ satisfying the stated image bounds, Lemma~\ref{app:discrete_error} gives the following data fidelity estimate for every feasible kernel $k$:}
		\begingroup
		\begin{equation}
			\begin{split}
				\left|\!\int_{\R^2}\!\!\left(\!\hat{r}_n(x)\!-\!\!\!\!\int_{\R^2}\!\!\!k(t)r_n(x\!-\!t)dt\!\right)^{2}\!\!\!dx\!\!-\!\!\!\!\sum_{{i,j}\in\Delta}\!\!\!\left(\!\!\hat{r}_n(x_{ij})\!-\!\!\!\!\sum_{{u,v}\in\Delta_p}\!\!k(t_{uv})r_n(x_{ij}\!-\!t_{uv})h^2\!\!\right)^2 \!\!\!s^2\!h^2\!\right|
				\!\!\leq \!\!\bar{C}h^2\!.
			\end{split}
		\end{equation}
		\endgroup
		{where $\bar C$ is the constant in Lemma~\ref{app:discrete_error}.}
		
		{Applying Lemma~\ref{app:min_k} to the minimization over $k$, we obtain:}
		\begingroup
		\begin{equation}\label{app:mink}
			\begin{split}
				&\left|\min_{k}\int_{\R^2}\!\!\left(\hat{r}_n(x)\!-\!\!\!\int_{\R^2}\!\!\!k(t)r_n(x\!-\!t)dt\!\right)^{2}\!\!\!dx\!\!-\!\!\min_{k}\!\!\sum_{{i,j}\in\Delta}\!\!\left(\!\!\hat{r}_n(x_{ij})\!-\!\!\!\!\sum_{{u,v}\in\Delta_p}k(t_{uv})r_n(x_{ij}\!\!-\!\!t_{uv})h^2\!\!\right)^2 \!\!\!\!s^2h^2\right| \\
				\leq& \bar{C}h^2.
			\end{split}
		\end{equation}
		\par
		\endgroup
		{Since the regularizers coincide, \eqref{app:mink} gives:}
		\begingroup
		\begin{equation}\label{app:error_F}
			\left|F[r_n]-\tilde{F}[r_n]\right|
			{\leq\bar Ch^2=\frac{\mu\tilde C}{12}h^2.}
		\end{equation}
		\par
		\endgroup
		{By \eqref{Modelall} and \eqref{app:smi-dis}, $\Phi$ and $\bar\Phi$ minimize $\sum_n F[r_n]$ and $\sum_n\tilde F[r_n]$, respectively, that is:
\begin{equation}
  \left[\Phi(\hat{r}_n) \right]=\arg \min_{[r_n]\in \Omega_{\varepsilon}}\sum_{n=1}^N F[r_n],~~~
  \left[\bar{\Phi}(\hat{r}_n) \right]=\arg \min_{[r_n]\in \Omega_{\varepsilon}}\sum_{n=1}^N \tilde{F}[r_n].
\end{equation}
Under Assumption~\ref{ass:qg}, $\{\bar\Phi[\hat r_n]\}_{n=1}^N$ lies in the quadratic growth neighborhood of $\{\Phi[\hat r_n]\}_{n=1}^N$. Hence,}
		\begingroup
		\begin{equation}
			\begin{split}
				&\sum_{n=1}^N\|\Phi[\hat{r}_n]-\bar{\Phi}[\hat{r}_n]\|_2^2\\
				\leq{}&\frac{1}{\mu}\sum_{n=1}^N
				\Big(F[\bar{\Phi}[\hat{r}_n]]-F[\Phi[\hat{r}_n]]\Big)\\
				={}&\frac{1}{\mu}\sum_{n=1}^N
				\Big(F[\bar{\Phi}[\hat{r}_n]]-\tilde{F}[\bar{\Phi}[\hat{r}_n]]
				+\tilde{F}[\bar{\Phi}[\hat{r}_n]]-\tilde{F}[\Phi[\hat{r}_n]]
				+\tilde{F}[\Phi[\hat{r}_n]]-F[\Phi[\hat{r}_n]]\Big)\\
				\leq{}&\frac{1}{\mu}\sum_{n=1}^N
				\Big(\left|F[\bar{\Phi}[\hat{r}_n]]-\tilde{F}[\bar{\Phi}[\hat{r}_n]]\right|
				+\left|\tilde{F}[\Phi[\hat{r}_n]]-F[\Phi[\hat{r}_n]]\right|\Big).
			\end{split}
		\end{equation}
		\endgroup
		{Here we use the optimality inequality $\sum_{n=1}^N\tilde{F}[\bar{\Phi}[\hat{r}_n]]\leq\sum_{n=1}^N\tilde{F}[\Phi[\hat{r}_n]]$.}
		The image bounds hold for both $r_n=\Phi[\hat r_n]$ and $r_n=\bar\Phi[\hat r_n]$. Substituting \eqref{app:error_F} at these minimizers, we obtain:
		\begingroup
		\begin{equation}\label{app:error_phi}
			\begin{split}
				\sum_{n = 1}^N\|\Phi[\hat{r}_n]-\bar{\Phi}[\hat{r}_n]\|_2^2
				\leq {\frac{N\tilde C}{6}h^2.}
			\end{split}
		\end{equation}
		\par
		\endgroup
		
		\textcolor{black}{To bound the equivariance error, we combine this approximation estimate with Theorem~\ref{Error-IR}. The Cauchy--Schwarz inequality gives:
		\begin{equation}
			\begin{split}
				\frac{1}{N}\sum_{n=1}^N\|\bar{\Phi}[\pi_{A_n}[\hat{r}_n]]-\pi_{A_n}[\bar{\Phi}[\hat{r}_n]]\|_2^2 
				\leq & \frac{{3}}{N}\sum_{n=1}^N\|\bar{\Phi}[\pi_{A_n}[\hat{r}_n]]-\Phi[\pi_{A_n}[\hat{r}_n]]\|_2^2\\
				&+\frac{{3}}{N}\sum_{n=1}^N\|\Phi[\pi_{A_n}[\hat{r}_n]]-\pi_{A_n}[\Phi[\hat{r}_n]]\|_2^2\\
				&+\frac{{3}}{N}\sum_{n=1}^N\|\pi_{A_n}[\Phi[\hat{r}_n]]-\pi_{A_n}[\bar{\Phi}[\hat{r}_n]]\|_2^2.
			\end{split}
		\end{equation}
		For the transformed inputs, Assumption~\ref{ass:qg} also places the semi-discrete minimizers in the neighborhoods of their corresponding continuous minimizers. Applying \eqref{app:error_phi} with the same bounds, we have:}
		\begingroup
		\begin{equation}\label{app:step1}
			\frac{1}{N}\sum_{n=1}^N\|\Phi[\pi_{A_n}[\hat{r}_n]]\!\!-\!\bar{\Phi}[\pi_{A_n}[\hat{r}_n]]\|_2^2 \leq {\frac{\tilde C}{6}h^2.}
		\end{equation}
		\par
		\endgroup
		{The continuous equivariance bound in Theorem~\ref{Error-IR} gives:}
		\begin{equation}\label{app:step2}
			\frac{1}{N}\sum_{n=1}^N\|\Phi[\pi_{A_n}[\hat{r}_n]]-\pi_{A_n}[\Phi[\hat{r}_n]]\|_2^2 \leq \frac{2\varepsilon}{\mu}.
		\end{equation}
		From \eqref{app:error_phi} and the Norm-Preserving property in Lemma~\ref{norm-preserve}, we obtain:
		\begingroup
		\begin{equation}\label{app:step3}
			\frac{1}{N}\sum_{n=1}^N\|\pi_{A_n}[\Phi[\hat{r}_n]]-\pi_{A_n}[\bar{\Phi}[\hat{r}_n]]\|_2^2 \!\leq {\frac{\tilde C}{6}h^2.}
		\end{equation}
		\par
		\endgroup
		Finally, by combining \eqref{app:step1}, \eqref{app:step2}, and \eqref{app:step3}, we can get:
		\begingroup
		\begin{equation}
			\begin{split}
				\frac{1}{N}\sum_{n=1}^N\|\bar{\Phi}[\pi_{A_n}[\hat{r}_n]]-\pi_{A_n}[\bar{\Phi}[\hat{r}_n]]\|_2^2
				\leq {\tilde C h^2}
				+\frac{6\varepsilon}{\mu}
				= {\tilde C h^2+\frac{6}{\mu}\varepsilon},
			\end{split}
		\end{equation}
		\endgroup
		This proves \eqref{app:error_bar}.
	\end{proof}
\par\endgroup
\subsection{Proof of Theorem~\ref{Error-dis} in the main text}\label{app:supp:theorem36-proof}

\begin{proof}

		Under the sampling conditions in Section~\ref{app:supp:discretization-framework}, every feasible discrete collection has an admissible continuous representation. For either the continuous minimizing collection or the selected admissible representations of the discrete minimizing collection, let $\{r_n\}$ denote the continuous functions and $\{Y_n\}$ their samples. These functions satisfy the image bounds of Theorem~\ref{Error-dis}, taking $A_n=I$ for the original dataset. The fidelity terms in $\sum_n\tilde F[r_n]$ and $(sh)^2J_h(\{Y_n\})$ agree. Lemma~\ref{error-R} therefore gives
		\begin{equation}\label{app:scaled-R-error}
			\begin{aligned}
				\left|\sum_{n=1}^N\tilde F[r_n]-(sh)^2J_h(\{Y_n\})\right|
				&=\left|\sum_{n=1}^N\big(R[r_n]-(sh)^2\tilde R(Y_n)\big)\right|\\
				&\leq NC_{\mathrm{reg}}h^2.
			\end{aligned}
		\end{equation}
		Applying Lemmas~\ref{app:discrete_error} and~\ref{app:min_k} to each such collection gives $|F[r_n]-\tilde F[r_n]|\leq\bar Ch^2$. Combining this with \eqref{app:scaled-R-error} yields, for these collections,
		\begin{equation}\label{app:discrete-objective-error}
			\left|\sum_{n=1}^NF[r_n]-(sh)^2J_h(\{Y_n\})\right|
			\leq N\big[\bar C+C_{\mathrm{reg}}\big]h^2.
		\end{equation}
		Now let $Y_n=\tilde\Phi(\hat Y_n)$ and choose the admissible representations $\{r_n\}\in\Omega_\varepsilon\cap\mathcal N$ guaranteed by Assumption~\ref{ass:qg}, so that
		\begin{equation}\label{app:sampled-minimizers}
			r_n(x_{ij})=\big(\tilde\Phi(\hat Y_n)\big)_{ij},\qquad (i,j)\in\Delta_r.
		\end{equation}
		The samples of $\{\Phi[\hat r_n]\}$ form a feasible discrete collection. Comparing the discrete optimum with that collection and applying \eqref{app:discrete-objective-error} to both collections gives
		\begin{equation}\label{app:discrete-optimality-gap}
			\sum_{n=1}^N\big(F[r_n]-F[\Phi[\hat r_n]]\big)
			\leq 2N\big[\bar C+C_{\mathrm{reg}}\big]h^2.
		\end{equation}
		The quadratic growth condition thus implies
		\begin{equation}\label{app:discrete-minimizer-error}
			\frac1N\sum_{n=1}^N\|r_n-\Phi[\hat r_n]\|_2^2
			\leq\frac{2}{\mu}\big[\bar C+C_{\mathrm{reg}}\big]h^2.
		\end{equation}
		For the transformed inputs, let $r_n^A$ denote the selected admissible continuous representation of $\tilde\Phi(\tilde\pi_{A_n}(\hat Y_n))$. The same estimate holds relative to $\Phi[\pi_{A_n}[\hat r_n]]$, with the same constants. Define
		\begin{equation}
			E_n=r_n^A-\pi_{A_n}[r_n].
		\end{equation}
		By \eqref{app:sampled-minimizers} and the transformation identity in \eqref{app:sampling-identities}, $E_n(x_{ij})$ equals the difference between the two discrete restoration outputs at pixel $(i,j)$.

		{By Theorem~\ref{Error-dis}, both $r_n^A$ and $\pi_{A_n}[r_n]$ satisfy the stated image support, value, and derivative bounds. Hence the two terms defining $E_n$ satisfy the same bounds:}
		\begingroup
\begin{equation}\label{app:boundE1}
			\left|r_n^A-\pi_{A_n}[r_n]\right| \leq \left|r_n^A\right| + \left|\pi_{A_n}[r_n]\right| \leq 2F_1.
		\end{equation}
\endgroup
		{For the gradient of $E_n$, we have:}
		\begingroup \color{black}
\begin{equation}
			\nabla E_n(x) = \nabla r_n^A(x) - \nabla \pi_{A_n}[r_n](x) = \nabla r_n^A(x) - (A_n^{-1})^T\nabla r_n(A_n^{-1}x).
		\end{equation}
		By the triangle inequality and the common gradient bound,
\begin{equation}\label{app:boundE2}
			{\left\|\nabla E_n(x)\right\| \leq \left\|\nabla r_n^A(x)\right\| + \left\|\nabla\!\left(\pi_{A_n}[r_n]\right)(x)\right\| \leq 2G_1.}
		\end{equation}
		Similarly, for $\left\|\nabla^2 E_n\right\|$, we have:
\begin{equation}\label{app:boundE3}
			{\left\|\nabla^2 E_n(x)\right\| \leq \left\|\nabla^2 r_n^A(x)\right\| + \left\|\nabla^2\!\left(\pi_{A_n}[r_n]\right)(x)\right\| \leq 2H_1.}
		\end{equation}
\endgroup
		{To apply Lemma~\ref{tan}, we now estimate the second derivatives of $E_n^2$:}
		\begin{equation}
			\begin{split}
				\frac{\partial(E_n)^2}{\partial x_1}=2E_n\frac{\partial E_n}{\partial x_1},~~~
				\frac{\partial^2 (E_n)^2}{\partial x_1^2}=2\left[\left(\frac{\partial E_n}{\partial x_1}\right)^2+E_n\frac{\partial^2 (E_n)}{\partial x_1^2}\right]. \\
				\frac{\partial (E_n)^2}{\partial x_2}=2E_n\frac{\partial E_n}{\partial x_2},~~~
				\frac{\partial^2 (E_n)^2}{\partial x_2^2}=2\left[\left(\frac{\partial E_n}{\partial x_2}\right)^2+E_n\frac{\partial^2 (E_n)}{\partial x_2^2}\right].
			\end{split}
		\end{equation}
		{Combining \eqref{app:boundE1}, \eqref{app:boundE2}, and \eqref{app:boundE3}, we obtain:}
		\begin{equation}\label{app:E^2}
			\max \left\{\left|\frac{\partial^2\left(E_n\right)^2}{\partial x_1^2} \right|, \left|\frac{\partial^2\left(E_n\right)^2}{\partial x_2^2} \right|\right\}
			\leq 2\left(4G_1^2+4F_1H_1\right)=8G_1^2+8F_1H_1.
		\end{equation}
		{Since $E_n(x)=0$ for $\|x\|\geq D_r$, the range condition of Lemma~\ref{tan} holds with $D=D_r+D_k$ and $M=m_r$ on the fixed restored-image window. }Given the boundedness condition in \eqref{app:E^2}, \textbf{Lemma 3.4 (Discretization Error)} implies:
		\begingroup
		\begin{equation}
			\left|\int_{\R^2}\left(E_n(x)\right)^2 dx - \sum_{{i,j}\in\Delta_r}\left(E_n(x_{ij})\right)^2h^2\right|\leq 16(G_1^2+F_1H_1)(D_r+D_k)^2h^2,
		\end{equation}
		\endgroup
		{where $\Delta_r=\{(i,j)\mid i,j=1,2,\ldots,m_r\}$ is the restored-image index set at mesh size $h$ and $x_{ij}=\big((i-(m_r+1)/2)h,(j-(m_r+1)/2)h\big)^T$ are its grid coordinates.}
		
		\begingroup
		Averaging the quadrature estimate over the dataset gives
		\begin{equation}\label{app:bound-norm}
			\left|\frac1N\sum_{n=1}^N\|E_n\|_2^2
			-\frac{h^2}{N}\sum_{n=1}^N\sum_{(i,j)\in\Delta_r}|E_n(x_{ij})|^2\right|
			\leq 16(G_1^2+F_1H_1)(D_r+D_k)^2h^2.
		\end{equation}
		By \eqref{app:sampled-minimizers} and the transformation identity in \eqref{app:sampling-identities}, the ordinary squared Frobenius norm is
		\begin{equation}\label{app:ordinary-Frobenius}
			\left\|\tilde\Phi\big(\tilde\pi_{A_n}(\hat Y_n)\big)
			-\tilde\pi_{A_n}\big(\tilde\Phi(\hat Y_n)\big)\right\|_F^2
			=\sum_{(i,j)\in\Delta_r}|E_n(x_{ij})|^2.
		\end{equation}
		Thus \eqref{app:bound-norm} yields
		
\begin{equation}\label{app:error-norm}
\begin{aligned}
 \frac1{Nm_r^2}\sum_{n=1}^N
 \left\|\tilde\Phi\big(\tilde\pi_{A_n}(\hat Y_n)\big)
 -\tilde\pi_{A_n}\big(\tilde\Phi(\hat Y_n)\big)\right\|_F^2
 \leq\frac1{4(D_r+D_k)^2N}\sum_{n=1}^N\|E_n\|_2^2
       +4(G_1^2+F_1H_1)h^2.
\end{aligned}
\end{equation}

\begingroup
		To bound the continuous term in \eqref{app:error-norm}, combine \eqref{app:discrete-minimizer-error} for the original and transformed inputs with Theorem~\ref{Error-IR} and the norm-preserving property:
		\begin{equation}\label{app:error-L2}
			\begin{aligned}
				\frac1N\sum_{n=1}^N\|E_n\|_2^2
				&\leq\frac3N\sum_{n=1}^N\Big(\|r_n^A-\Phi[\pi_{A_n}[\hat r_n]]\|_2^2
				+\|\Phi[\pi_{A_n}[\hat r_n]]-\pi_{A_n}[\Phi[\hat r_n]]\|_2^2\\
				&\hspace{37mm}+\|\pi_{A_n}[\Phi[\hat r_n]]-\pi_{A_n}[r_n]\|_2^2\Big)\\
				&\leq\frac{12(\bar C+C_{\mathrm{reg}})}{\mu}h^2+\frac{6}{\mu}\varepsilon.
			\end{aligned}
		\end{equation}
\endgroup
		
Since $m_r^2h^2=4(D_r+D_k)^2$ is fixed, substitute \eqref{app:error-L2} directly into \eqref{app:error-norm}. To express the final constants using the stated bounds, use $2F_1G_1\leq F_1^2+G_1^2$ in \eqref{app:supp:fidelity-constant} to obtain
\[
\begin{aligned}
 \bar C\leq{}4(D_r+D_k)^2(1+s^2)
 \bigl(1+4D_k^2(F_2+G_2+H_2)\bigr)^2
 \cdot(F_1^2+G_1^2+F_1H_1).
\end{aligned}
\]
Together with $G_1^2+F_1H_1\leq F_1^2+G_1^2+F_1H_1$, this yields
\begin{equation}\label{app:supp:theorem36-explicit}
\begin{aligned}
 &\frac1{Nm_r^2}\sum_{n=1}^N
 \left\|\tilde\Phi\big(\tilde\pi_{A_n}(\hat Y_n)\big)
 -\tilde\pi_{A_n}\big(\tilde\Phi(\hat Y_n)\big)\right\|_F^2
 \leq Ch^2+\hat C\varepsilon,\\
 &C=\left[4+\frac{12(1+s^2)}{\mu}\bigl(1+4D_k^2(F_2+G_2+H_2)\bigr)^2\right](F_1^2+G_1^2+F_1H_1)\\
 &\qquad+\frac{3C_{\mathrm{reg}}}{\mu(D_r+D_k)^2},\qquad
 \hat C=\frac{3}{2\mu(D_r+D_k)^2}.
\end{aligned}
\end{equation}
$C_{\mathrm{reg}}$ is given in \eqref{app:supp:reg-constant} with $(F_r,G_r,H_r)=(F_1,G_1,H_1)$ and $(F_k,G_k,H_k)=(F_3,G_3,H_3)$. These constants depend only on the fixed physical support radii, $s$, the stated uniform bounds, the fixed response function, and the common positive quadratic-growth constant $\mu$. In particular, they do not grow with the grid sizes. This completes the proof.
\endgroup
\end{proof}
\begingroup\color{black}
\section{Proof of Lemma~\ref{bar_f}}
	Before proving Lemma 4.1, we first present the following lemma and its proof.
	
	\begin{lemma} \label{app:pi_gg}
	For a prescribed transformation group $G_n$, $\pi_{A_n}$ represents how $A_n\in G_n$ acts on an image function $r(x)$, i.e., $\pi_{A_n}[r](x)=r(A_n^{-1}x)$ for $x\in\mathbb{R}^2$. Then, for any $A_n,\bar{A}_n\in G_n$, we have $\pi_{A_n}[\pi_{\bar{A}_n}[r]]=\pi_{A_n\bar{A}_n}[r]$.
	\end{lemma}
	\begin{proof}
		We can deduce that:
		\begin{equation}
			\begin{split}
				\pi_{A_n}[\pi_{\bar{A}_n}[r]](x)
				=r(\bar{A}_n^{-1}A_n^{-1}x)
				=r((A_n\bar{A}_n)^{-1}x)=\pi_{A_n\bar{A}_n}[r](x).
			\end{split}
		\end{equation}
	\end{proof}
	
	Based on Lemma \ref{app:pi_gg} above, we now present the proof of Lemma 4.1 in the main text. For the discrete actions, we use the corresponding composition and inverse identities in the calculation below.
	\begin{proof}
		For each $n$, let $G_n=\mathcal H_G(\hat Y_n)$ and $\bar A_n\in G_n$. Assumption~\ref{ass:adaptive_group} gives $\mathcal H_G(\tilde\pi_{\bar A_n}(\hat Y_n))=\mathcal H_G(\hat Y_n)=G_n$. Thus, both evaluations of $\bar\Psi$ use the same group, and we have:
	
		\begin{equation}
			\begin{split}
				\bar{\Psi}(\tilde{\pi}_{\bar{A}_n}(\hat Y_n))
				={}&\sum_{A_n\in G_n}\frac{1}{|G_n|}\tilde{\pi}_{A_n}
				\bigl(\Psi(\tilde{\pi}_{A_n}^{-1}(\tilde{\pi}_{\bar{A}_n}(\hat Y_n)))\bigr)\\
				={}&\sum_{A_n\in G_n}\frac{1}{|G_n|}\tilde{\pi}_{\bar{A}_n}\Bigl(
				\tilde{\pi}_{\bar{A}_n^{-1}A_n}
				\bigl(\Psi(\tilde{\pi}_{A_n^{-1}\bar{A}_n}(\hat Y_n))\bigr)\Bigr)\\
				={}&\tilde{\pi}_{\bar{A}_n}\Bigl(\sum_{A_n\in G_n}\frac{1}{|G_n|}
				\tilde{\pi}_{\bar{A}_n^{-1}A_n}
				\bigl(\Psi(\tilde{\pi}_{\bar{A}_n^{-1}A_n}^{-1}(\hat Y_n))\bigr)\Bigr)\\
				={}&\tilde{\pi}_{\bar{A}_n}(\bar{\Psi}(\hat Y_n)).
			\end{split}
		\end{equation}
	The third equality uses the linearity of the output action, and the last equality follows because $\bar{A}_n^{-1}A_n$ ranges over $G_n$ as $A_n$ does.
	The proof is completed.
	\end{proof}

\endgroup	
	
\begingroup
\color{black}
\section{Generalization Analysis for Equivariant Networks}\label{app:supp:fixed-group-generalization}

This section details the generalization properties of networks where sample-dependent group parameters are governed by a fixed ideal hypernetwork. We first outline the precise technical conditions of Lemma~\ref{eq-ge}, rigorously establishing the generalization bound for i.i.d.\ image pairs and validating the shared-parameter formulation. Finally, we elucidate how this bound combines with the empirical risk analysis to yield expected risk bounds.

\subsection{Network Classes and Conditions}\label{app:supp:gen-assumptions}
The following conditions define the classes in Lemma~\ref{eq-ge}. The architecture, parameterization rules, and associated bounds are strictly specified before observing the training sample.
\begin{enumerate}[label=(A\arabic*),leftmargin=*,itemsep=4pt]
	
    \item\label{app:supp:gen-a1}\textit{Image-pair sampling and fixed hypernetwork.}
The training pairs $D_N=\{(\hat X_n,\hat Y_n)\}_{n=1}^N$
are sampled i.i.d.\ from a distribution $\mathcal P$ on
$[0,1]^{m\times m}\times[0,1]^{m\times m}$, where $N$ counts independent image pairs, $m=m_r$ and $d=m^2=m_r^2$ in this same-size statistical setting. The result is stated at each fixed resolution.
The measurable ideal hypernetwork
$\mathcal H:[0,1]^{m\times m}\to\mathcal W$
and the nonempty measurable set $\mathcal W\subset\mathbb R^3$
are specified before sampling.
The hypernetwork remains fixed during training and depends only
on the current input image, independently of the restoration coefficients.
Its output specifies the group
$G_{\mathcal H(Y)}=\mathcal H_G(Y)$,
where $\mathcal H_G$ satisfies
Assumption~\ref{ass:adaptive_group}.

	\item\label{app:supp:gen-a2}\textit{Restoration architecture and group order.}
	Both backbone networks consist of $L\geq2$ convolutional layers, sharing identical odd kernel sizes $p_\ell\times p_\ell$ and channel counts $c_\ell$, with $c_0=c_L=1$. The convolutions employ a stride of one and size-preserving zero padding. The hidden activations utilize ReLU, and the final outputs are clipped to the range $[0,1]$. A common finite group order $T$ divides each hidden channel count $c_\ell$. In alignment with the construction in Section~\ref{sec:network-design}, we define $G_w=D_wSD_w^{-1}$, where $D_w$ is invertible for every $w\in\mathcal W$ and $|S|=T$. It should be noted that depth and channel comparisons apply strictly to the restoration backbones and exclude the additional fixed hypernetwork.
	
	\item\label{app:supp:gen-a3}\textit{Shared linear parameterization.}
	For layer $\ell$, the expanded kernel and bias are:
	\begin{equation}\label{app:supp:gen-parameterization}
		\omega_\ell(w;\zeta_\ell):=
		\begin{pmatrix}\operatorname{vec}W_\ell(w;\zeta_\ell)\\ b_\ell(w;\zeta_\ell)\end{pmatrix}
		=M_{\ell,T}(w)\zeta_\ell,
		\qquad \zeta_\ell\in\mathbb R^{q_{\ell,T}}.
	\end{equation}
	The measurable matrix-valued maps are defined before sampling and are completely independent of the learned coefficients. The same coefficient vectors and coordinates apply universally to every $w$. We assume common constants $0<a_{\rm par}\leq1\leq b_{\rm par}<\infty$ that satisfy:
	\begin{equation}\label{app:supp:gen-matrix-bounds}
		a_{\rm par}^2I\preceq M_{\ell,T}(w)^\top M_{\ell,T}(w)
		\preceq b_{\rm par}^2I,
		\qquad w\in\mathcal W,\quad \ell=1,\ldots,L.
	\end{equation}
	These constants hold uniformly across all allowed $w$, network layers, and compared admissible values of $T$, rather than being restricted solely to the observed descriptors.
	
	\item\label{app:supp:gen-a4}\textit{Coefficient dimension and norm budget.}
	Write $q_T=\sum_{\ell=1}^Lq_{\ell,T}$. The ordinary CNN contains $P_{\rm sh}+P_0$ effective parameters, where
	\begin{equation}\label{app:supp:gen-parameter-count}
		\begin{aligned}
			P_{\rm sh}&=\sum_{\ell=1}^Lp_\ell^2c_{\ell-1}c_\ell
			+\sum_{\ell=1}^{L-1}c_\ell,\\
			P_0&=c_L=1,\qquad q_T\leq P_{\rm sh}/T+P_0,
		\end{aligned}
	\end{equation}
	where $P_{\rm sh}$ quantifies the shared kernel weights and hidden biases, $P_0$ accounts for the output bias. Given a fixed budget $\rho_{\rm par}>0$, the structured coefficients are constrained to the set
	\begin{equation}\label{app:supp:gen-feasible-set}
		\mathcal V_T^{\rho_{\rm par}}=
		\left\{\zeta:\sup_{w\in\mathcal W}\sum_{\ell=1}^L
		\|M_{\ell,T}(w)\zeta_\ell\|_2^2\leq\rho_{\rm par}^2\right\}.
	\end{equation}
	Accordingly, ordinary CNN parameters satisfy $\sum_\ell(\|W_\ell\|_F^2+\|b_\ell\|_2^2)\leq\rho_{\rm par}^2$. All these structural constraints are established prior to sampling.
	
	\item\label{app:supp:gen-a5}\textit{Biases, BN, and evaluated predictors.}
	If batch normalization (BN) is applied, the aforementioned representation and norm constraints take effect after folding its frozen inference-mode statistics and affine parameters into the convolutional weights. Within the structured class, these folded quantities are shared across orientations and remain independent of both the current input and the descriptor. This folding process must yield a single shared coefficient vector that is valid for all $w$. All risk evaluations are performed on this specific inference-mode predictor. Other than the fixed $\mathcal H$, no supplementary modules are incorporated.
\end{enumerate}

Let $\Psi_\zeta(Y,w)$ denote the conditional restoration backbone and define the class of complete image-input predictors by
\begin{equation}\label{app:supp:gen-composite-class}
	\begin{aligned}
		\Psi_\zeta^{\mathcal H}(Y)&=\Psi_\zeta(Y,\mathcal H(Y)),\\
		\Lambda_{EQ}^{\rho_{\rm par},\mathcal H}
		&=\{\Psi_\zeta^{\mathcal H}:\zeta\in\mathcal V_T^{\rho_{\rm par}}\}.
	\end{aligned}
\end{equation}
Let $\Lambda_{\rm CNN}^{\rho_{\rm par}}$ represent the ordinary CNN class specified in (A2), (A4), and (A5). Both network classes accept solely an image as input, and we do not assume any subset inclusion between them. The statistical class is rigorously defined prior to sampling, rather than being determined solely by the training-sample relations in \eqref{eqNet}. 
When bridging this bound with the empirical-risk analysis, $\Lambda_{EQ}$ and $\Lambda$ denote these prescribed classes, subject to the equivariance and candidate conditions outlined below.

For $N\geq1$ and $0<\delta<1$, the terms evaluated in Lemma~\ref{eq-ge} are given by

\begingroup
\begin{equation}\label{app:supp:gen-main-terms}
	\begin{aligned}
		U_N&=\sqrt{\frac{P_{\rm sh}\log(1+\kappa_{\rm gen}N)}{2N}},\\
		V_{N,\delta}
		&=\left[\frac2N+
		\sqrt{\frac{P_0\log(1+\kappa_{\rm gen}N)+\log(4/\delta)}{2N}}\right].
	\end{aligned}
\end{equation}
\par\endgroup

The positive constant $\kappa_{\rm gen}$ is explicitly specified in the proof and depends only on $L$, $(p_\ell)$, $\rho_{\rm par}$, $a_{\rm par}$, and $b_{\rm par}$. Therefore, once the architecture and these uniform bounds are fixed, the displayed terms remain independent of $T$. The probability statement concerns a single predetermined admissible $T$. For simultaneous guarantees over a predetermined finite collection $\mathcal T$, one simply replaces $\delta$ with $\delta/|\mathcal T|$.

The hypernetwork introduces no trainable coordinates, regardless of its internal parameter count, rendering any small-size assumption unnecessary. Furthermore, its exact accuracy, distributional symmetry, and strict equivariance are not required for the generalization proof. The supplementary conditions governing composite equivariance and the empirical-risk connection are detailed subsequently.

\subsection{Proof of the Generalization Bound}
\begin{proof}[Proof of Lemma~\ref{eq-ge}]
	Stack the layer coefficients into $\zeta\in\mathbb R^{q_T}$. Let $M_T(w)$ be the block-diagonal matrix configured with blocks $M_{\ell,T}(w)$ for $\ell=1,\ldots,L$. Based on \eqref{app:supp:gen-matrix-bounds}, it follows that for every $w\in\mathcal W$,
	\begin{equation}\label{app:supp:gen-block-bounds}
		a_{\rm par}\|\zeta\|_2\leq\|M_T(w)\zeta\|_2\leq b_{\rm par}\|\zeta\|_2,
		\qquad\|M_T(w)(\zeta-\zeta')\|_2\leq b_{\rm par}\|\zeta-\zeta'\|_2.
	\end{equation}
	Consequently, the Euclidean balls $B_2^{q_T}(r)$ strictly satisfy
	\begin{equation}\label{app:supp:gen-feasible-balls}
		B_2^{q_T}(\rho_{\rm par}/b_{\rm par})\subset\mathcal V_T^{\rho_{\rm par}}
		\subset B_2^{q_T}(\rho_{\rm par}/a_{\rm par}).
	\end{equation}
	The feasible set constitutes an intersection of closed ellipsoids; it is nonempty, compact, and fixed before the sample is drawn. However, the specific coefficients selected by training may depend on that sample.
	
	We subsequently bound the variation in the network output when its shared coefficients undergo changes. A zero-padded convolution with a $p_\ell\times p_\ell$ kernel obeys
	\begin{equation}\label{app:supp:gen-convolution}
		\|T_\ell(W)z\|_F\leq\sum_u\|W(u)\|_{\rm op}\|z\|_F
		\leq p_\ell\|W\|_F\|z\|_F.
	\end{equation}
	The final inequality is derived using the Cauchy-Schwarz inequality over the $p_\ell^2$ spatial offsets. A bias broadcast over the image yields a Frobenius norm of $\sqrt d\|b_\ell\|_2$. We define $C_0^{\rm act}=1$ and $C_\ell^{\rm act}=\rho_{\rm par}(p_\ell C_{\ell-1}^{\rm act}+1)$. Because the ReLU activation and output clipping operations are nonexpansive and fix zero, the layer activations consistently satisfy $\|z_\ell\|_F\leq\sqrt d\,C_\ell^{\rm act}$ for all admissible inputs, coefficients, and $w$.
	
	Consider two admissible networks evaluated at the same input and group parameter. Let $\Delta_\ell=\|z_\ell-z'_\ell\|_F$ with $\Delta_0=0$. Applying \eqref{app:supp:gen-convolution} to analyze the difference at layer $\ell$ yields
	\begin{equation}\label{app:supp:gen-layer-stability}
		\Delta_\ell\leq p_\ell\rho_{\rm par}\Delta_{\ell-1}
		+\sqrt d\sqrt{p_\ell^2(C_{\ell-1}^{\rm act})^2+1}\,
		\|\omega_\ell-\omega'_\ell\|_2.
	\end{equation}
	Iterating this inequality and applying the Cauchy-Schwarz inequality across the network layers gives
	\begin{equation}\label{app:supp:gen-network-stability}
		\|\Psi_\zeta(\hat Y,w)-\Psi_{\zeta'}(\hat Y,w)\|_F
		\leq\sqrt d\,C_{\rm net}\|M_T(w)(\zeta-\zeta')\|_2
		\leq\sqrt d\,C_{\rm net}b_{\rm par}\|\zeta-\zeta'\|_2,
	\end{equation}
	where
	\begin{equation}\label{app:supp:gen-network-constant}
		C_{\rm net}^2=\sum_{\ell=1}^L
		\left[p_\ell^2(C_{\ell-1}^{\rm act})^2+1\right]
		\prod_{j=\ell+1}^L(p_j\rho_{\rm par})^2.
	\end{equation}
	This constant depends strictly on the depth, kernel sizes, and norm budget, but remains entirely independent of $w$ or $T$.
	
	Now, substitute $w=\mathcal H(\hat Y)$ for both coefficient vectors in \eqref{app:supp:gen-network-stability}. Since $\mathcal H$ is fixed independently of $\zeta$, the descriptor remains unchanged when the restoration coefficients vary. We define the normalized loss of the complete predictor as
	\[
	g_\zeta(\hat X,\hat Y)
	=\|\hat X-\Psi_\zeta^{\mathcal H}(\hat Y)\|_F^2/d.
	\]
	This loss is bounded within $[0,1]$. For two clipped outputs, the difference between their normalized squared losses is at most $2\|\Psi_\zeta^{\mathcal H}-\Psi_{\zeta'}^{\mathcal H}\|_F/\sqrt d$. Consequently,
	\begin{equation}\label{app:supp:gen-loss-stability}
		\sup_{\hat X,\hat Y}
		|g_\zeta(\hat X,\hat Y)-g_{\zeta'}(\hat X,\hat Y)|
		\leq 2C_{\rm net}b_{\rm par}\|\zeta-\zeta'\|_2.
	\end{equation}
	Note that the supremum is taken over all admissible images. This estimate does not require $\mathcal H$ to be continuous or Lipschitz bounded with respect to the image. Equivalently, while the fixed measurable map $(\hat X,\hat Y)\mapsto(\hat X,\hat Y,\mathcal H(\hat Y))$ induces i.i.d.\ triples, the observed data and the population risk remain rigorously defined on image pairs.
	
	Take a maximal $\tau$-separated subset $\mathcal A_\tau$ of $\mathcal V_T^{\rho_{\rm par}}$. The points in this subset form a $\tau$-net within the feasible region. By referencing \eqref{app:supp:gen-feasible-balls}, a volume comparison of the Euclidean balls establishes
	\begin{equation}\label{app:supp:gen-cover-count}
		|\mathcal A_\tau|\leq
		\left(1+\frac{2\rho_{\rm par}}{a_{\rm par}\tau}\right)^{q_T}.
	\end{equation}
	Set $\tau_N=(2C_{\rm net}b_{\rm par}N)^{-1}$, $\kappa_0=4\rho_{\rm par}C_{\rm net}$, and $\kappa_{\rm gen}=\kappa_0b_{\rm par}/a_{\rm par}$. This configuration ensures $|\mathcal A_{\tau_N}|\leq(1+\kappa_{\rm gen}N)^{q_T}$. For each fixed net point, $g_\zeta$ acts as a fixed function of an observed image pair. Consequently, its values across the training pairs are i.i.d., even if the transformation matrices $M_T(\mathcal H(\hat Y_n))$ vary. Hoeffding's inequality \cite{mohri2018foundations} combined with a union bound over the net points using a failure probability of $\delta/2$ yields
	\begin{equation}\label{app:supp:gen-hoeffding}
		\max_{\zeta\in\mathcal A_{\tau_N}}
		\left|\mathbb E g_\zeta-\frac1N\sum_{n=1}^Ng_\zeta(\hat X_n,\hat Y_n)\right|
		\leq\sqrt{\frac{\log(4|\mathcal A_{\tau_N}|/\delta)}{2N}}.
	\end{equation}
	The Lipschitz bound introduced in \eqref{app:supp:gen-loss-stability} contributes one approximation term for the expectation and another for the empirical mean. Since their sum equals $4C_{\rm net}b_{\rm par}\tau_N=2/N$, the preceding relationship naturally extends to every $\zeta\in\mathcal V_T^{\rho_{\rm par}}$ as
	\begin{equation}\label{app:supp:gen-uniform}
		\sup_{\zeta\in\mathcal V_T^{\rho_{\rm par}}}
		\left|\mathbb E g_\zeta-\frac1N\sum_{n=1}^Ng_\zeta(\hat X_n,\hat Y_n)\right|
		\leq\frac2N+\sqrt{\frac{q_T\log(1+\kappa_{\rm gen}N)+\log(4/\delta)}{2N}}.
	\end{equation}
	{Since the empirical and expected risks now use $g_\zeta$ itself, no multiplication by $d$ is needed. Applying $q_T\leq P_{\rm sh}/T+P_0$ and splitting the square root yields \eqref{eq:GE_standard}. With uniform architecture and parameter bounds, these constants are independent of resolution.}
	
	For the ordinary CNN, the expanded parameters themselves reside within a Euclidean ball of radius $\rho_{\rm par}$ and dimension $P_{\rm sh}+P_0$. An identical argument applies with $\kappa_0$ replacing $\kappa_{\rm gen}$. Because $\kappa_{\rm gen}\geq\kappa_0$, we can confidently use the common logarithm in both bounds, which produces \eqref{eq:GE_CNN} under another failure probability limit of $\delta/2$. The two estimates therefore hold simultaneously with a probability of at least $1-\delta$. Measurability concerning the observations and continuity concerning the coefficients permit each class supremum to be taken over a countable dense subset of its compact coefficient set, ensuring the event remains measurable. These uniform bounds robustly include coefficients selected directly from the training data.
\end{proof}

\subsection{Shared Parameters in the Adaptive Convolutions}
The filter construction detailed in Section~\ref{sec:network-design} of the main text elucidates how the dimension reduction presented in \eqref{app:supp:gen-parameter-count} arises. Let $S=\{A(2\pi j/T):j=0,\ldots,T-1\}$ define the finite rotation group and set $G_w=D_wSD_w^{-1}$ for an invertible matrix $D_w$. Consequently, every $G_w$ contains $T$ elements, and the identical index $A\in S$ can be utilized for orientation channels across different samples.

For layer $\ell$, let $J_\ell$ denote the number of coefficients contained within a single base filter and let $\varphi_{\ell,k}$ represent its $k$th basis function, which corresponds to the coefficients $v_k$ and bases $\varphi_k$ in \eqref{TL-filter}. On the $p_\ell\times p_\ell$ filter grid with spatial offsets $\delta_{ij}$, we define the sampled basis matrix as
\begin{equation}\label{app:supp:gen-sampled-basis}
	[B_{\ell,T}(w)]_{(A,i,j),k}
	=\varphi_{\ell,k}(A^{-1}D_w^{-1}\delta_{ij}),
	\qquad B_{\ell,T}(w)\in\mathbb R^{Tp_\ell^2\times J_\ell}.
\end{equation}
This constitutes the basis evaluation utilized by the discrete filters from \eqref{app:newPsi_} through \eqref{app:newUpsilon_}. We assume
\begin{equation}\label{app:supp:gen-basis-spectral}
	J_\ell\leq p_\ell^2,\qquad
	a_{\rm par}^2I\preceq\frac1T B_{\ell,T}(w)^\top B_{\ell,T}(w)
	\preceq b_{\rm par}^2I,\qquad w\in\mathcal W.
\end{equation}
This sampled-basis condition supplements the fundamental invertibility of $D_w$. It guarantees that the shared coefficients maintain a uniform norm control for every valid group parameter.

To accurately enumerate the coefficients, let $f_\ell=c_\ell/T$ define the number of feature groups in hidden layer $\ell$ and let $B_{\ell,T}^A(w)$ denote the rows of $B_{\ell,T}(w)$ indexed by $A\in S$. The input, intermediate, and output kernels exhibit the following forms:
\begin{equation}\label{app:supp:gen-kernel-assembly}
	\begin{aligned}
		W_{1;(j,A),i}(u;w)&=[B_{1,T}^A(w)v_{1;j,i}]_u,\\
		W_{\ell;(j,A),(i,B)}(u;w)&=[B_{\ell,T}^A(w)v_{\ell;j,i,A^{-1}B}]_u,\quad 2\leq\ell<L,\\
		W_{L;j,(i,B)}(u;w)&=[B_{L,T}^B(w)v_{L;j,i}]_u.
	\end{aligned}
\end{equation}
Fix an intermediate coefficient vector indexed by $(j,i,C)$. As $A$ varies, its copies populate the kernel entries $((j,A),(i,AC),u)$, meaning that different coefficient vectors possess distinctly disjoint entries. Thus, the complete filter assembly functions as a direct sum of copies of $B_{\ell,T}(w)$ up to a predetermined permutation. Using normalized filter coefficients $\zeta_{\rm filter}=\sqrt T\,v$, each expansion block becomes $B_{\ell,T}(w)/\sqrt T$ and satisfies the essential matrix bounds in \eqref{app:supp:gen-matrix-bounds} according to \eqref{app:supp:gen-basis-spectral}.

Hidden biases are similarly shared across the $T$ directions. If $b_{\ell;j}^{\rm raw}$ represents one such bias, defining $\zeta_{\rm bias,\ell;j}=\sqrt T\,b_{\ell;j}^{\rm raw}$ renders its repeated expansion an isometry. The output bias remains unshared. There are $c_{\ell-1}c_\ell/T$ base-filter vectors present at each layer, incorporating both the input and output layers. Consequently,
\begin{equation}\label{app:supp:gen-count-verification}
	q_T=\sum_{\ell=1}^L\frac{J_\ell c_{\ell-1}c_\ell}{T}
	+\sum_{\ell=1}^{L-1}\frac{c_\ell}{T}+c_L
	\leq\frac{P_{\rm sh}}{T}+P_0.
\end{equation}
If batch normalization is implemented with frozen inference statistics and affine parameters shared across directions, it can be folded directly into the convolution, provided the same folded coefficient vector successfully represents the filters for every $w$. The norm budget is rigorously imposed after this folding process. The fixed $\mathcal H$ contributes no trainable coordinates to the composite class. This calculation pertains exclusively to the single-channel restoration backbone discussed above and excludes any additional trainable modules.

\subsection{Empirical-Risk Connection and Sampling}\label{app:supp:gen-risk-connection}
This subsection establishes the connection between the generalization bound and the empirical-risk comparison presented in Theorem~\ref{ER}. Let $\mathcal E_{\rm gen}$ denote the simultaneous uniform event of Lemma~\ref{eq-ge}, with
$\Pr_{\mathcal P^{\otimes N}}(\mathcal E_{\rm gen}^c)\leq\delta$.
Let $\Psi^*_{EQ}$ and $\Psi^*$ denote measurable empirical-risk minimizers within
$\Lambda_{EQ}^{\rho_{\rm par},\mathcal H}$ and
$\Lambda_{\rm CNN}^{\rho_{\rm par}}$, respectively.
These minima reliably exist because the coefficient sets are compact and the empirical losses are continuous with respect to the coefficients.
The separate premise supporting the main-text risk comparison is expressed as

\begin{equation}\label{app:supp:gen-empirical-premise}
	\EM(\Psi^*_{EQ})\leq
	\EM(\Psi^*)+C_hh+C_\varepsilon\sqrt{\varepsilon}.
\end{equation}

To derive this from Theorem~\ref{ER}, the underlying approximation and equivariance assumptions must hold for these exact predictors, and the Reynolds candidate $\bar\Psi$ from Lemma~\ref{bar_f} must correspond to the same architecture featuring the identical fixed $\mathcal H$ and norm budget. Any approximation of that specific candidate must be actively accounted for in the empirical-risk estimate. Unrestricted universal approximation alone is insufficient to establish this constrained realizability.

The terms $C_h$ and $C_\varepsilon$ denote the empirical-risk coefficients specified in Theorem~\ref{ER}, which may depend on $h$, $\varepsilon$, and $\EM(\Psi^*)$. The constants $C$ and $\hat C$ retain their original meaning from Theorem~\ref{Error-dis}. For a specified $\varepsilon$, let $\mathcal E_{\rm emp}(\varepsilon)$ be the precise event where \eqref{app:supp:gen-empirical-premise} holds. On the event $\mathcal E_{\rm gen}\cap\mathcal E_{\rm emp}(\varepsilon)$, combining the generalization bounds yields \eqref{eq:Expected_EQ} and \eqref{eq:Expected_Std}. The upper bound for the equivariant network becomes strictly tighter when
\[
C_hh+C_\varepsilon\sqrt{\varepsilon}<U_N(1-T^{-1/2}),
\]
which corresponds to \eqref{eq:trade_off_condition}. As established in the main text, this strictly evaluates theoretical upper bounds and does not dictate the definitive ordering of the actual expected risks. Since the coefficients are evaluated using the identical dataset and predictors, the inequality accurately reflects the necessary trade-off between the empirical penalty and the corresponding reduction in the generalization bound. Even without the empirical-risk premise, the lemma continues to guarantee
\[
\mathcal L_{\rm exp}(\Psi^*_{EQ})
\leq\EM(\Psi^*_{EQ})+U_N/\sqrt T+V_{N,\delta}.
\]

In particular,
\begin{equation}\label{app:supp:gen-intersection}
	\Pr_{\mathcal P^{\otimes N}}\{
	\mathcal E_{\rm emp}(\varepsilon)
	\text{ and \eqref{eq:Expected_EQ} fails}\}\leq\delta.
\end{equation}
If the stated premise holds with a probability of at least $1-\delta_{\rm emp}$, then the premise and the risk bounds hold concurrently with a probability of at least $1-\delta-\delta_{\rm emp}$. No assumption of independence between these events is necessary. Conditioning directly on the empirical premise does not automatically preserve the original confidence level: if it has a probability of $p_{\rm emp}>0$, the available conditional failure bound becomes $\min\{1,\delta/p_{\rm emp}\}$. Identical reasoning applies to the nonzero-$\eta$ empirical estimate outlined in Section~\ref{app:supp:nonzero-eta}.

The assumption of i.i.d.\ sampling for the original supervised pairs remains compatible with the jointly constrained restoration variables specified in \eqref{Modelall} through \eqref{Modelall_dis}. While the optimization outputs and a trained predictor may depend on the entire dataset, the uniform bound naturally accommodates data-dependent coefficients within a predetermined class. However, generating pseudo-labels jointly or selecting datasets specifically to enforce symmetry does not necessarily preserve the original product sampling law. Similarly, population symmetry does not guarantee the exact symmetry of any given finite sample. Consequently, the generalization lemma neither relies on a finite-sample symmetry event nor attempts to supply its probability.

The generalization analysis above assumes a hypernetwork fixed before sampling; it does not directly cover joint training of the hypernetwork and restoration backbone.

\endgroup

\begingroup\color{black}

\section{Risk Analysis with Approximation Error}\label{app:supp:nonzero-eta}

Assumption~\ref{assum} allows a nonzero approximation error $\eta\geq0$ between the trained network $\Psi^*$ and the optimal restoration operator $\tilde\Phi$. We now retain this error in the empirical risk analysis of Theorem~\ref{ER}, and then combine the result with the generalization bounds of Lemma~\ref{eq-ge}. We use the same notation as in the main text: $C$ and $\hat C$ are the equivariance-error constants in Theorem~\ref{Error-dis}, $C_h$ and $C_\varepsilon$ are the empirical-risk coefficients in Theorem~\ref{ER}. {Here $\eta$ bounds the Frobenius approximation error, so the per-pixel root-mean-square bound is $\eta/m_r$. All image losses are averaged over $m_r^2$ output pixels. It should be noted that setting $\eta=0$ recovers the main-text empirical and expected risk bounds.}

\subsection{Approximation on Transformed Inputs}

For each sample, let $G_n=\mathcal H_G(\hat Y_n)=\{A_n^t\}_{t=1}^T$ be the finite transformation group in Assumption~\ref{ass:adaptive_group}. The selected group is unchanged when the input is transformed by one of its elements. We use the same linear discrete group actions $\tilde\pi_{A_n}$. To control transformed approximation errors, let $\kappa\geq1$ satisfy
\begin{equation}\label{app:supp:eta-action-bound}
    \|\tilde\pi_{A_n}(Y)\|_F\leq\kappa\|Y\|_F,
    \qquad A_n\in G_n,\quad n=1,\ldots,N,
\end{equation}
for every output image $Y$. For Frobenius-norm-preserving actions, take $\kappa=1$. {The same $\kappa$ is used uniformly over the grid sequence and all allowed group parameters. Dividing both sides by $m_r$ gives the identical operator bound for the normalized image norm.}

	According to Assumption~\ref{assum}, for each observed input $\hat{Y}_n$ and any $A_n\in G_n$, we have:
	\begin{equation}\label{app:supp:eta-approximation}
		\begin{aligned}
			\|\Psi^*(\hat{Y}_n)-\tilde{\Phi}(\hat{Y}_n)\|_F&\leq{\eta},\\
			\big\|\Psi^*\bigl(\tilde{\pi}_{A_n}(\hat{Y}_n)\bigr)
			-\tilde{\Phi}\bigl(\tilde{\pi}_{A_n}(\hat{Y}_n)\bigr)\big\|_F&\leq{\eta}.
		\end{aligned}
	\end{equation}
	
	As in Theorem~\ref{ER}, assume the conditions of Theorem~\ref{Error-dis} hold, with the same constants, for every transformed dataset $\{\tilde\pi_{B_n}(\hat Y_n)\}_{n=1}^N$, $B_n\in G_n$, $n=1,\ldots,N$. Then, for any $A_n\in G_n$, $n=1,\ldots,N$, Theorem~\ref{Error-dis} gives:
	\begin{equation}\label{app:supp:eta-operator-bound}
		{\frac{1}{Nm_r^2}}\sum_{n=1}^N
		\big\|\tilde{\Phi}\bigl(\tilde{\pi}_{A_n}(\hat{Y}_n)\bigr)
		-\tilde{\pi}_{A_n}\bigl(\tilde{\Phi}(\hat{Y}_n)\bigr)\big\|_F^2
		\leq Ch^2+\hat C\varepsilon.
	\end{equation}
	As in Theorem~\ref{ER}, we apply this estimate to both the original and transformed inputs with the same constants $C$ and $\hat C$.
	
	\begin{lemma}[Network equivariance with approximation error]\label{app:supp:eta-network}
		Under \eqref{app:supp:eta-action-bound}--\eqref{app:supp:eta-operator-bound}, for any $A_n\in G_n$, $n=1,\ldots,N$, the trained network satisfies
		\begin{equation}\label{app:supp:eta-network-bound}
			\begin{aligned}
				{\frac{1}{Nm_r^2}}\sum_{n=1}^N
				\big\|\Psi^*\bigl(\tilde{\pi}_{A_n}(\hat{Y}_n)\bigr)
				-\tilde{\pi}_{A_n}\bigl(\Psi^*(\hat{Y}_n)\bigr)\big\|_F^2 \leq\left(\sqrt{Ch^2+\hat C\varepsilon}+{\frac{1+\kappa}{m_r}\eta}\right)^2.
			\end{aligned}
		\end{equation}
	\end{lemma}
	
	\begin{proof}
		By the linearity of $\tilde{\pi}_{A_n}$, we have:
		\begin{equation}\label{app:supp:eta-decomposition}
			\begin{aligned}
				& \Psi^*\bigl(\tilde{\pi}_{A_n}(\hat{Y}_n)\bigr)
				-\tilde{\pi}_{A_n}\bigl(\Psi^*(\hat{Y}_n)\bigr) \\
				=& \bigl[\Psi^*\bigl(\tilde{\pi}_{A_n}(\hat{Y}_n)\bigr)
				-\tilde{\Phi}\bigl(\tilde{\pi}_{A_n}(\hat{Y}_n)\bigr)\bigr]
				+\bigl[\tilde{\Phi}\bigl(\tilde{\pi}_{A_n}(\hat{Y}_n)\bigr)
				-\tilde{\pi}_{A_n}\bigl(\tilde{\Phi}(\hat{Y}_n)\bigr)\bigr]
				+\tilde{\pi}_{A_n}\bigl(\tilde{\Phi}(\hat{Y}_n)-\Psi^*(\hat{Y}_n)\bigr).
			\end{aligned}
		\end{equation}
		Applying Minkowski's inequality gives:
		\begin{equation}\label{app:supp:eta-minkowski}
			\begin{aligned}
				&\left[{\frac{1}{Nm_r^2}}\sum_{n=1}^N
				\big\|\Psi^*\bigl(\tilde{\pi}_{A_n}(\hat{Y}_n)\bigr)
				-\tilde{\pi}_{A_n}\bigl(\Psi^*(\hat{Y}_n)\bigr)\big\|_F^2\right]^{1/2}\\
				&\quad\leq\left[{\frac{1}{Nm_r^2}}\sum_{n=1}^N
				\big\|\Psi^*\bigl(\tilde{\pi}_{A_n}(\hat{Y}_n)\bigr)
				-\tilde{\Phi}\bigl(\tilde{\pi}_{A_n}(\hat{Y}_n)\bigr)\big\|_F^2\right]^{1/2}\\
				&\qquad+\left[{\frac{1}{Nm_r^2}}\sum_{n=1}^N
				\big\|\tilde{\Phi}\bigl(\tilde{\pi}_{A_n}(\hat{Y}_n)\bigr)
				-\tilde{\pi}_{A_n}\bigl(\tilde{\Phi}(\hat{Y}_n)\bigr)\big\|_F^2\right]^{1/2}\\
				&\qquad+\left[{\frac{1}{Nm_r^2}}\sum_{n=1}^N
				\big\|\tilde{\pi}_{A_n}\bigl(\tilde{\Phi}(\hat{Y}_n)-\Psi^*(\hat{Y}_n)\bigr)\big\|_F^2\right]^{1/2}\\
				&\quad\leq{\frac{\eta}{m_r}}+\sqrt{Ch^2+\hat C\varepsilon}+{\frac{\kappa\eta}{m_r}}.
			\end{aligned}
		\end{equation}
		The last step follows from \eqref{app:supp:eta-approximation}, \eqref{app:supp:eta-operator-bound}, and \eqref{app:supp:eta-action-bound}. Squaring both sides completes the proof.
	\end{proof}

\subsection{Empirical Risk of Equivariant Networks}

Using the network equivariance estimate above, we now extend Theorem~\ref{ER} to nonzero $\eta$. The empirical risk $\EM$ and the network classes $\Lambda$ and $\Lambda_{EQ}$ are defined in \eqref{emp_risk}, \eqref{DL-Model}, and \eqref{eqNet}, respectively. As in the main text, we use the unified Reynolds mapping $\bar\Psi$ defined in Lemma~\ref{bar_f}, whose equivariance follows from Assumption~\ref{ass:adaptive_group}. When using the prescribed classes of Section~\ref{app:supp:gen-assumptions}, this candidate must belong to the same class, as discussed in Section~\ref{app:supp:gen-risk-connection}.

\begin{theorem}[Empirical Risk of Equivariant Network]\label{app:supp:eta-empirical}
    Consider the image restoration problem and observed dataset $\{\hat Y_n\}_{n=1}^N$ in Theorem~\ref{ER}, under Assumption~\ref{ass:adaptive_group} and with the same discretization, range, boundedness, and local quadratic growth conditions. Let $\Psi^*\in\Lambda$ and $\Psi^*_{EQ}\in\Lambda_{EQ}$ denote the standard and equivariant networks defined there, with $\Psi^*_{EQ}$ minimizing $\EM$ over $\Lambda_{EQ}$. Under the conditions of Lemma~\ref{app:supp:eta-network}, their empirical risks satisfy:
    \begin{equation}\label{app:supp:eta-empirical-expanded}
        \EM(\Psi^*_{EQ})\leq
        \EM(\Psi^*)+C_hh+C_\varepsilon\sqrt\varepsilon+C_\eta\eta,
    \end{equation}
    where $C_h=2\sqrt{C\,\EM(\Psi^*)}+Ch$ and $C_\varepsilon=2\sqrt{\hat C\,\EM(\Psi^*)}+\hat C\sqrt\varepsilon$ are the coefficients in Theorem~\ref{ER}, and
    $
        {C_\eta=\frac{2(1+\kappa)}{m_r}\left(\sqrt{\EM(\Psi^*)}+\sqrt{Ch^2+\hat C\varepsilon}\right)+\frac{(1+\kappa)^2}{m_r^2}\eta.}
    $
    Here $C$ and $\hat C$ are the constants in Theorem~\ref{Error-dis}, $\eta$ is the approximation error in Assumption~\ref{assum}, and $\kappa$ is defined in \eqref{app:supp:eta-action-bound}. The coefficients remain bounded when $h$, $\varepsilon$, $\eta$, $\EM(\Psi^*)$, $\kappa$, and the equivariance-error constants are uniformly bounded.
\end{theorem}

\begin{proof}
   Let $\bar{\Psi}(\hat{Y}_n)=\frac{1}{T}\sum_{A\in\mathcal H_G(\hat{Y}_n)}\tilde{\pi}_{A}\bigl(\Psi^*(\tilde{\pi}_{A}^{-1}(\hat{Y}_n))\bigr)$. By Lemma~\ref{bar_f}, $\bar\Psi\in\Lambda_{EQ}$. Since $\Psi^*_{EQ}$ minimizes the empirical risk over $\Lambda_{EQ}$, we have:
    \begin{equation}\label{app:supp:eta-risk-candidate}
			\EM(\Psi^*_{EQ})\leq{\frac1{Nm_r^2}}\sum_{n=1}^N\|\hat X_n-\bar\Psi(\hat Y_n)\|_F^2.
		\end{equation}
    Let $e_n=\hat X_n-\Psi^*(\hat Y_n)$ and $d_n=\Psi^*(\hat Y_n)-\bar\Psi(\hat Y_n)$, and write $E=\EM(\Psi^*)={\frac1{Nm_r^2}}\sum_{n=1}^N\|e_n\|_F^2$ and $D={\frac1{Nm_r^2}}\sum_{n=1}^N\|d_n\|_F^2$. Expanding the squared residual and applying the Cauchy--Schwarz inequality yields{:}
    \begin{equation}\label{app:supp:eta-risk-comparison}
			\EM(\Psi^*_{EQ})\leq{\frac1{Nm_r^2}}\sum_{n=1}^N\|e_n+d_n\|_F^2
			=E+D+{\frac2{Nm_r^2}}\sum_{n=1}^N\langle e_n,d_n\rangle_F
			\leq E+D+2\sqrt{ED}.
		\end{equation}
    By Lemma~\ref{bar_f}, $\bar\Psi(\hat Y_n)=\frac1T\sum_{t=1}^T\tilde\pi_{A_n^t}[\Psi^*[\tilde\pi_{A_n^t}^{-1}]](\hat Y_n)$, where $A_n^t$ is the $t$-th element of $G_n=\mathcal H_G(\hat Y_n)$. Applying Jensen's inequality, interchanging the finite sums, and using Lemma~\ref{app:supp:eta-network} for each fixed $t$, we have
    \begin{equation}\label{app:supp:eta-risk-average}
			\begin{aligned}
				D &={\frac1{Nm_r^2}}\sum_{n=1}^N\left\|\Psi^*(\hat Y_n){-}\bar\Psi(\hat Y_n)\right\|_F^2 ={\frac1{Nm_r^2}}\sum_{n=1}^N\left\|\frac1T\sum_{t=1}^T\left[\Psi^*(\hat Y_n)-\tilde\pi_{A_n^t}\left(\Psi^*(\tilde\pi_{A_n^t}^{-1}(\hat Y_n))\right)\right]\right\|_F^2\\
				&\leq{\frac1{Nm_r^2}}\sum_{n=1}^N\frac1T\sum_{t=1}^T\left\|\Psi^*(\hat Y_n)-\tilde\pi_{A_n^t}\left(\Psi^*(\tilde\pi_{A_n^t}^{-1}(\hat Y_n))\right)\right\|_F^2\\
				&=\frac1T\sum_{t=1}^T{\frac1{Nm_r^2}}\sum_{n=1}^N\left\|\Psi^*(\hat Y_n)-\tilde\pi_{A_n^t}\left(\Psi^*(\tilde\pi_{A_n^t}^{-1}(\hat Y_n))\right)\right\|_F^2
				\leq\left(\sqrt{Ch^2+\hat C\varepsilon}+{\frac{1+\kappa}{m_r}\eta}\right)^2.
			\end{aligned}
		\end{equation}
    Here, the first inequality is Jensen's inequality for the convex squared norm. Letting $\tilde Y_n=\tilde\pi_{A_n^t}^{-1}(\hat Y_n)$, so that $\hat Y_n=\tilde\pi_{A_n^t}(\tilde Y_n)$, Lemma~\ref{app:supp:eta-network} applied to $\{\tilde Y_n\}_{n=1}^N$ bounds the inner dataset average by $\left(\sqrt{Ch^2+\hat C\varepsilon}+{\frac{1+\kappa}{m_r}\eta}\right)^2$ for every such $t$. Assumption~\ref{assum} covers these inputs and their transformations, and the transformed-dataset premise above guarantees the same equivariance-error constants for each application.

    Substituting this estimate into \eqref{app:supp:eta-risk-comparison} and expanding the terms involving $\eta$ gives
    \[
        \begin{aligned}
        \EM(\Psi^*_{EQ})
        &\leq E+\left(\sqrt{Ch^2+\hat C\varepsilon}+{\frac{1+\kappa}{m_r}\eta}\right)^2
              +2\sqrt E\left(\sqrt{Ch^2+\hat C\varepsilon}+{\frac{1+\kappa}{m_r}\eta}\right)\\
        &=E+Ch^2+\hat C\varepsilon+2\sqrt E\sqrt{Ch^2+\hat C\varepsilon}+C_\eta\eta.
        \end{aligned}
    \]
    Using $\sqrt{x+y}\leq\sqrt x+\sqrt y$ for $x,y\geq0$, as in the main text, we obtain
    \begin{equation}\label{app:supp:eta-empirical-bound}
        \begin{aligned}
        \EM(\Psi^*_{EQ})
        &\leq E+Ch^2+\hat C\varepsilon+2\sqrt{C\,E}\,h+2\sqrt{\hat C\,E}\,\sqrt\varepsilon+C_\eta\eta\\
        &=E+\bigl(Ch+2\sqrt{C\,E}\bigr)h+\bigl(\hat C\sqrt\varepsilon+2\sqrt{\hat C\,E}\bigr)\sqrt\varepsilon+C_\eta\eta\\
        &=E+C_hh+C_\varepsilon\sqrt\varepsilon+C_\eta\eta.
        \end{aligned}
    \end{equation}
    The proof is completed.
\end{proof}

\textbf{Remark.}
For Frobenius-norm-preserving transformations, $\kappa=1$ and the approximation contribution is
\[
    {C_\eta\eta=\frac{4\eta}{m_r}\left(\sqrt E+\sqrt{Ch^2+\hat C\varepsilon}+\frac{\eta}{m_r}\right).}
\]
When $\eta=0$, this contribution vanishes and \eqref{app:supp:eta-empirical-expanded} reduces to \eqref{eq:empirical-bound}, with exactly the same coefficients $C_h$ and $C_\varepsilon$. The result separates the empirical penalty due to spatial discretization and data symmetry from the additional contribution of network approximation. With the baseline empirical risk and other quantities fixed, reducing $\varepsilon$ or $\eta$ tightens this upper bound.

\subsection{Generalization and Expected Risk Bounds}

We now combine the empirical-risk bound in Theorem~\ref{app:supp:eta-empirical} with the generalization bound stated in Lemma~\ref{eq-ge} of the main text. We use the network classes in Section~\ref{app:supp:gen-assumptions}, with the same ideal hypernetwork $\mathcal H$ fixed before sampling. As discussed in Section~\ref{app:supp:gen-risk-connection}, both bounds are applied to the same predictors and network classes.

The approximation error $\eta$ is incorporated into the empirical-risk comparison, while the prescribed network classes and the sampling assumptions of Lemma~\ref{eq-ge} remain unchanged. Its proof in Section~\ref{app:supp:fixed-group-generalization} therefore gives the same generalization bounds, simultaneously with probability at least $1-\delta$:
\[
    \begin{aligned}
    |GE(\Psi_{EQ},D_N)|&\leq U_N/\sqrt T+V_{N,\delta},
        &&\forall\Psi_{EQ}\in\Lambda_{EQ},\\
    |GE(\Psi,D_N)|&\leq U_N+V_{N,\delta},
        &&\forall\Psi\in\Lambda.
    \end{aligned}
\]
The terms $U_N$ and $V_{N,\delta}$ are defined in \eqref{app:supp:gen-main-terms} and have the same meaning as in the main text. Since these bounds are uniform over the prescribed classes, they also apply to the networks selected by training.

On this uniform event, whenever the empirical-risk premise of Theorem~\ref{app:supp:eta-empirical} holds, adding the empirical risk and generalization error gives
\begin{equation}\label{app:supp:eta-expected-bound}
    \begin{aligned}
    \mathcal L_{\mathrm{exp}}(\Psi^*_{EQ})
    &\leq\EM(\Psi^*_{EQ})+U_N/\sqrt T+V_{N,\delta}\\
    &\leq\EM(\Psi^*)+C_hh+C_\varepsilon\sqrt\varepsilon
        +C_\eta\eta+U_N/\sqrt T+V_{N,\delta}.
    \end{aligned}
\end{equation}
For the standard network, the same uniform event gives \eqref{eq:Expected_Std}, that is
\[
    \mathcal L_{\mathrm{exp}}(\Psi^*)\leq\EM(\Psi^*)+U_N+V_{N,\delta}.
\]
Thus the expected risk upper bound for the equivariant network is strictly tighter when
\begin{equation}\label{app:supp:eta-comparison-condition}
    C_hh+C_\varepsilon\sqrt\varepsilon+C_\eta\eta<U_N(1-T^{-1/2}).
\end{equation}
The right-hand side is the reduction in the generalization upper bound due to coefficient sharing, while the left-hand side collects the empirical penalties from spatial discretization, data symmetry, and network approximation. As in \eqref{eq:trade_off_condition}, this comparison provides qualitative guidance on their trade-off. 

For $\eta=0$, \eqref{app:supp:eta-expected-bound} and \eqref{app:supp:eta-comparison-condition} reduce to \eqref{eq:Expected_EQ} and \eqref{eq:trade_off_condition}. For a fixed group $G_n=G$, the same comparison applies under the fixed-group conditions without requiring an ideal hypernetwork. These results retain the main-text guidance: reducing the mismatch between network equivariance and data symmetry reduces the empirical penalty that should be balanced against the benefit of coefficient sharing. The probability and sampling qualifications in Section~\ref{app:supp:gen-risk-connection} apply to this comparison as well.

The expected-risk bounds in this section retain the fixed-hypernetwork assumption of Lemma~\ref{eq-ge}, together with the equivariance and candidate-realizability conditions stated in Section~\ref{app:supp:gen-risk-connection}. They do not directly cover the jointly trained hypernetwork.
% END APPROVED-REVIEW 03-expected-risk-scope

\endgroup

\section{Details of the Sample-Adaptive Equivariant Network}
	This section provides further details of the sample-adaptive equivariant network in Section~\ref{sec:network-design}. We first describe the transformation adaptive convolutions in the continuous and discrete domains. 
We then discuss the compatibility condition for equivariance of the complete predictor and its implications for the practical parameter-prediction mechanism.
% END APPROVED-REVIEW 03-network-section-scope

%	Throughout this network construction, continuous coordinates are reference pixel coordinates with unit spacing. The functions $r$ and $e$ below represent images and features in these coordinates at the resolution being processed. These local coordinates are distinct from the fixed physical coordinates in the inverse-problem analysis.
\subsection{The framework in Continuous Domain}
	The convolution for input layer is denoted as $\Psi_w$, with $w$ being the adaptive transformation parameters. It maps an input $r\in C^\infty(\R^2)$ to a feature map defined on $E(2)$. Specifically,   $\forall (x, A)\in E(2)$:
	\begin{equation}\label{app:new_input_Conv}
		\Psi_w[r](x,A) \!=\!\!  \int_{\R^2}\!\varphi_{in}\!\left(A^{-1}\!D_w^{-1} \delta \right)\!r(x-\delta)d\sigma(\delta),
	\end{equation}
	where $\sigma$ is the  measure on $\R^2$, $A$ denotes the rotation transformation, which is also the index of group dimension, $D_w$ is the adaptive transformation, and $\varphi_{in}$ is the parameterized filter as defined in the main text. 
	
	Let $\Phi_w$ denote the convolution in the intermediate layer,
	which maps a feature map $e\in C^\infty(E(2))$ to another, then for $\forall (x, B)\in E(2)$:
	\begin{equation}\label{app:new_regular_Conv}
		\begin{split}
			\Phi_w[e](x,\!B) \!\!=\!\!\!
			\int_{\!O(2)}\!\int_{\R^2}\!\!\varphi_{\!A}\!\left(B^{-\!1}D_w^{-\!1}\delta\right)\!e(x\!-\!\delta, \!B\!A)d\sigma(\delta)dv(A),
		\end{split}
	\end{equation}
	where $v$ is a measure on $O(2)$\footnote{$O(2)$ is the orthogonal group.}, $A,B\!\in\! O(2)$ denote transformations in the considered group, and $\varphi_A$ is the parameterized filter.
	
	Let $\Upsilon_w$ denote the convolution in the output layer,
	which maps a feature map $e\in C^\infty(E(2))$ to a function defined on $\R^2$, for $\forall x\in \R^2$:
	\begin{equation}\label{app:new_output_Conv}
		\Upsilon_w[e](x) \!\!=\!\!\!  \int_{O(2)}\!\int_{\R^2}\!\!\!\varphi_{out}\!\left(\!B^{-1}\!D_w^{-1}\!\delta\right)\!e(x\!-\!\delta, \!B)d\sigma(\!\delta\!)dv(\!B\!),
	\end{equation}
	where $B \in O(2)$ and $\varphi_{out}$ is the parameterized filter.
	\subsection{The framework in Discrete Domain}\label{app:supp:network-discrete-grid}
	The filters in (5.3) of the main text are in the continuous domain, while in practical applications, we can only use their discrete form. 
	For the finite rotation group $S$ considered here, the group variable is summed over $S$, while the spatial variables are sampled in the reference pixel coordinates. For layer $\ell$, $p_\ell$ denotes its fixed filter width in pixels, consistently with Section~\ref{app:supp:gen-assumptions}, which is independent of the physical blur-grid size $p$ in \eqref{fixed-windows}. The layer index is suppressed in the spatial offsets below. The spatial discretization of the continuous formulation can be easily performed following the strategy of previous work \cite{tan2026image,xie2022fourier}.
	
	Specifically, for $i, j = 1,2, \cdots, p_\ell$, $A, B\in S$\footnote{The rotation transformation group $S$ is defined as: $S\!=\!\l\{\!A\! \!=\!\!
		\begin{bmatrix}
			\cos \theta\! & \!\sin \theta \\
			-\sin \theta\! &  \!\cos \theta
		\end{bmatrix}\bigg| \theta \!=\! \nicefrac{2\pi k}{T}, k = 0,1,\!\cdots\!,T\!-\!1
		\r\}$. } and transformation $D_w$, we  define the discrete filters for the input layer as:
	\begin{equation}\label{app:newPsi_}
		(\tilde{\Psi}_w)_{ij}^A =\sum_k{v_k^{in}\varphi_k(A^{-1}D^{-1}_w\delta_{ij})}:= \varphi_{in}\left(A^{-1}D_w^{-1}\delta_{ij}\right),
	\end{equation}
	where $\delta_{ij} = \left(i-\nicefrac{(p_\ell+1)}{2}, j-\nicefrac{(p_\ell+1)}{2}\right)^T$ is the fixed reference filter grid, 
	$\varphi_{in}$ is the parameterized filter of input layer, and $v^{in}_k$ is its coefficient.

	Similarly, we define the discrete filters for intermediate and output layers as follows:
	\begin{equation}\label{app:newPhi_}
		(\tilde{\Phi}_w)_{ij}^{B,A} = \varphi_A\left(B^{-1}D_w^{-1}\delta_{ij}\right),
	\end{equation}
	\begin{equation}\label{app:newUpsilon_}
		(\tilde{\Upsilon}_w)_{ij}^{B} = \varphi_{out}\left(B^{-1}D_w^{-1}\delta_{ij}\right),
	\end{equation}
	where $\varphi_A$ indicates the parameterized filter with respect to the channel of feature map indexed by $A$ in intermediate layer,  and $\varphi_{out}$ is the parameterized filter in the output layer.
	
	Based on the continuous convolutions and the discrete filters defined above, we can construct the discrete convolutions.
	
	We assume that an image $Y\in \R^{n\times n}$ represents a two-dimensional grid function obtained by discretizing a smooth function, i.e., for $i, j = 1,2, \cdots, n$,
	\begin{equation}\label{app:I}
		Y_{ij} = r(x_{ij}),
	\end{equation}
	where $x_{ij} = \left(i-\frac{n+1}{2}, j-\frac{n+1}{2}\right)^T$ are reference pixel coordinates. In physical coordinates these offsets are multiplied by $h$, and a fixed $p_\ell\times p_\ell$ network filter therefore has a physical footprint proportional to $h$. $F$ represents a three-dimensional grid function sampled from a smooth function $e: \R^2\times S\to \R$, i.e., for $i, j = 1,2, \cdots, n$,
	\begin{equation}
		F_{ij}^A = e(x_{ij}, A),
	\end{equation}
	where $A\in S$.
	Accordingly, we can discretize the continuous convolutions for the input layer, intermediate layer, and output layer (i.e., (\ref{app:new_input_Conv}), (\ref{app:new_regular_Conv}), and (\ref{app:new_output_Conv})). For $A\in S$ and $w\in\mathbb{R}^3$ with $s_xs_y\neq0$, and $i,j = 1,2,\cdots,n,$ $\tilde{\Psi}_w$ denotes the convolution of input layer, which can be calculated by:
	\begin{equation}\label{app:new_input_}
		\left(\tilde{\Psi}_w\star Y\right)^{A}  =(\tilde{\Psi}_w)^A*Y, ~ \forall A\in S.
	\end{equation}
	
	Similarly, for any $B\in S$ and $i,j = 1,2,\cdots,n,$ we can define the discrete convolutions of intermediate layer as the discretization of \eqref{app:new_regular_Conv}:
	\begin{equation}\label{app:new_regular_}
		\left(\tilde{\Phi}_w\star F\right)^B =\sum_{A\in S}\left(\tilde{\Phi}_w\right)^{B,A}*F^{BA},~  \forall B\in S.
	\end{equation}
	To ensure efficient implementation, we avoid channel transformations on feature maps and instead use an equivalent form of the above equation:
	\begin{equation}\label{app:new_regular_2}
		\left(\tilde{\Phi}_w\star F\right)^B =\sum_{A\in S}\left(\tilde{\Phi}_w\right)^{B,B^{-1}A}*F^{A},~  \forall B\in S.
	\end{equation}
	It is clear that (\ref{app:new_regular_2}) can be implemented by applying the commonly used convolution on $F$ and a $p_\ell\times p_\ell\times T\times T$ filter rotated spatially and shifted cyclically along the group dimension. 
	
	Then the convolution of output layer can be formulated as:
	\begin{equation}\label{app:new_output_}
		\tilde{\Upsilon}_w\star F =\sum_{B\in S}\left(\tilde{\Upsilon}_w\right)^{B}\ast F^{B}.
	\end{equation}

\begingroup\color{black}

\subsection{Equivariance of the Complete Network}\label{app:supp:gen-composite-equivariance}
The proposed method introduces a hypernetwork that acts as a selector for the transformation parameters. While the restoration backbone is strictly equivariant when the parameter $w$ is fixed, the equivariance of the complete predictor additionally relies on the consistency of the hypernetwork. 
In this section, we jointly discuss the hypernetwork and the restoration backbone to analyze the variations or potential errors in equivariance introduced by this parameter prediction mechanism.
We denote the fixed-parameter backbone by $\Psi_w$ and the complete predictor by $F(Y)=\Psi_{\mathcal{H}(Y)}(Y)$.

% END METRIC-ALIGNMENT sm103-network-notation

Assumption~\ref{ass:adaptive_group} in the main text requires consistency of the selected group $\mathcal H_G(Y)=G_{\mathcal H(Y)}$. Here we state a sufficient condition directly on the predicted parameters, since different parameters may specify the same group.

For fixed transformation parameters $w$, the backbone in \cite{tan2026image} is equivariant under the corresponding group $G_w$:

\begin{equation}
    \Psi_w(\tilde{\pi}_A(Y))
    =\tilde{\pi}_A(\Psi_w(Y)),\qquad A\in G_w.
\end{equation}

To extend this property to the complete predictor, it is sufficient for the hypernetwork to maintain consistency along the transformation orbits. Specifically, assuming all transformed inputs lie within the valid domain, we require:
\begin{equation}\label{app:supp:gen-orbit-consistency}
	\mathcal H(\tilde\pi_A(Y))=\mathcal H(Y),
	\qquad A\in G_{\mathcal H(Y)}.
\end{equation}

For $w=\mathcal{H}(Y)$ and $A\in G_w$, the compatibility condition and fixed-parameter equivariance give:
% END METRIC-ALIGNMENT sm103-identity-introduction

\begin{equation}\label{app:supp:gen-composite-identity}
    \begin{aligned}
        F(\tilde{\pi}_A(Y))
        &=\Psi_{\mathcal{H}(\tilde{\pi}_A(Y))}(\tilde{\pi}_A(Y))\\
        &=\Psi_w(\tilde{\pi}_A(Y))
        =\tilde{\pi}_A(\Psi_w(Y))\\
        &=\tilde{\pi}_A(\Psi_{\mathcal{H}(Y)}(Y))
        =\tilde{\pi}_A(F(Y)).
    \end{aligned}
\end{equation}
% END METRIC-ALIGNMENT sm103-complete-identity

The above derivation rigorously establishes a sufficient compatibility condition for end-to-end equivariance: the complete predictor inherits the backbone's equivariance when the hypernetwork acts as an invariant parameter selector along the transformation orbit, satisfying \eqref{app:supp:gen-orbit-consistency}. 

In practical scenarios, however, exact algebraic invariance might be challenging to maintain due to network approximation limits or spatial discretization. Nevertheless, this parameter prediction mechanism is robust. Recall the complete mapping $F(Y)=\Psi_{\mathcal{H}(Y)}(Y)$ and the fixed-parameter equivariance $\Psi_w(\tilde{\pi}_A(Y))=\tilde{\pi}_A(\Psi_w(Y))$ for $A\in G_w$ established above. When the hypernetwork prediction is inexact, i.e., $\mathcal{H}(\tilde{\pi}_A(Y))\neq\mathcal{H}(Y)$, the equivariance of the complete predictor is violated, and the complete-network discrepancy evaluates to
\begin{equation}\label{app:supp:gen-inexact-discrepancy}
	F(\tilde{\pi}_A(Y))-\tilde{\pi}_A(F(Y))
	=\Psi_{\mathcal{H}(\tilde{\pi}_A(Y))}(\tilde{\pi}_A(Y))-\Psi_{\mathcal{H}(Y)}(\tilde{\pi}_A(Y)),
\end{equation}
where the equality uses the fixed-parameter equivariance of the backbone. Assuming the conditional backbone $\Psi_w$ is Lipschitz continuous with respect to the transformation parameter $w$, with Lipschitz constant $L_w$, this additional equivariance error is strictly bounded by
\begin{equation}\label{app:supp:gen-inexact-bound}
	\left\|F(\tilde{\pi}_A(Y))-\tilde{\pi}_A(F(Y))\right\|
	\leq L_w\,\left\|\mathcal{H}(\tilde{\pi}_A(Y))-\mathcal{H}(Y)\right\|.
\end{equation}
Consequently, any small perturbation in the hypernetwork's prediction only introduces a gracefully bounded equivariance error, controlled by the deviation of the predicted parameters along the transformation orbit. Hence, although the parameter adaptation introduces sample-dependent dynamics, the equivariance error bounds and empirical risk analysis derived in the main text under idealized assumptions remain highly instructive. They synergize with our current findings to further clarify the fundamental properties and generalization capabilities of the proposed equivariant network.

\endgroup

% ===== Referee 2, Comment 2.8: stable parameterization of the inverse transformation =====

\begingroup
\color{black}
\subsection{Implementation Details}\label{app:supp:hypernet-init}
In the adaptive convolutions of Section~\ref{app:supp:network-discrete-grid}, the adaptive filter bases are evaluated at points of the form $A^{-1}D_w^{-1}\delta_{ij}$ and $B^{-1}D_w^{-1}\delta_{ij}$, so the inverse transformation matrix $D_w^{-1}$ enters the computation directly. As defined in \eqref{D} of the main text, $D_w$ is the affine matrix with adaptive parameters $w=[\alpha,s_x,s_y]$, where $s_x$ and $s_y$ are the scaling factors along the two orthogonal axes, $\alpha$ determines the principal orientation, and $\det(D_w)=s_xs_y$. The group $G_w=D_wSD_w^{-1}$ and the filter evaluations are well defined precisely when $s_xs_y\neq0$. In this subsection, we describe how the inverse transformation is constructed directly and stably in practice, thereby avoiding the numerical inversion of a potentially near-singular matrix.

Instead of computing $D_w^{-1}$ numerically at every forward pass, the hypernetwork $\mathcal{H}$ directly predicts the coefficients of the inverse coordinate transformation. Specifically, it outputs the inverse scaling factors
\begin{equation}\label{app:supp:inv-scaling}
	\hat{s}_x=s_x^{-1}=1+\phi_x(Y),\qquad
	\hat{s}_y=s_y^{-1}=1+\phi_y(Y),
\end{equation}
together with the orientation parameter $\hat{\alpha}=\alpha$, where $\phi_x$ and $\phi_y$ denote the outputs of the corresponding network modules evaluated at the input image $Y$. The inverse transformation matrix is then assembled directly as
\begin{equation}\label{app:supp:inverse-matrix}
	D_w^{-1}
	=\begin{bmatrix}
		\hat{s}_x\cos\hat{\alpha} & -\hat{s}_x\sin\hat{\alpha}\\
		\hat{s}_y\sin\hat{\alpha} & \hat{s}_y\cos\hat{\alpha}
	\end{bmatrix},
\end{equation}
which is used to evaluate the adaptive filter bases in \eqref{app:newPsi_}--\eqref{app:newUpsilon_}. Since the inverse parameters are predicted directly, no matrix inversion is performed, and the numerical instability associated with inverting a near-singular $D_w$ is avoided by construction.

The initialization of the hypernetwork is controlled so that $\phi_x(Y)$ and $\phi_y(Y)$ output $0$ at the beginning of training, which yields $\hat{s}_x=\hat{s}_y=1$. The angle head is initialized analogously to output $\hat{\alpha}=0$. Consequently, $D_w^{-1}=D_w=I$ at initialization, and the network reduces to the standard non-adaptive configuration. During training, the predicted parameters remain close to their initial values.

In practice, the magnitudes of the module outputs (e.g., $\phi_x(Y)$) remain small, so the predicted $D_w$ and $D_w^{-1}$ stay close to the identity matrix and remain safely bounded away from singularity. This empirical behavior is consistent with the premise of the theoretical analysis: the uniform derivative bounds required in Theorem~\ref{Error-dis} (Section~\ref{app:supp:discretization-assumptions}) implicitly assume that the condition number $\kappa(D_{w_n})$ of every transformation matrix is strictly bounded, uniformly over the compatible meshes, the samples, and the admissible transformations. The above parameterization and initialization therefore naturally fulfill this condition-number premise, ensuring that the theoretical guarantees of Theorem~\ref{Error-dis} are upheld in practice.
\endgroup

\begingroup
\color{black}
	\section{Empirical Validation under Progressive Symmetry Mismatch}\label{app:supp:controlled-symmetry}
	
	In the main text, we demonstrate the efficacy of Sample-Adaptive Equivariant Networks (SA-Conv) across standard restoration tasks, and our theoretical framework shows that a smaller data symmetry error $\varepsilon$ tightens the empirical risk upper bound for a given baseline risk and discretization setting. To further examine the benefits of data symmetry, we design a controlled experiment with progressively weakened data symmetry. By applying directional deformations of increasing strength, we systematically alter the intrinsic geometric structures of the images and quantify the resulting symmetry error. We then evaluate how the performance advantage of SA-Conv over the non-adaptive TL-Conv responds to this growing symmetry mismatch.
	
	\subsection{Experimental Setup}
	Image denoising is a classic restoration task that removes noise while preserving image structures. We therefore use it as an example and conduct this analysis on Urban100 dataset, which also serves as a test dataset in the natural-image denoising experiment. To ensure a strict comparison, we evaluate the exact pretrained models of SA-Conv and TL-Conv from the main text without any fine-tuning. TL-Conv learns transformation parameters shared across samples, while SA-Conv uses a hypernetwork to predict sample-specific transformation parameters $w_n$ for each input. At each deformation level, a $192\times192$ patch is sampled from the same center of each clean source image, keeping the output patch dimensions fixed across settings.
	
	To progressively disrupt the natural symmetry of these images, we apply area-preserving directional deformations by stretching one spatial axis and compressing the orthogonal axis.
	The deformation strength is governed by a parameter $\lambda\in\{0,0.05,0.10,\cdots,0.35\}$, where larger values correspond to more pronounced deformations. For the $n$-th clean sample $r_n$, the horizontal and vertical deformation matrices are defined as:
	\begin{equation}\label{app:supp:controlled-deformations}
		D^{(x)}_{n,\lambda}=\operatorname{diag}\bigl(e^{\lambda a_n},e^{-\lambda a_n}\bigr),\qquad
		D^{(y)}_{n,\lambda}=\operatorname{diag}\bigl(e^{-\lambda a_n},e^{\lambda a_n}\bigr),
	\end{equation}
	where $a_n$ is drawn uniformly from $[0.5,1)$ with a fixed seed and then held fixed for each image. The deformed image $r_{n,\lambda}$ is generated via bilinear interpolation using the inverse mapping $r_{n,\lambda}(x) = r_n(D_{n,\lambda}^{-1} x)$, ensuring all sampling points remain within the source image.   
	%	The first matrix stretches the horizontal axis and compresses the vertical axis, while the second reverses these directions. Both preserve area, and a larger $\lambda$ produces stronger deformation. At $\lambda=0$, both matrices reduce to the identity.
	%	Using coordinates centered at the common patch sampling center, we apply either matrix through the image action defined in the main text: $r_{n,\lambda}(x)=\pi_{D_{n,\lambda}}[r_n](x)=r_n(D_{n,\lambda}^{-1}x)$, where $D_{n,\lambda}$ is either $D^{(x)}_{n,\lambda}$ or $D^{(y)}_{n,\lambda}$. Thus, each pixel of the output patch is sampled from its inverse-mapped location in the original image using bilinear interpolation. All sampling points lie within the source image, and the same half-pixel offset along both axes is used at every $\lambda$. The two deformation families use the same samples, centers, and coefficients $a_n$.
	
	Following the original experimental setup in the main text, we then corrupt these deformed patches with Gaussian noise ($\sigma=50$). Crucially, the specific noise realization for each source image is held constant across all deformation levels and directions to rigorously isolate the impact of geometric distortion. Both models process the identical noisy patches and are evaluated against their respective deformed clean ground truths.

	\begin{figure}
		\centering
		\includegraphics[width=1\linewidth]{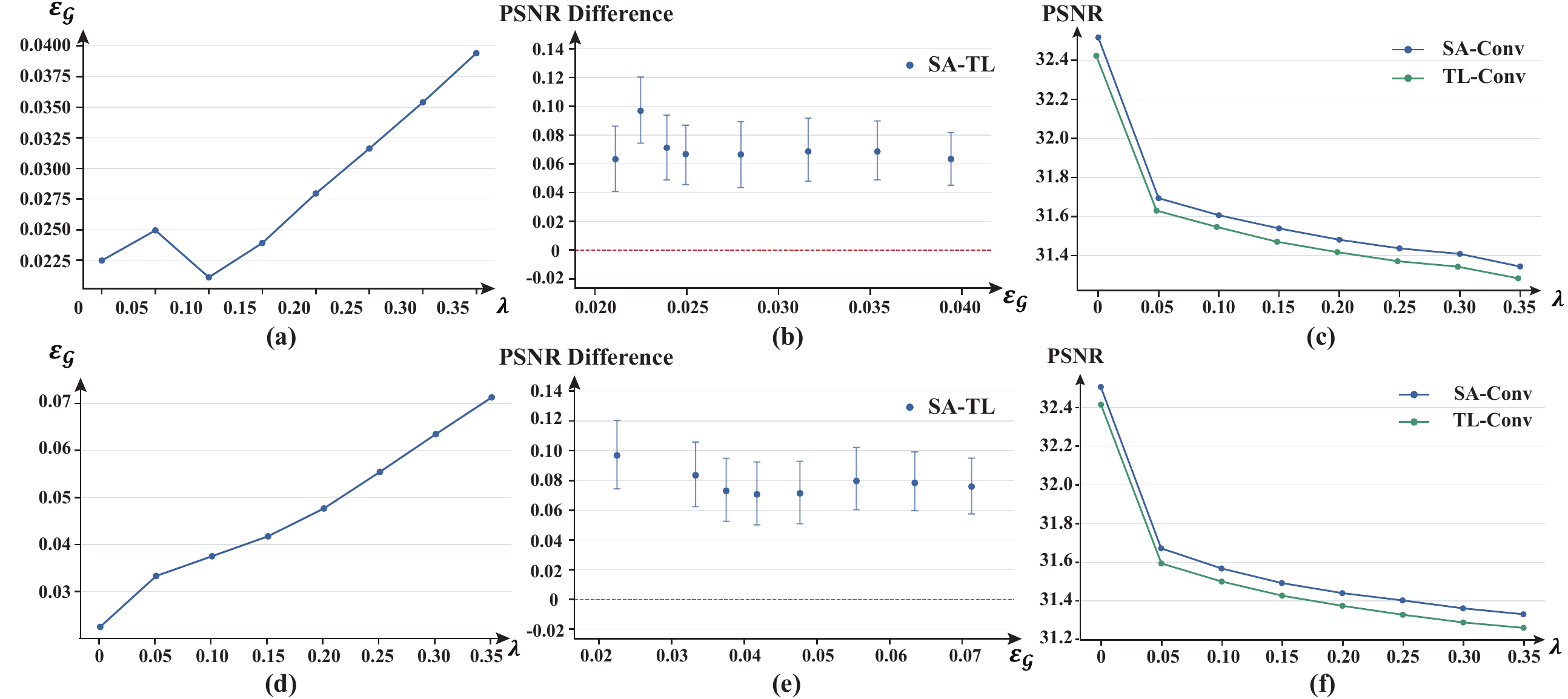}
		\caption{\textcolor{black}{Quantitative evaluation of data symmetry and restoration performance under directional deformations. The top and bottom rows correspond to horizontal and vertical stretching, respectively.
				(a), (d): The variation of measured symmetry error $\varepsilon_{\mathcal G}$ with deformation strength $\lambda$.
				(b), (e): The PSNR gain of SA-Conv over TL-Conv at different measured symmetry errors, displaying 95\% confidence intervals for the mean paired PSNR differences.
				(c), (f): The variation of mean denoising PSNR with deformation strength $\lambda$. All evaluations are strictly conducted on the same 100 images from Urban100 using fixed pretrained networks.}}
		\label{app:fig:supp:controlled-symmetry}
		\vspace{-7mm}
	\end{figure}
	
	The 95\% confidence intervals for the mean PSNR gains are estimated by resampling the paired image-wise PSNR differences between SA-Conv and TL-Conv with replacement 2,000 times.
	%%	At each deformation strength, the error bars in Fig.~\ref{app:fig:supp:controlled-symmetry} show a 95\% confidence interval for the mean PSNR difference between SA-Conv and TL-Conv. We estimate this interval by resampling the paired image-wise PSNR differences with replacement 2,000 times.
	\subsection{Quantitative Comparison}
	Following the empirical feature-response evaluation established in Section~\ref{sec:exp-symmetry} of the main text, we adopt the Total Variation (TV) operator as our specific feature descriptor, denoted by $R$.

Adapting the metric in \eqref{metric} to the deformed dataset, we quantify its symmetry error $\varepsilon_{\mathcal G}^{(s)}(\lambda)$ by the range of the dataset-averaged responses over the sampled rotations:

\begin{equation}\label{app:supp:directional-tv}
    \varepsilon_{\mathcal G}^{(s)}(\lambda)
    =\max_{1\leq t_1,t_2\leq T}\frac{1}{N}\sum_{n=1}^{N}
    \left(
    R[\pi_{A_n^{t_1}}[r_{n,\lambda}^{(s)}]]
    -R[\pi_{A_n^{t_2}}[r_{n,\lambda}^{(s)}]]
    \right),
\end{equation}

	where $r_n$ is the $n$-th clean source image, $r_{n,\lambda}^{(s)} = \pi_{D_{n,\lambda}^{(s)}}[r_n]$ and $s\in\{x,y\}$ specifies the horizontal or vertical stretch setting, with $D_{n,\lambda}^{(s)}$ being the corresponding deformation matrix at strength $\lambda$. The operator $\pi$ mathematically denotes the spatial transformation action. To align with our main experiments, we sample $T=8$ rotation angles, using $G_n=\{A^t\}_{t=1}^{T}$ for all $n$, where $A^t$ represents a rotation by $\theta_t=2\pi(t-1)/T$. The feature descriptor $R$, implemented here as the directional TV operator, evaluates each image's response under every sampled rotation. 
For each rotation index $t$, we first average these responses over all $N$ images. We then take the difference between the largest and smallest dataset averages, using common indices $t_1,t_2$ across the samples.

	With this metric established, we systematically evaluate both the absolute denoising PSNR and the paired PSNR gain of SA-Conv over TL-Conv across all deformation levels to examine their direct correlation with the quantified symmetry mismatch.
	
	As illustrated in Fig.~\ref{app:fig:supp:controlled-symmetry}, increasing the deformation strength $\lambda$ from $0$ to $0.35$ weakens the data symmetry overall. The measured symmetry error $\varepsilon_{\mathcal{G}}$ increases from $0.0225$ to $0.0394$ under horizontal stretching, and to $0.0713$ under vertical stretching (Fig.~\ref{app:fig:supp:controlled-symmetry}(a) and (d)). At the same time,  we can observe that as the structural regularity of the images is compromised, the absolute restoration performance of both models naturally degrades (Fig.~\ref{app:fig:supp:controlled-symmetry}(c) and (f)).
	
	Despite this severe structural disruption, SA-Conv consistently outperforms TL-Conv across all measured symmetry error levels. As detailed in Fig.~\ref{app:fig:supp:controlled-symmetry}(b) and (e), the PSNR advantage of SA-Conv starts at $0.0968$ dB for the undeformed baseline ($\lambda=0$). Across nonzero deformation levels, this paired gain exhibits slight fluctuations but remains positive and statistically significant, ranging between $0.0632$ and $0.0712$ dB for horizontal stretches, and $0.0706$ to $0.0835$ dB for vertical stretches. 
	
	{These controlled results complement the risk analysis: even when natural data symmetry is weakened, dynamically adapting the transformation group to individual samples retains an observed performance advantage. This robustness supports the use of sample-adaptive transformations to improve alignment with the data in this experiment.}

	\endgroup

\clearpage
\bibliographystyle{siamplain}
\bibliography{references}

\end{document}